%% file: main.tex
\documentclass[12pt,reqno]{article}

\usepackage[main, final]{neurips_2026}

\usepackage[dvipsnames]{xcolor}
\usepackage[hidelinks,bookmarks=false,linkcolor=black]{hyperref}
\usepackage{graphicx,here,comment}
\usepackage{multirow}
\usepackage{enumitem}
\usepackage{natbib}
\usepackage{booktabs}
\usepackage{subcaption}

\usepackage{amssymb, amsmath, amsfonts, amsthm}
\usepackage{bm}
\usepackage{mathtools}

\newtheorem{theorem}{Theorem}[section]
\newtheorem{corollary}[theorem]{Corollary}
\newtheorem{lemma}[theorem]{Lemma}
\newtheorem{proposition}[theorem]{Proposition}

\newtheorem{assumption}{Assumption}[section]
\newtheorem{fact}{Fact}[section]

\theoremstyle{definition}
\newtheorem{definition}{Definition}[section]
\newtheorem{remark}{Remark}[section]
\newtheorem{example}{Example}[section]

\title{Spectrum-Adaptive Generalization Bounds \\for Trained Deep Transformers}
\author{Mana Sakai$^{1,2}$\quad Masaaki Imaizumi$^{1,2,3}$}
\date{}

\begin{document}

\maketitle

\vspace{-2em}
\begin{center}
    $^{1}$The University of Tokyo\\
    $^{2}$RIKEN Center for Advanced Intelligence Project\\
    $^{3}$Kyoto University
\end{center}

\begin{center}
    \texttt{mana.sakai.77@gmail.com,\quad imaizumi@g.ecc.u-tokyo.ac.jp}
\end{center}

\begin{abstract}
    Understanding why trained Transformers generalize well is a fundamental problem in modern machine learning theory, and complexity-based generalization bounds provide a principled way to study this question. While existing norm-based bounds for Transformers remove the explicit polynomial dependence on the hidden dimension, they typically impose fixed norm constraints specified a priori and can exhibit unfavorable exponential dependence on depth. In this paper, we derive spectrum-adaptive post hoc generalization bounds for multi-layer Transformers. Under layerwise spectral norm control, the bounds are expressed in terms of layerwise Schatten quantities of the query-key, value, and feedforward weight matrices. Since the Schatten indices need not be fixed a priori and can instead be selected after training, separately for each matrix type and layer, the bounds adaptively trade off spectral complexity against the dimension- and depth-dependent factors according to the learned singular-value profiles. Empirical comparisons of BERT-adapted proxies for the leading complexity factors suggest that the proxies induced by our bounds grow more slowly with depth and hidden dimension than the corresponding norm-based proxies. Overall, our results provide a complexity-based perspective on how the spectral structure of trained Transformers is reflected in generalization analyses.
\end{abstract}

\section{Introduction}\label{sec:Introduction}

Transformers have become a central architecture of modern machine learning since their introduction by \citet{vaswani2017attention}. Their success across language \citep{devlin2019bert,brown2020language,chowdhery2023palm}, vision \citep{dosovitskiy2021image}, and many other modalities has made it increasingly important to understand why heavily overparameterized Transformers can still generalize well. Generalization bounds offer one principled route toward this goal, because they can reveal which features of a learned model are statistically relevant.

For generalization analysis of Transformers, \citet{edelman2022inductive} established norm-based generalization gap bounds for multi-layer Transformers. A notable feature of their bounds is that the dependence on both the hidden dimension and the token length appears only through logarithmic factors. \citet{trauger2024sequence} sharpened this work by removing the explicit token-length dependence. More recently, \citet{li2026sharper} used the offset Rademacher complexity to derive excess risk bounds with an $O(1/n)$ rate, where $n$ is the training sample size.

At the same time, these advances expose two difficulties that arise when existing norm-based bounds are used to assess trained deep Transformers. First, the improvement in explicit dimension dependence can come with layerwise propagation factors that can scale as $C^{L}$ in the worst case, where $L$ is the depth and $C$ is determined by Lipschitz constants and spectral norm bounds across layers. This issue is benign for shallow networks or when the propagation constant is at most one, but it can quickly make the bounds loose for deep models. Second, the fixed norm radii appearing in existing bounds are not naturally dimension-independent for large matrices. In particular, the bounds of \citet{edelman2022inductive} and \citet{trauger2024sequence} are controlled by mixed $(2,1)$- and $(1,1)$-norm radii, respectively, where the mixed $(\alpha,\beta)$-norm is defined by $\|W\|_{\alpha,\beta}=[\sum_{j}(\sum_{i}|W_{ij}|^{\alpha})^{\frac{\beta}{\alpha}}]^{\frac{1}{\beta}}$. Such quantities can remain small for matrices that are sparse in the corresponding coordinatewise sense, but they typically grow with the hidden dimension.

A further motivation comes from the observed heterogeneity of trained Transformer weights. Recent compression and spectral analyses suggest that different sublayers can have markedly different singular-value profiles: attention-related matrices often exhibit stronger low-rank structure or more compressible spectra than feedforward-related matrices \citep{li2024lorap,yuan2023asvd}. This suggests that a useful complexity measure should not impose a single global norm description on all layers and matrix types. Rather, it should be able to adapt to the spectral profile of each weight.

This point is especially important for post hoc generalization bounds. In post hoc bounds, the high-probability event is specified independently of the trained weights, while the admissible complexity parameters may be chosen after the weights are observed. Such guarantees are useful when the relevant structure of the trained model is not known a priori. In the context of DNNs and CNNs, \citet{ledent2025generalization} derived generalization bounds based on Schatten (quasi) norms that hold simultaneously over all choices of layerwise Schatten indices, so that these indices can be selected after training. Here, for a matrix $W$ and a Schatten index $p\in(0,2]$, the Schatten $p$ (quasi) norm is defined by $\|W\|_{\mathrm{s},p}=(\sum_{i}\sigma_{i}(W)^{p})^{1/p}$, where $\sigma_{i}(W)$ is the $i$-th singular value. For Transformers, post hoc adaptivity is particularly natural because the spectral profiles of trained matrices may differ substantially across layers.

In this paper, we derive post hoc generalization gap bounds for multi-layer Transformers that can adaptively reflect the spectral structure of each learned weight. Specifically, our bounds are based on spectral quantities of the weights: besides the Schatten quantity $\|W\|_{\mathrm{s},p}^{p}\ (p\in[0,2])$, with the convention $\|W\|_{\mathrm{s},0}^{0}=\operatorname{rank}(W)$, we also use the spectral norm $\|W\|_{2}=\sigma_{1}(W)$. Table~\ref{tab:GenBoundComparisonOriginal} summarizes a specialized version of our result along with existing norm-based bounds. A notable feature of our bounds is that the complexity measure, through the Schatten index $p$, can be selected after training.

\begin{table}[t]
    \caption{Leading complexity factor $B$ in simplified generalization gap bounds of the form $\tilde{O}(B/\sqrt{n})$. All weights are $N\times N$ and satisfy spectral norm constraints $\|W\|_{2}=O(1)$. Here $L$ is the depth, $n$ is the sample size, and $C$ denotes the worst-case layerwise propagation factor appearing in existing norm-based bounds. For our post hoc bounds, the displayed expression is the common-$p$ specialization of Theorem~\ref{thm:GenGapPostHocSimplified}; the theorem itself allows a separate $p$ for every matrix type and layer. While our bounds and the bounds of \citet{edelman2022inductive} depend on the token length logarithmically, the bounds of \citet{trauger2024sequence} are independent of the token length.
    }
    \label{tab:GenBoundComparisonOriginal}
    \centering
    \begin{tabular}{lcccc}
        \toprule
        &Bound&Assumption&Post Hoc\\
        \midrule
        \textbf{Ours} (Theorem~\ref{thm:GenGapPostHocSimplified})&$\inf_{p}(\|W\|_{\mathrm{s},p}^{p})^{\frac{1}{p+2}}C^{\frac{Lp}{p+2}}L^{\frac{2p+2}{p+2}}N^{\frac{p+1}{p+2}}$& & $\surd$\\
        \citet{edelman2022inductive}&$C_{2,1}C^{L}L^{\frac{3}{2}}$&$\|W\|_{2,1}\le C_{2,1}$ &
        \\
        \citet{trauger2024sequence}&$C_{1,1}C^{L}L^{\frac{3}{2}}$&$\|W\|_{1,1}\le C_{1,1}$ & \\
        \bottomrule
    \end{tabular}
\end{table}

These bounds have several consequences. First, they provide a concrete interpolation between rank-based and norm-based regimes. When $p=0$, the Schatten quantity is the rank; when $p=2$, it is the squared Frobenius norm. Intermediate values of $p$ describe a soft-rank measure through singular-value decay. Second, this interpolation directly controls the trade-off among spectral complexity, hidden dimension, and depth. Smaller values of $p$ reduce the layerwise propagation factor $C^{Lp/(p+2)}$, the polynomial depth factor $L^{(2p+2)/(p+2)}$, and the hidden dimension $N^{(p+1)/(p+2)}$, whereas larger values of $p$ provide a more Frobenius-like description. Because our bounds take an infimum over $p$ after training, they automatically select the most favorable balance for each learned matrix. Third, using BERT Miniatures checkpoints of \citet{turc2019well}, we find that the leading complexity factors suggested by our bounds grow substantially more slowly than the corresponding norm-based complexity factors of \citet{edelman2022inductive} as the depth or hidden dimension increases.

To obtain these spectrum-adaptive bounds, we derive covering number bounds that retain the dependence on the Schatten indices throughout the Transformer composition. Technically, we develop covering number bounds under layerwise spectral norm and Schatten-quantity constraints, building on the recently developed parametric interpolation \citep{ledent2024generalization,ledent2025generalization}. This technique decomposes each weight matrix into a low-rank leading part and a Frobenius-controlled tail, thereby allowing the bounds to interpolate between rank-based and norm-based regimes.

The contributions of this paper are summarized as follows.
\begin{itemize}
    \item We prove post hoc generalization gap bounds for multi-layer Transformers under layerwise spectral norm control, where the complexity parameters can be selected separately for every layer and matrix type after training (Theorem~\ref{thm:GenGapPostHocSimplified}).
    \item We extend parametric interpolation to matrix-valued function classes (Theorem~\ref{thm:ParametricInterpolation2infSimplified}) and use it as the basic building block for Transformer generalization bounds.
    \item We theoretically compare our bounds with existing norm-based bounds in representative regimes, including Frobenius norm or rank constraints, and show improved dependence on depth (Section~\ref{sec:examples}; Table~\ref{tab:GenBoundComparison}).
    \item We evaluate BERT-adapted leading complexity factors on BERT Miniatures checkpoints of \citet{turc2019well} and observe that our proxies grow more slowly than the corresponding norm-based proxies as the depth or hidden dimension increases (Section~\ref{sec:experiments}; Figure~\ref{fig:ObservedBoundComparison}).
\end{itemize}

Related work is deferred to Appendix~\ref{sec:RelatedWorks}.

\section{Preliminaries}\label{sec:Preliminaries}

\subsection{Notation}

We use the following notation throughout the paper. We use $O(\cdot)$ and $\tilde{O}(\cdot)$ in the standard asymptotic sense, where $\tilde{O}(\cdot)$ suppresses polylogarithmic factors. We write $a\lesssim b$ if there exists a universal constant $C>0$ such that $a\le Cb$. For a vector $v$, we write $v_{i}$ for its $i$-th entry. For a matrix $W$, we write $W_{ij}$ for its $(i,j)$-th entry, and write $W_{i\cdot}$ and $W_{\cdot j}$ for its $i$-th row and $j$-th column, respectively. We denote by $\sigma_{i}(W)$ the $i$-th largest singular value of $W$.

We use the following matrix norms. First, $\|W\|_{2}=\sigma_{1}(W)$ is the spectral norm. For $p\in(0,2]$, the Schatten $p$ (quasi) norm $\|W\|_{\mathrm{s},p}$ is defined by $\|W\|_{\mathrm{s},p}=(\sum_{i}\sigma_{i}(W)^{p})^{{1}/{p}}$. We refer to $p$ as the Schatten index. For the endpoint $p=0$, we use the convention $\|W\|_{\mathrm{s},0}^{0}:=\operatorname{rank}(W)$. Finally, for $\alpha,\beta\ge1$, we define the mixed $(\alpha,\beta)$-norm by $\|W\|_{\alpha,\beta}=[\sum_{j}(\sum_{i}|W_{ij}|^{\alpha})^{{\beta}/{\alpha}}]^{{1}/{\beta}}$.

\subsection{Transformers}

We specify the simplified Transformer architecture analyzed in this paper. Let $X\in\mathbb{R}^{T\times N}$ be an input matrix, where $T$ is the token length and $N$ is the hidden dimension. The model is obtained by composing Transformer blocks, each consisting of a single-head attention mechanism followed by a feedforward map and rowwise normalization.

\paragraph{Transformer head}

Let $W^{QK},W^{V}\in\mathbb{R}^{N\times N}$ denote the combined query-key weight matrix and the value weight matrix, respectively. Define a Transformer head $f_{\mathrm{head}}(\ \cdot\ ;W^{QK},W^{V}):\mathbb{R}^{T\times N}\to\mathbb{R}^{T\times N}$ parameterized by $W^{QK},W^{V}$ by
\begin{equation}\label{eq:DefHead}
    f_{\mathrm{head}}(X;W^{QK},W^{V})
    =\mathrm{SoftMax}\left(XW^{QK}X^{\top}\right)XW^{V}
    ,
\end{equation}
where $\mathrm{SoftMax}:\mathbb{R}^{T\times T}\to\mathbb{R}^{T\times T}$ is applied rowwise:
\[
    (\mathrm{SoftMax}(Z))_{st}
    =\frac{\exp(Z_{st})}{\sum_{j=1}^{T}\exp(Z_{sj})}
    \qquad(s,t\in[T])
    .
\]

\paragraph{Transformer block}

Following \citet{edelman2022inductive} and \citet{trauger2024sequence}, we consider a normalized Transformer block. Let $\phi:\mathbb{R}^{N}\to\mathbb{R}^{N}$ be a fixed activation function in the feedforward layer, and apply it rowwise to matrices by setting $(\phi(Z))_{t\cdot}=\phi(Z_{t\cdot})$ for $Z\in\mathbb{R}^{T\times N}$. Let $\Pi_{\mathrm{norm}}:\mathbb{R}^{T\times N}\to\mathbb{R}^{T\times N}$ be the rowwise projection onto the unit ball defined by $(\Pi_{\mathrm{norm}}(Z))_{t\cdot}=Z_{t\cdot}/(1\vee\|Z_{t\cdot}\|_{2})$, which keeps the row norms at most one. We define the Transformer block $f_{\mathrm{block}}(\ \cdot\ ;W^{QK},W^{V},W^{M}):\mathbb{R}^{T\times N}\to\mathbb{R}^{T\times N}$ parameterized by $W^{QK},W^{V},W^{M}\in\mathbb{R}^{N\times N}$ as
\begin{equation}\label{eq:DefBlock}
    f_{\mathrm{block}}(X;W^{QK},W^{V},W^{M})
    =\Pi_{\mathrm{norm}}(\phi(\Pi_{\mathrm{norm}}(f_{\mathrm{head}}(X;W^{QK},W^{V})))W^{M})
    .
\end{equation}

\paragraph{Multi-layer Transformer}

We introduce a multi-layer Transformer obtained by composing $L$ Transformer blocks. For each $\ell\in[L]$, we write $W^{(\ell)}=(W^{QK,(\ell)},W^{V,(\ell)},W^{M,(\ell)})$ to denote the weight matrices for the $\ell$-th Transformer block. Then we define
\[
    f_{\mathrm{tf}}^{(1)}(X;W^{(1)})
    =f_{\mathrm{block}}(X;W^{QK,(1)},W^{V,(1)},W^{M,(1)}),
\]
and, for $\ell\ge2$,
\begin{equation}\label{eq:DefMultiLayerTf}
    f_{\mathrm{tf}}^{(\ell)}(X;W^{(1:\ell)})
    =f_{\mathrm{block}}(f_{\mathrm{tf}}^{(\ell-1)}(X;W^{(1:\ell-1)});W^{(\ell)})
    ,
\end{equation}
where $W^{(1:\ell)}=(W^{(1)},\dots,W^{(\ell)})$ denotes the weights corresponding to the first $\ell$ layers.

\paragraph{Scalar output of a Transformer}

We convert the matrix-valued output of the final Transformer layer into a scalar prediction by using a special classification token. Let $t_{\mathrm{CLS}}\in[T]$ denote the index of this token, and assume that the input $X\in\mathbb{R}^{T\times N}$ has already been augmented with the CLS token. For a matrix $Z\in\mathbb{R}^{T\times N}$, we write $[Z]_{[\mathrm{CLS}]}:=(Z_{t_{\mathrm{CLS}}\cdot})^{\top}\in\mathbb{R}^{N}$ to denote the hidden representation at the CLS position. We then apply a trainable readout vector $w\in\mathbb{R}^{N}$ to the representation $[f_{\mathrm{tf}}^{(L)}(X;W^{(1:L)})]_{[\mathrm{CLS}]}$ and define
\begin{equation}\label{eq:DefTfOut}
    f_{\mathrm{out}}(X;W^{(1:L)},w)
    =w^{\top}[f_{\mathrm{tf}}^{(L)}(X;W^{(1:L)})]_{[\mathrm{CLS}]}
    .
\end{equation}

\paragraph{Class of Transformer outputs}

For each $\star\in\{QK,V,M\}$ and $\ell\in[L]$, define the class of parameters as
\begin{equation}\label{eq:ParSpectralConst}
    \mathcal{W}^{\star,(\ell)}
    =\{W\in\mathbb{R}^{N\times N}\mid\|W\|_{2}\le C_{2}^{\star,(\ell)}\}
    .
\end{equation}
Set $\mathcal{W}^{(\ell)}=\mathcal{W}^{QK,(\ell)}\times\mathcal{W}^{V,(\ell)}\times\mathcal{W}^{M,(\ell)}$ and $\mathcal{W}^{(1:L)}=\mathcal{W}^{(1)}\times\cdots\times\mathcal{W}^{(L)}$. Then the scalar-output Transformer class is
\begin{equation}\label{eq:TfOutClass}
    \mathcal{F}_{\mathrm{out}}
    =\{f_{\mathrm{out}}(\ \cdot\ ;W^{(1:L)},w):\mathbb{R}^{T\times N}\to\mathbb{R}\mid W^{(1:L)}\in\mathcal{W}^{(1:L)},w\in\mathbb{R}^{N},\|w\|_{2}\le C_{2}^{\mathrm{out}}\}
    .
\end{equation}

\subsection{Generalization gap}

Let $\mathcal{L}:\mathbb{R}\times\mathbb{R}\to\mathbb{R}$ be a $B_{\mathcal{L}}$-bounded loss function that is $L_{\mathcal{L}}$-Lipschitz in its first argument. Suppose $\{(X_{i},Y_{i})\}_{i\in[n]}$ is an i.i.d. sample drawn from distribution $\mathcal{D}$ on $\mathbb{R}^{T\times N}\times\mathbb{R}$. For each $f_{\mathrm{out}}\in\mathcal{F}_{\mathrm{out}}$, define the population risk and the empirical risk by
\[
    \mathcal{R}(f_{\mathrm{out}})=\mathbb{E}_{(X,Y)\sim\mathcal{D}}[\mathcal{L}(f_{\mathrm{out}}(X),Y)]
    ,\qquad
    \hat{\mathcal{R}}_{n}(f_{\mathrm{out}})=\frac{1}{n}\sum_{i=1}^{n}\mathcal{L}(f_{\mathrm{out}}(X_{i}),Y_{i})
    ,
\]
respectively. Then, the generalization gap is defined by
\[
    \mathrm{GAP}(f_{\mathrm{out}})=|\mathcal{R}(f_{\mathrm{out}})-\hat{\mathcal{R}}_{n}(f_{\mathrm{out}})|
    .
\]

In this paper, we derive generalization bounds in which the high-probability event holds uniformly over a prescribed family of complexity parameters. In our setting, the post hoc parameters are the Schatten indices $\bm{p}=(p^{\star,(\ell)})_{\star\in\{QK,V,M\},\ell\in[L]}\in[0,2]^{3L}$ associated with the parameter matrices. Accordingly, our bounds take the following form: for any $\delta\in(0,1)$, with probability at least $1-\delta$, it holds that
\[
    \mathrm{GAP}(f_{\mathrm{out}}(\ \cdot\ ;W^{(1:L)},w))
    \le B(W^{(1:L)},w,\bm{p},\delta)
    \text{ for all }f_{\mathrm{out}}(\ \cdot\ ;W^{(1:L)},w)\in\mathcal{F}_{\mathrm{out}},\ \bm{p}\in[0,2]^{3L}
    .
\]
Therefore, the Schatten indices may be chosen after observing the trained weights.

\section{Spectrum-adaptive post hoc generalization bounds for Transformers}\label{sec:MainResult}

In this section, we state our spectrum-adaptive post hoc generalization bounds for multi-layer Transformers. The model class is defined in Eq.~\eqref{eq:TfOutClass}. Throughout this paper, we impose the following regularity assumption.

\begin{assumption}\label{asm:Activation}
    The activation function $\phi:\mathbb{R}^{N}\to\mathbb{R}^{N}$ satisfies $\phi(0)=0$ and is $L_{\phi}$-Lipschitz with respect to the Euclidean norm; that is, $\|\phi(x)-\phi(y)\|_{2}\le L_{\phi}\|x-y\|_{2}$ holds for all $x,y\in\mathbb{R}^{N}$.
\end{assumption}

For example, the ReLU activation satisfies Assumption~\ref{asm:Activation} with $L_{\phi}=1$.

To isolate the leading scaling, the simplified bounds below treat the following quantities as $\Theta(1)$: the loss constants $L_{\mathcal{L}}$ and $B_{\mathcal{L}}$, the readout norm constraint $C_{2}^{\mathrm{out}}$, the layerwise spectral norm constraints $C_{2}^{\star,(\ell)}$, the activation Lipschitz constant $L_{\phi}$, and the input-norm constants appearing in our full statement. We also suppress polylogarithmic factors, except for those involving $n$, $T$, and the post hoc uniformity penalty.

\begin{theorem}[Spectrum-adaptive post hoc bounds (Theorem~\ref{thm:GenGapPostHoc}, simplified)]\label{thm:GenGapPostHocSimplified}
    Suppose $n\ge3$ holds. Suppose further that there exists a universal constant $c_{0}>0$ such that $\|W^{\star,(\ell)}\|_{2}\ge\exp[-c_{0}(L+\log(N))]$ holds for every nonzero $W^{\star,(\ell)}$.\footnote{This condition is imposed only to state the post hoc bounds in the simplified form of Eq.~\eqref{eq:GenGapPostHocSimplified}: it rules out exponentially small nonzero weights so that the additional terms introduced by the post hoc discretization can be absorbed into the displayed logarithmic penalty. In the full statement, these scale-dependent terms are kept explicit.} Then, for any $\delta\in(0,1)$, with probability at least $1-\delta$, it holds simultaneously for all $f_{\mathrm{out}}(\cdot;W^{(1:L)},w)\in\mathcal{F}_{\mathrm{out}}$ satisfying the above condition that
    \begin{equation}\label{eq:GenGapPostHocSimplified}
        \begin{aligned}
            &\mathrm{GAP}(f_{\mathrm{out}})\\
            &\lesssim\sqrt{\frac{\log(nT)}{n}}\inf_{\bm{p}\in[0,2]^{3L}}\sum_{\ell=1}^{L}\sum_{\star\in\{QK,V,M\}}(\|W^{\star,(\ell)}\|_{\mathrm{s},p^{\star,(\ell)}}^{p^{\star,(\ell)}})^{\frac{1}{p^{\star,(\ell)}+2}}C^{\frac{Lp^{\star,(\ell)}}{p^{\star,(\ell)}+2}}L^{\frac{p^{\star,(\ell)}}{p^{\star,(\ell)}+2}}N^{\frac{p^{\star,(\ell)}+1}{p^{\star,(\ell)}+2}}\\
            &\quad+\frac{(\log(n))^{\frac{3}{2}}}{\sqrt{n}}+\sqrt{\frac{\log(1/\delta)+L\log(L+\log(N))}{n}}
            ,
        \end{aligned}
    \end{equation}
    where the infimum is taken over a vector of Schatten indices $\bm{p}\in[0,2]^{3L}$, and $C>0$ is a constant depending on the spectral norm bounds $C_{2}^{\star,(\ell)}$ and on $L_{\phi}$.
\end{theorem}

The three terms on the right-hand side of Eq.~\eqref{eq:GenGapPostHocSimplified} have distinct origins. The first term is the main complexity term, obtained by covering the Transformer body, namely the query-key, value, and feedforward components across layers. The second term comes from covering the final readout vector $w$. The last term is the confidence and post hoc uniformity penalty. The $\log(1/\delta)$ part is the usual concentration term, while the additional $L\log(L+\log(N))$ factor is used to make the guarantee uniform over the post hoc choices of the Schatten indices.

The significance of Theorem~\ref{thm:GenGapPostHocSimplified} is that the Schatten indices become post hoc, weight-adaptive complexity parameters, rather than modeling assumptions fixed in advance. Since the high-probability event holds uniformly over all admissible Schatten indices, the index $p^{\star,(\ell)}$ can be selected after training, separately for each matrix type and each layer. Smaller values of $p^{\star,(\ell)}$ exploit rank-like spectral structure and reduce the per-matrix summand factors $C^{Lp^{\star,(\ell)}/(p^{\star,(\ell)}+2)}$, $L^{p^{\star,(\ell)}/(p^{\star,(\ell)}+2)}$, and $N^{(p^{\star,(\ell)}+1)/(p^{\star,(\ell)}+2)}$. In contrast, larger values of $p^{\star,(\ell)}$ yield a more Frobenius-like description. Thus, the infimum in Eq.~\eqref{eq:GenGapPostHocSimplified} balances spectral complexity, hidden dimension, depth, and layerwise propagation on a matrixwise and layerwise basis. Consequently, even in regimes where existing norm-based bounds may become loose because of depth accumulation or large fixed norm constants, Theorem~\ref{thm:GenGapPostHocSimplified} can select a more favorable trade-off among depth, hidden dimension, and spectral complexity.

We note that the remaining logarithmic dependence on the token length $T$ reflects a geometric trade-off. \citet{trauger2024sequence} remove explicit token-length dependence by using mixed $(1,1)$-norm constraints and an $\ell_{1}$-type covering argument. In contrast, our bounds are based on SVD/Schatten quantities and use an $\ell_{2}$-type covering argument over the empirical token rows, which introduces the factor $\log(nT)$.

\section{Proof outline}\label{sec:ProofOutline}

We outline the proof of Theorem~\ref{thm:GenGapPostHocSimplified}. Define the $2\to\infty$ norm of a matrix by $\|W\|_{2\to\infty}=\max_{i}\|W_{i\cdot}\|_{2}$. We use empirical covering numbers with respect to this metric. Fix inputs $\{X_{i}\}_{i\in[n]}$ and a class $\mathcal{F}$ of matrix-valued functions. Then the covering number $\mathcal{N}_{\infty}(\mathcal{F},\|\cdot\|_{2\to\infty},\epsilon;\{X_{i}\}_{i\in[n]})$ is the minimum cardinality of a finite set $\mathcal{C}\subset\mathcal{F}$ such that, for every $f\in\mathcal{F}$, there exists $\tilde{f}\in\mathcal{C}$ satisfying $\max_{i\in[n]}\|f(X_{i})-\tilde{f}(X_{i})\|_{2\to\infty}\le\epsilon$. For scalar-valued classes, we use the same notation with the metric $|\cdot|$.

\subsection{Parametric interpolation for matrix-valued linear maps}

The following theorem is the basic building block of our analysis.  It gives covering number bounds for linear matrix-valued function classes under spectral norm and Schatten-quantity constraints, and therefore provides a flexible way to interpolate between rank-based and norm-based complexity measures.

\begin{theorem}[Matrix-valued parametric interpolation (Theorem~\ref{thm:ParametricInterpolation2inf}, simplified)]\label{thm:ParametricInterpolation2infSimplified}
    Fix an arbitrary $p\in[0,2]$. Consider a class of matrix-valued functions $\mathcal{F}=\{f:\mathbb{R}^{d\times\ell}\to\mathbb{R}^{d\times m}\mid f(X)=XW,\ W\in\mathbb{R}^{\ell\times m},\ \|W\|_{\mathrm{s},p}^{p}\le C_{\mathrm{s}},\ \|W\|_{2}\le C_{2}\}$. Then, for any $\epsilon>0$, we have
    \[
        \log\mathcal{N}_{\infty}(\mathcal{F},\|\cdot\|_{2\to\infty},\epsilon;\{X_{i}\}_{i\in[n]})
        \lesssim\left(\frac{[C_{\mathrm{s}}(\ell+m)]^{2}(\min\{\ell,m\}m)^{p}}{\epsilon^{2p}}\right)^{\frac{1}{p+2}}\log(nd)
        ,
    \]
    where logarithmic factors other than $\log(nd)$ are suppressed.
\end{theorem}
\begin{proof}[Proof idea]
    The proof is based on a parametric interpolation argument.  For any $W$ in the parameter set, write its singular value decomposition as $W=\sum_{j}\sigma_{j}u_{j}v_{j}^{\top}$ and fix a threshold $\tau>0$.  We decompose $W=W_{1}+W_{2}$, where $W_{1}$ contains the singular components with $\sigma_{j}>\tau$ and $W_{2}$ is the remaining tail.  The Schatten constraint implies $\operatorname{rank}(W_{1})\le C_{\mathrm{s}}/\tau^{p}$, while the tail satisfies $\|W_{2}\|_{\mathrm{s},2}\le\tau\sqrt{\min\{\ell,m\}}$. Hence the leading part $W_{1}$ can be covered as a low-rank, spectral-norm-bounded matrix class, whereas the tail $W_{2}$ can be covered as a Frobenius-norm-bounded linear class.  These two covers yield two competing entropy terms of the form $(\ell+m)C_{\mathrm{s}}/\tau^{p}$ and $\tau^{2}\min\{\ell,m\}m/\epsilon^{2}$ up to logarithmic factors.  Optimizing the threshold $\tau$ balances these two terms and gives the exponent $1/(p+2)$ in the theorem.
\end{proof}

This argument follows the same principle as the vector-valued parametric interpolation argument of \citet{ledent2025generalization}.  The matrix-valued setting, however, requires additional care because the functions take values in $(\mathbb{R}^{d\times m},\|\cdot\|_{2\to\infty})$.  In particular, the Frobenius-tail cover must be converted into uniform rowwise output control over the sample $\{X_{i}\}_{i\in[n]}$.

The logarithmic dependence on $T$ in Theorem~\ref{thm:GenGapPostHocSimplified} arises from this matrix-valued linear covering step. Theorem~\ref{thm:ParametricInterpolation2infSimplified} contains $\log(nd)$, which gives $\log(nT)$ when applied to Transformer layers. This reflects the Schatten covering route, unlike the $\ell_{1}$-based argument of \citet{trauger2024sequence}.

\subsection{Fixed-index generalization gap bounds}

We next explain how Theorem~\ref{thm:ParametricInterpolation2infSimplified} is lifted to multi-layer Transformers. First, fix the Schatten indices $\bm{p}\in[0,2]^{3L}$ and the corresponding Schatten-quantity radii $\bm{C}_{\mathrm{s}}=(C_{\mathrm{s}}^{\star,(\ell)})_{\star\in\{QK,V,M\},\ell\in[L]}\in(0,\infty)^{3L}$ in advance. Define the constrained parameter class by
\begin{align*}
    &\mathcal{W}^{(1:L)}(\bm{p},\bm{C}_{\mathrm{s}})
    =\mathcal{W}^{(1)}(\bm{p}^{(1)},\bm{C}_{\mathrm{s}}^{(1)})\times\cdots\times\mathcal{W}^{(L)}(\bm{p}^{(L)},\bm{C}_{\mathrm{s}}^{(L)})
    ,\\
    &\mathcal{W}^{(\ell)}(\bm{p}^{(\ell)},\bm{C}_{\mathrm{s}}^{(\ell)})
    =\mathcal{W}^{QK,(\ell)}(p^{QK,(\ell)},C_{\mathrm{s}}^{QK,(\ell)})\times\mathcal{W}^{V,(\ell)}(p^{V,(\ell)},C_{\mathrm{s}}^{V,(\ell)})\times\mathcal{W}^{M,(\ell)}(p^{M,(\ell)},C_{\mathrm{s}}^{M,(\ell)})
    ,\\
    &\mathcal{W}^{\star,(\ell)}(p^{\star,(\ell)},C_{\mathrm{s}}^{\star,(\ell)})
    =\{W\in\mathcal{W}^{\star,(\ell)}\mid\|W\|_{\mathrm{s},p^{\star,(\ell)}}^{p^{\star,(\ell)}}\le C_{\mathrm{s}}^{\star,(\ell)}\}
    .
\end{align*}
The corresponding scalar output class is
\[
    \mathcal{F}_{\mathrm{out}}^{(\bm{p},\bm{C}_{\mathrm{s}})}
    =\{f_{\mathrm{out}}(\ \cdot\ ;W^{(1:L)},w):\mathbb{R}^{T\times N}\to\mathbb{R}\mid W^{(1:L)}\in\mathcal{W}^{(1:L)}(\bm{p},\bm{C}_{\mathrm{s}}),w\in\mathbb{R}^{N},\|w\|_{2}\le C_{2}^{\mathrm{out}}\}
    .
\]

Theorem~\ref{thm:ParametricInterpolation2infSimplified} is applied to the linear maps arising from the query-key, value, and feedforward matrices in each Transformer block. The resulting covers are composed layer by layer: the Lipschitz properties of the softmax map, the rowwise normalization, and the activation function allow the local approximation errors to be propagated through the Transformer in the $2\to\infty$ norm, and the final readout layer converts the matrix-valued cover into a scalar output cover. Combining the Transformer covering number bounds with a Dudley-type entropy integral and the standard Lipschitz-loss generalization inequality yields the fixed-index generalization bounds in the following two theorems.

\begin{theorem}[Fixed-index generalization bounds (Theorem~\ref{thm:GenGapGeneralP}, simplified)]\label{thm:GenGapGeneralPSimplified}
    Suppose $n\ge3$ holds. Then, for any $\delta\in(0,1)$, with probability at least $1-\delta$, it holds simultaneously for all $f_{\mathrm{out}}\in\mathcal{F}_{\mathrm{out}}^{(\bm{p},\bm{C}_{\mathrm{s}})}$ that
    \begin{align*}
        \mathrm{GAP}(f_{\mathrm{out}})
        &\lesssim\sqrt{\frac{\log(nT)}{n}}\sum_{\ell=1}^{L}\sum_{\star\in\{QK,V,M\}}(C_{\mathrm{s}}^{\star,(\ell)})^{\frac{1}{p^{\star,(\ell)}+2}}C^{\frac{Lp^{\star,(\ell)}}{p^{\star,(\ell)}+2}}L^{\frac{p^{\star,(\ell)}}{p^{\star,(\ell)}+2}}N^{\frac{p^{\star,(\ell)}+1}{p^{\star,(\ell)}+2}}\\
        &\quad+\frac{(\log(n))^{\frac{3}{2}}}{\sqrt{n}}+\sqrt{\frac{\log(1/\delta)}{n}}
        ,
    \end{align*}
    where $C>0$ is the same constant that appeared in Theorem~\ref{thm:GenGapPostHocSimplified}.
\end{theorem}

\begin{theorem}[Common-$p$ bounds (Theorem~\ref{thm:GenGapCommonP}, simplified)]\label{thm:GenGapCommonPSimplified}
    Consider the case $\bm{p}=(p,\dots,p)$. Suppose $n\ge3$ holds. Then, for any $\delta\in(0,1)$, with probability at least $1-\delta$, it holds simultaneously for all $f_{\mathrm{out}}\in\mathcal{F}_{\mathrm{out}}^{(\bm{p},\bm{C}_{\mathrm{s}})}$ that
    \begin{align*}
        \mathrm{GAP}(f_{\mathrm{out}})
        &\lesssim\sqrt{\frac{\log(nT)}{n}}C^{\frac{Lp}{p+2}}\bigg(\sum_{\ell=1}^{L}\sum_{\star\in\{QK,V,M\}}(C_{\mathrm{s}}^{\star,(\ell)})^{\frac{2}{3p+2}}\bigg)^{\frac{3p+2}{2(p+2)}}N^{\frac{p+1}{p+2}}\\
        &\quad+\frac{(\log(n))^{\frac{3}{2}}}{\sqrt{n}}+\sqrt{\frac{\log(1/\delta)}{n}}
        ,
    \end{align*}
    where $C>0$ is the same constant that appeared in Theorem~\ref{thm:GenGapPostHocSimplified}.
\end{theorem}

Theorem~\ref{thm:GenGapCommonPSimplified} is not merely the result of substituting a common value of $p$ into Theorem~\ref{thm:GenGapGeneralPSimplified}. When all matrices share a common Schatten index, the proof can optimize the covering radii across layers and matrix types when we derive the covering entropy bounds. This optimized allocation leads to the aggregate term in Theorem~\ref{thm:GenGapCommonPSimplified} and improves the polynomial dependence on the depth $L$ compared to the naive balanced allocation underlying Theorem~\ref{thm:GenGapGeneralPSimplified}.

\subsection{Post hoc selection of the Schatten indices}

The fixed-index bounds in Theorem~\ref{thm:GenGapGeneralPSimplified} assume that the Schatten indices and the corresponding radii are specified in advance. To obtain the spectrum-adaptive bounds stated in Theorem~\ref{thm:GenGapPostHocSimplified}, we remove this restriction by allowing the Schatten indices to be chosen after observing the trained weights. The key point is that the high-probability event can be made uniform over all admissible choices of the indices.

Technically, we first discretize the interval $[0,2]$ for each matrix type and layer, and combine this discretization with a dyadic peeling argument over the realized Schatten-quantity radii. A weighted union bound then gives simultaneous guarantees over all shells and grid points. Finally, each continuous index is rounded upward to the grid, and the resulting changes in the bounds are controlled using the assumption of Theorem~\ref{thm:GenGapPostHocSimplified}. This allows us to pass from the finite grid to the full continuum $[0,2]^{3L}$. A detailed argument is deferred to Appendix~\ref{subsec:GenGapPostHoc}.

\section{Examples}\label{sec:examples}

The main result of this paper is Theorem~\ref{thm:GenGapPostHocSimplified}, which provides post hoc bounds that allow the Schatten indices to be chosen after training, separately for each matrix type and each layer. In this section, however, we use the common-$p$ bounds in Theorem~\ref{thm:GenGapCommonPSimplified} in order to compare our bounds with existing bounds under matched constraints. Specializing the common Schatten index $p$, Theorem~\ref{thm:GenGapCommonPSimplified} yields Frobenius-type, rank-type, and spectral-norm-only regimes, which correspond to the three columns of Table~\ref{tab:GenBoundComparison}. These examples show how our spectral complexity improves the depth dependence, and can also improve the hidden-dimension dependence in some of the regimes.

\begin{table}[t]
    \centering
    \caption{Comparison of the leading complexity factor $B$ in the generalization gap bounds $\mathrm{GAP}(f_{\mathrm{out}})=\tilde{O}(B/\sqrt{n})$ under the matched constraints. The three columns correspond, respectively, to Frobenius norm constraints, rank constraints, and spectral norm constraints only. All other settings remain the same as in Table~\ref{tab:GenBoundComparisonOriginal}. For the norm-based baselines, the displayed rates follow from the inequalities $\|W\|_{2,1}\le\sqrt{N}\|W\|_{F}\le\sqrt{N\operatorname{rank}(W)}\|W\|_{2}$ and $\|W\|_{1,1}\le N\|W\|_{F}\le N\sqrt{\operatorname{rank}(W)}\|W\|_{2}$ for every $W\in\mathbb{R}^{N\times N}$. We note that, although these inequalities are sharp in the worst case (think of the case, for example, $W=N^{-1}\mathbf{1}_{N}\mathbf{1}_{N}^{\top}$), they may be loose for particular trained weight matrices. Accordingly, the comparisons in the table do not imply pointwise dominance over the original mixed-norm bounds.}

    \label{tab:GenBoundComparison}
    \begin{tabular}{lccc}
        \toprule
        & $\|W\|_{F}\le C_{F}$
        & $\operatorname{rank}(W)\le r$
        & $\|W\|_2=O(1)$ only \\
        \midrule
        \textbf{Ours}
        & $\sqrt{C_{F}}C^{\frac{L}{2}}LN^{\frac{3}{4}}$
        & $\sqrt{rLN}$
        & $\sqrt{L}N$ \\
        \cite{edelman2022inductive}
        & $C_{F}C^{L}L^{\frac{3}{2}}\sqrt N$
        & $C^{L}L^{\frac{3}{2}}\sqrt{rN}$
        & $C^{L}L^{\frac{3}{2}}N$ \\
        \cite{trauger2024sequence}
        & $C_{F}C^{L}L^{\frac{3}{2}}N$
        & $C^{L}L^{\frac{3}{2}}\sqrt{r}N$
        & $C^{L}L^{\frac{3}{2}}N^{\frac{3}{2}}$ \\
        \bottomrule
    \end{tabular}
\end{table}

\begin{example}[Frobenius and spectral norm constraints]
    Set $p^{\star,(\ell)}=2$ for all $\star\in\{QK,V,M\}$ and $\ell\in[L]$. This corresponds to the parameter class with simultaneous Frobenius and spectral norm constraints, $\mathcal{W}^{\star,(\ell)}(2,(C_{F}^{\star,(\ell)})^{2})=\{W\in\mathbb{R}^{N\times N}\mid\|W\|_{F}\le C_{F}^{\star,(\ell)},\ \|W\|_{2}\le C_{2}^{\star,(\ell)}\}$. Here, $C_{F}^{\star,(\ell)}$ denotes the Frobenius norm radius, so that $C_{\mathrm{s}}^{\star,(\ell)}=(C_{F}^{\star,(\ell)})^{2}$. Then, Theorem~\ref{thm:GenGapCommonPSimplified} gives
    \[
        \mathrm{GAP}(f_{\mathrm{out}})
        \lesssim C^{\frac{L}{2}}\bigg(\sum_{\ell=1}^{L}\sum_{\star\in\{QK,V,M\}}\sqrt{C_{F}^{\star,(\ell)}}\bigg)N^{\frac{3}{4}}\sqrt{\frac{\log(nT)}{n}}+\frac{(\log(n))^{\frac{3}{2}}}{\sqrt{n}}+\sqrt{\frac{\log(1/\delta)}{n}}
        .
    \]
    As summarized in Table~\ref{tab:GenBoundComparison}, if the Frobenius radii are uniformly bounded by $C_{F}$, our bounds improve the dependence on the layerwise propagation factor, depth, and the Frobenius radius relative to existing norm-based bounds. This improvement comes with a larger hidden-dimension factor relative to \citet{edelman2022inductive}, increasing from $N^{\frac{1}{2}}$ to $N^{\frac{3}{4}}$.
\end{example}

\begin{example}[Rank and spectral norm constraints]\label{eg:RankConstraint}
    Set $p^{\star,(\ell)}=0$ for all $\star\in\{QK,V,M\}$ and $\ell\in[L]$. This corresponds to imposing rank constraints together with spectral norm constraints on the parameter class, which can be written as $\mathcal{W}^{\star,(\ell)}(0,r^{\star,(\ell)})=\{W\in\mathbb{R}^{N\times N}\mid\operatorname{rank}(W)\le r^{\star,(\ell)},\ \|W\|_{2}\le C_{2}^{\star,(\ell)}\}$. Here, $r^{\star,(\ell)}$ denotes the rank bound, rather than $C_{\mathrm{s}}^{\star,(\ell)}$. Theorem~\ref{thm:GenGapCommonPSimplified} gives
    \[
        \mathrm{GAP}(f_{\mathrm{out}})
        \lesssim\bigg(\sum_{\ell=1}^{L}\sum_{\star\in\{QK,V,M\}}r^{\star,(\ell)}\bigg)^{\frac{1}{2}}\sqrt{\frac{N\log(nT)}{n}}+\frac{(\log(n))^{\frac{3}{2}}}{\sqrt{n}}+\sqrt{\frac{\log(1/\delta)}{n}}
        .
    \]
    When the ranks are uniformly bounded by $r$, Table~\ref{tab:GenBoundComparison} shows that our bounds remove the layerwise propagation factor $C^{L}$ and reduce the polynomial depth dependence from $L^{\frac{3}{2}}$ to $\sqrt{L}$, compared to the norm-based baselines.
\end{example}

\begin{example}[Spectral norm constraints only]
    Finally, consider the parameter class in Eq.~\eqref{eq:ParSpectralConst}, where only the layerwise spectral norm constraints are imposed. Since every $N\times N$ matrix has rank at most $N$, Example~\ref{eg:RankConstraint} can be applied with $r^{\star,(\ell)}=N$. This yields
    \[
        \mathrm{GAP}(f_{\mathrm{out}})
        \lesssim\frac{N\sqrt{L\log(nT)}}{\sqrt{n}}+\frac{(\log(n))^{\frac{3}{2}}}{\sqrt{n}}+\sqrt{\frac{\log(1/\delta)}{n}}
        .
    \]
    As shown in Table~\ref{tab:GenBoundComparison}, our bounds improve the depth dependence from the existing $C^{L}L^{\frac{3}{2}}$ scaling to $\sqrt{L}$.
\end{example}

\section{Empirical comparison of BERT-adapted proxies}\label{sec:experiments}

We use the publicly released BERT Miniatures checkpoints of \citet{turc2019well} to compare how the leading complexity factors suggested by our theory scale for trained Transformer weights.\footnote{We do not train any new models; all quantities are computed from the released checkpoints.} These checkpoints are based on the BERT encoder architecture \citep{devlin2019bert}, rather than the simplified single-head Transformer model analyzed in this paper. We therefore do not interpret the resulting quantities as numerical evaluations of the theorems for BERT. Instead, we extract the leading complexity factors from our bounds and from the norm-based bounds of \citet{edelman2022inductive}, adapt these factors to the BERT architecture, and compare the resulting BERT-adapted proxies. These proxies are computed directly from the realized Schatten quantities and mixed $(2,1)$-norms of the checkpoint weights, rather than from the generic norm inequalities used to derive the matched-constraint comparisons in Table~\ref{tab:GenBoundComparison}. The precise construction of these proxies is given in Appendix~\ref{sec:EmpiricalComparisonDiscussion}.

\begin{figure}[t]
    \centering
    \includegraphics[width=0.49\linewidth]{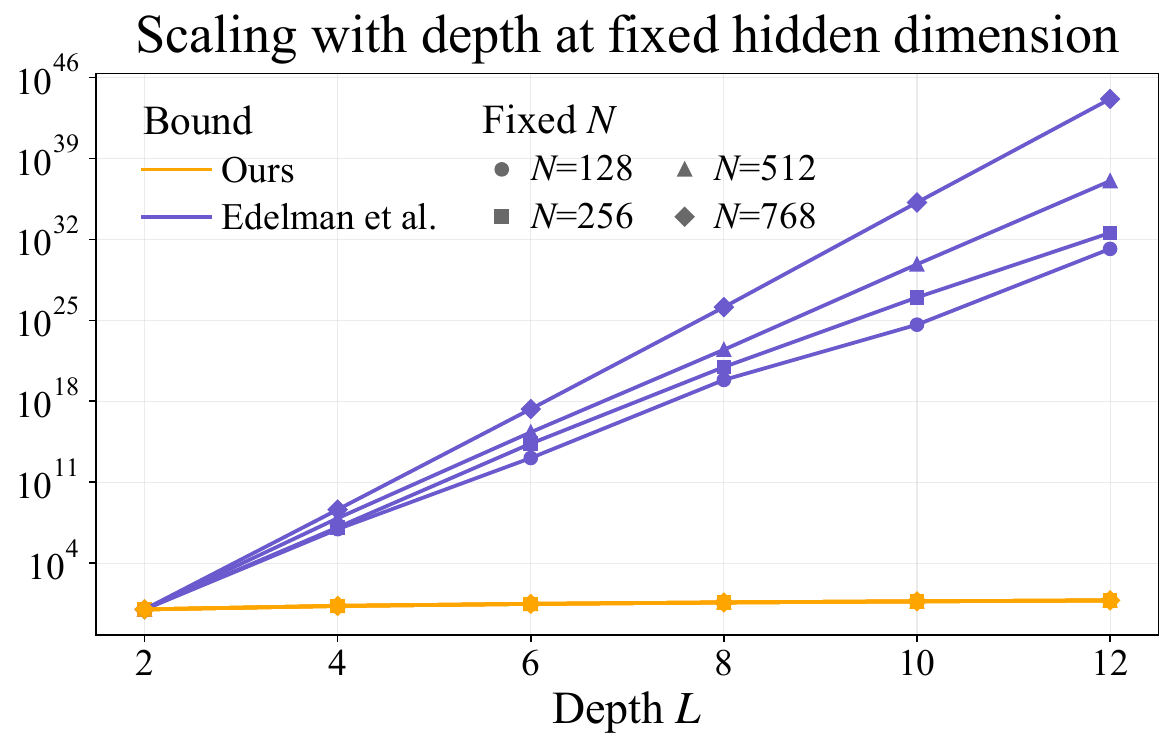}
    \includegraphics[width=0.49\linewidth]{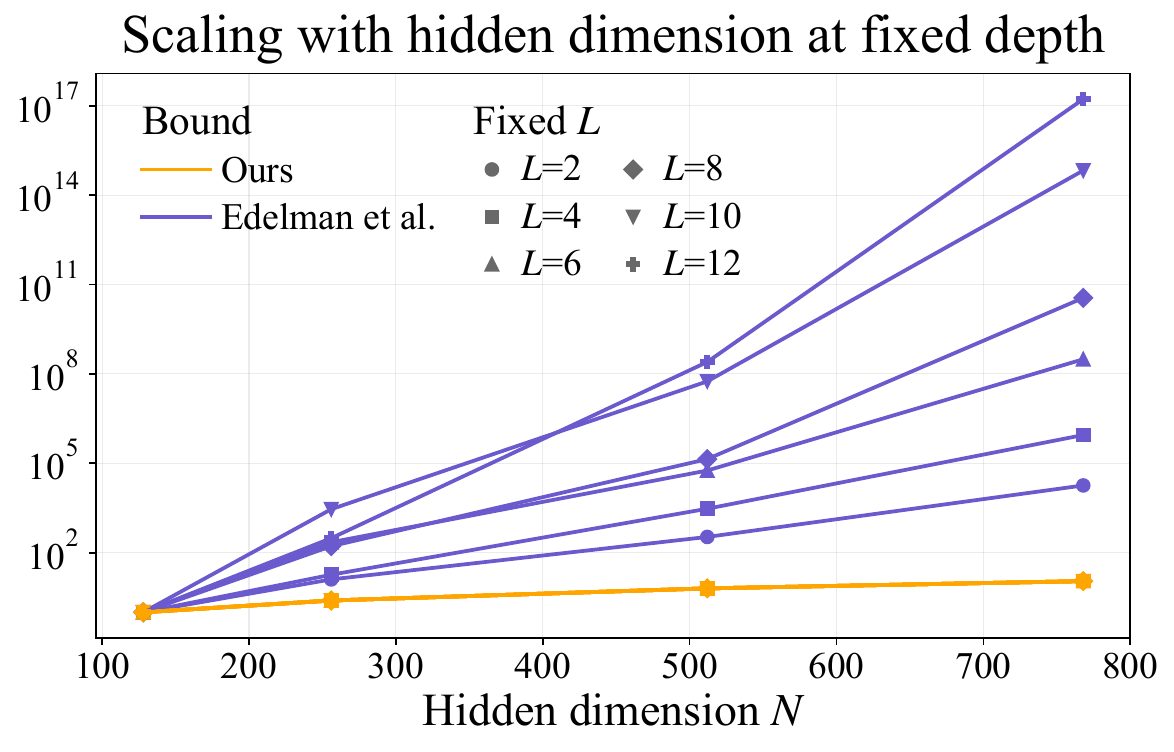}
    \caption{Comparison of the normalized BERT-adapted leading-factor proxies for our spectrum-adaptive bounds and the norm-based bounds of \citet{edelman2022inductive}. Each curve is normalized by its own value at the smallest value of the varying architectural parameter: $L=2$ at the corresponding fixed $N$ in the left panel, and $N=128$ at the corresponding fixed $L$ in the right panel. \textbf{Left:} scaling with the depth $L$ at fixed hidden dimension $N$. \textbf{Right:} scaling with the hidden dimension $N$ at fixed depth $L$. In both regimes, our proxies grow more slowly than the proxies based on \citet{edelman2022inductive}.}
    \label{fig:ObservedBoundComparison}
\end{figure}

In Figure~\ref{fig:ObservedBoundComparison}, each curve is normalized by its own value at the smallest value of the varying architectural parameter: $L=2$ at the corresponding fixed $N$ in the left panel, and $N=128$ at the corresponding fixed $L$ in the right panel. The relevant comparison is therefore the relative growth rate as $L$ or $N$ increases, rather than the absolute vertical scale. The figure shows that our post hoc proxies increase more slowly than the norm-based proxies both as the depth $L$ increases at fixed hidden dimension $N$ and as the hidden dimension $N$ increases at fixed depth $L$.

Notably, for every BERT Miniatures checkpoint considered, the optimized proxy selects the rank-based endpoint $p=0$ for every matrix type, layer, and attention head. Thus, while our theory permits a separate post hoc choice of Schatten index for each matrix, these experiments do not demonstrate the empirical selection of intermediate indices $p\in(0,2)$. The endpoint selection is partly explained by the BERT adaptation: the headwise query-key and value-output matrices have rank at most the head dimension $d_{h}=64$. For the feedforward matrices, which exhibit slower singular-value decay, the potential benefit of a positive Schatten index is outweighed by the corresponding depth- and dimension-dependent factors. Consequently, Figure~\ref{fig:ObservedBoundComparison} primarily illustrates improved scaling of the leading complexity proxies at the rank-based endpoint, rather than empirical adaptivity across Schatten indices.\footnote{Appendix~\ref{sec:PolynomialDecay} provides a controlled example with polynomially decaying singular values in which an intermediate Schatten index $p\in(0,2)$ is optimal.} Further diagnostics are provided in Appendix~\ref{subsec:EmpiricalComparisonInterpret}.

Finally, the empirical comparison should be interpreted with several limitations. Since both bounds involve unspecified multiplicative constants and suppressed logarithmic factors, the normalized proxies do not establish pointwise dominance of the complete generalization bounds. Moreover, the absolute bounds are likely to be numerically vacuous at the scales considered here, particularly because the constants arising from spectral norm and Lipschitz control may be substantial. The released checkpoints also do not provide the corresponding empirical training losses needed to evaluate observed generalization gaps. Therefore, our experiments assess only the relative scaling of leading complexity proxies, rather than the numerical tightness or non-vacuity of the generalization bounds.

\section{Conclusion}\label{sec:Conclusion}

We derived spectrum-adaptive post hoc generalization bounds for a simplified class of multi-layer Transformers under layerwise spectral norm and readout norm constraints. The main feature of the bounds is that the Schatten indices can be selected after training, separately for each layer and each matrix type. This allows the complexity measures to adapt to the learned singular-value profiles of the query-key, value, and feedforward matrices, rather than imposing a single fixed complexity measure across the whole network. The resulting bounds interpolate between rank-based and norm-based regimes, and show improved dependence on depth in representative spectral settings. Our empirical comparison using BERT Miniatures checkpoints shows that the BERT-adapted proxies derived from our bounds grow more slowly with depth and hidden dimension than the corresponding norm-based proxies.

Several limitations remain. First, our analysis relies on layerwise spectral norm and readout norm constraints and considers a simplified Transformer architecture that omits components of modern Transformers, including residual connections, multi-head attention, and LayerNorm/RMSNorm. Extending the framework to these components, particularly normalization layers, requires additional control of approximation errors and their propagation across layers. Second, the absolute generalization bounds may be numerically vacuous in practical settings. Finally, our analysis is agnostic to the training process, and uses the learned weights only through post hoc spectral quantities. A more complete explanation of Transformer generalization may require incorporating information about the optimization trajectory, implicit regularization, and the evolution of representations during training. Furthermore, our bounds do not exploit potentially low-dimensional structure in intermediate activations. Combining such data-dependent structure with the spectral properties of the learned weights could lead to sharper and more explanatory bounds.

\section*{Acknowledgements}

Mana Sakai was supported by RIKEN Junior Research Associate Program.
Masaaki Imaizumi was supported by JSPS KAKENHI (Grant No. 24K02904), JST CREST (Grant No. JPMJCR21D2), JST FOREST (Grant No. JPMJFR216I), JST ASPIRE (Grant No. JPMJAP2329), and JST BOOST (Grant No. JPMJBY24A9).

\bibliographystyle{chicago}
\bibliography{main}

\vfill
\pagebreak

\appendix

\section{Additional related work}\label{sec:RelatedWorks}

\subsection{Generalization bounds for Transformers}

The closest line of work studies complexity-based generalization bounds for Transformers. \citet{edelman2022inductive} derive covering number and Rademacher complexity bounds for multi-layer self-attention networks under norm constraints. Their bounds show that bounded-norm self-attention can have only logarithmic dependence on the token length and the hidden dimension, but the resulting complexity is controlled by fixed mixed $(2,1)$-norm radii and by layerwise propagation factors. \citet{trauger2024sequence} further develop this covering number approach and obtain norm-based bounds that are independent of the input token length under mixed $(1,1)$-norm constraints. These results provide the main norm-based baselines for our analysis. In contrast, our bounds use spectral complexity measures of the learned weight matrices and allow the Schatten indices to be chosen after training, separately for each layer and matrix type.

Several recent works refine or complement this complexity-based view. \citet{zhang2022analysis} analyze attention through exchangeability and latent-variable models and obtain token-length-independent generalization guarantees under exchangeability assumptions on the input tokens. This differs from our setting, where no exchangeability assumption is imposed and the token length enters only through logarithmic factors. A closely related work is \citet{truong2024rank}, which develops rank-dependent covering number bounds for linear function classes and applies them to single-layer Transformers. Their framework includes both prescribed column-subspace constraints and rank constraints, with the full-rank setting as a special case. Our focus is on post hoc selection of a separate Schatten index for each layer and matrix type. \citet{li2026sharper} use offset Rademacher complexity to derive fast excess risk bounds for Transformers from suitable covering number bounds. Our covering number bounds are complementary to this direction: in principle, they can be combined with the offset Rademacher complexity as in \citet{li2026sharper}, whereas the present paper focuses on post hoc generalization gap bounds.

Other theoretical studies focus on more specialized Transformer settings. \citet{fu2023what} analyze a random-feature attention model with randomly sampled and frozen query-key parameters and trainable value parameters, and obtain excess risk bounds. \citet{li2023transformers} study in-context learning by viewing a Transformer as implementing an algorithm at inference time, and relate generalization to stability in multitask learning. \citet{wei2022statistically} study statistically meaningful approximation of Turing machines by Transformers and derive sample-complexity guarantees in that framework. \citet{huang2025formal} develop a formal theory of length generalization for causal Transformers with learnable absolute positional encodings. \citet{mwigo2026generalization} analyze a shallow Transformer trained by gradient descent under a bounded-drift regime in which the parameters remain close to initialization. These works identify important mechanisms in particular tasks, architectures, or training regimes, whereas our bounds are training-agnostic, apply to multi-layer Transformers, and are expressed directly in terms of the spectra of the trained weights.

\subsection{Norm-based and spectral generalization bounds for neural networks}

Early norm-based generalization bounds for DNNs include \citet{neyshabur2015norm}, \citet{bartlett2017spectrally}, and \citet{golowich2018size}. This line of work explains how large networks can admit dimension-favorable bounds when their learned weights have controlled norms. The Transformer bounds of \citet{edelman2022inductive} and \citet{trauger2024sequence} can be viewed as architecture-specific continuations of this program. Our contribution follows the same complexity-based philosophy, but it replaces fixed coordinatewise norm constraints by spectral quantities that can be evaluated and optimized post hoc.

A related approach derives generalization guarantees through compressibility. Compression-based bounds show that networks admitting short descriptions or accurate compressed approximations can generalize well \citep{arora2018stronger,baykal2019data,suzuki2019compression}. In particular, \citet{suzuki2019compression} convert compression guarantees into bounds for the original non-compressed network, making compressibility itself a statistical complexity measure. This viewpoint is close in spirit to ours, because fast singular-value decay can be interpreted as a form of spectral compressibility. However, our bounds do not require constructing an explicit compressed network; instead, the complexity is expressed directly through Schatten quantities of the trained weights.

PAC-Bayesian methods provide a complementary approach to generalization analysis. \citet{dziugaite2017computing} obtain non-vacuous certificates by optimizing a posterior over stochastic neural-network predictors, while \citet{zhou2019nonvacuous} combine PAC-Bayesian analysis with compression to certify compressed networks.\footnote{Numerical comparisons with the certificates of \citet{dziugaite2017computing} and \citet{zhou2019nonvacuous} would require additional posterior optimization or compression for each model, whereas our experiments evaluate complexity proxies directly from released checkpoints. We therefore leave such comparisons for future work.} These approaches differ from our direct generalization-gap bound for a given deterministic Transformer. Nevertheless, \citet{neyshabur2018pac} show that PAC-Bayesian guarantees can also be transferred to the original deterministic predictor by controlling changes in network outputs under random weight perturbations. Exploring whether such methods can incorporate the learned spectral structure of Transformer weights while retaining post hoc adaptivity would be an interesting direction for future work.

Low-rank and spectral structures have also been studied more directly in neural-network generalization. \citet{pinto2025generalization} derive Gaussian-complexity bounds for networks with low-rank layers and show how low-rank constraints can mitigate the accumulation of dimension-dependent factors across depth. \citet{ledent2025generalization} obtain post hoc generalization bounds using Schatten (quasi) norms. Their analysis is technically close to ours and provides the main inspiration for the spectrum-adaptive aspect of our bounds. We extend this idea to Transformer architectures, where the proof must control matrix-valued function classes under the $\|\cdot\|_{2\to\infty}$ metric.

\section{Mathematical preliminaries}\label{sec:MathPrelim}

\subsection{Matrix norms}

For a matrix $W$, we use the notation summarized in Table~\ref{tab:MatrixNormNotation}. Here, $\sigma_{i}(W)$ denotes the $i$-th largest singular value of $W$, and $W_{i\cdot}$ denotes the $i$-th row of $W$.

\begin{table}[h]
    \centering
    \caption{Matrix norm notation used throughout the paper.}
    \begin{tabular}{lll}
        \hline
        Notation & Definition & Description \\
        \hline
        $\|W\|_{F}$ & $(\sum_{i,j}W_{ij}^{2})^{\frac{1}{2}}$ & Frobenius norm \\
        $\|W\|_{\alpha\to\beta}$ & $\sup\{\|Wx\|_{\beta}:\|x\|_{\alpha}\le1\}$ & Induced norm from $\ell_{\alpha}$ to $\ell_{\beta}$ \\
        $\|W\|_{\alpha}$ & $\|W\|_{\alpha\to\alpha}$ & Induced norm on $\ell_{\alpha}$ \\
        $\|W\|_{2}$ & $\sigma_{1}(W)$ & Spectral norm \\
        $\|W\|_{2\to\infty}$ & $\max_{i}\|W_{i\cdot}\|_{2}$ & Maximum row Euclidean norm \\
        $\|W\|_{\alpha,\beta}$ & $[\sum_{j}(\sum_{i}|W_{ij}|^{\alpha})^{\frac{\beta}{\alpha}}]^{\frac{1}{\beta}}$ & Mixed $(\alpha,\beta)$-norm \\
        \hline
    \end{tabular}
    \label{tab:MatrixNormNotation}
\end{table}

For $\alpha,\beta\ge1$, the induced norm from $\ell_{\alpha}$ to $\ell_{\beta}$ is denoted by $\|W\|_{\alpha\to\beta}$. When $\alpha=\beta$, we simply write $\|W\|_{\alpha}:=\|W\|_{\alpha\to\alpha}$. In particular, $\|W\|_{2}=\sigma_{1}(W)$ is the spectral norm. The identity $\|W\|_{2\to\infty}=\max_{i}\|W_{i\cdot}\|_{2}$ is used throughout the paper.

We also use the following Schatten $p$ (quasi) norm.

\begin{definition}
    For $p\in(0,2]$, the Schatten $p$ (quasi) norm $\|W\|_{\mathrm{s},p}$ of a matrix $W$ is defined by $\|W\|_{\mathrm{s},p}=(\sum_{i}\sigma_{i}(W)^{p})^{\frac{1}{p}}$. We refer to $p$ as the Schatten index. For $p=0$, we use the convention $\|W\|_{\mathrm{s},0}^{0}:=\lim_{p\downarrow0}\|W\|_{\mathrm{s},p}^{p}=\operatorname{rank}(W)$.
\end{definition}

\subsection{Some basic results}

The following results are some well-known matrix norm inequalities.

\begin{lemma}\label{lem:MatrixProduct2InfInf}
    $\|AB\|_{2\to\infty}\le\|A\|_{\infty}\|B\|_{2\to\infty}$.
\end{lemma}

\begin{lemma}\label{lem:MatrixProduct2Inf2}
    $\|AB\|_{2\to\infty}\le\|A\|_{2\to\infty}\|B\|_{2}$.
\end{lemma}

\begin{lemma}\label{lem:MatrixProduct2Frob}
    \begin{enumerate}\renewcommand{\labelenumi}{(\roman{enumi})}
        \item $\|AB\|_{F}\le\|A\|_{F}\|B\|_{2}$.
        \item $\|AB\|_{F}\le\|A\|_{2}\|B\|_{F}$.
    \end{enumerate}
\end{lemma}

\begin{lemma}\label{lem:VecMatrixProduct}
    For any $a\in\mathbb{R}^{\ell}$ and $B\in\mathbb{R}^{\ell\times m}$, we have
    \[
        \|a^{\top}B\|_{2}
        \le\|a\|_{1}\|B\|_{2\to\infty}
        .
    \]
\end{lemma}
\begin{proof}
    For $a=0$, the claim is immediate. For $a\neq0$, we have
    \[
        \|a^{\top}B\|_{2}
        =\|B^{\top}a\|_{2}
        =\left\|B^{\top}\frac{a}{\|a\|_{1}}\right\|_{2}\|a\|_{1}
        \le \|B^{\top}\|_{1\to2}\|a\|_{1}
        =\|B\|_{2\to\infty}\|a\|_{1}
        .
    \]
    Here, the last equality follows from the duality relation $\|A\|_{2\to\infty}=\|A^{\top}\|_{1\to2}$.
\end{proof}

We also state some useful properties of the softmax function and the rowwise projection operator.

\begin{lemma}[\citet{edelman2022inductive}, Lemma~A.9]\label{lem:Edelman2022LemA9}
    Suppose $\Pi_{\mathrm{norm}}$ is the rowwise projection operator onto the unit ball. Then, for any $Z,Z'\in\mathbb{R}^{T\times N}$, we have
    \[
        \|\Pi_{\mathrm{norm}}(Z)-\Pi_{\mathrm{norm}}(Z')\|_{2\to\infty}
        \le\|Z-Z'\|_{2\to\infty}
        .
    \]
\end{lemma}

\begin{lemma}\label{lem:SoftmaxMatrixInfinity}
    For any $G\in\mathbb{R}^{T\times T}$, the softmax function satisfies
    \[
        \|\mathrm{SoftMax}(G)\|_{\infty}
        =\max_{j\in[T]}\sum_{k=1}^{T}\mathrm{SoftMax}_{k}(G_{j\cdot})
        =1
        .
    \]
\end{lemma}

\begin{lemma}[\citet{edelman2022inductive}, Corollary~A.7]\label{lem:SoftmaxLipschitzVector1}
    For any $x,y\in\mathbb{R}^{T}$, it holds that
    \[
        \|\mathrm{SoftMax}(x)-\mathrm{SoftMax}(y)\|_{1}
        \le 2\|x-y\|_{\infty}
        .
    \]
\end{lemma}

The following lemma is used when we optimize the allocation of covering radii in the covering entropy bounds of Transformers.

\begin{lemma}\label{lem:OptimizationPolynomial}
    Fix $a_{i},b_{i},c,\nu>0$. The unique solution to the optimization problem
    \[
        \min_{z_{1},\ldots,z_{m}>0}\sum_{i=1}^{m}a_{i}z_{i}^{-\nu}
        \qquad\textnormal{subject to}\quad
        \sum_{i=1}^{m}b_{i}z_{i}=c
    \]
    is given by
    \[
        z_{i}^{\ast}
        =\frac{ca_{i}^{\frac{1}{\nu+1}}b_{i}^{-\frac{1}{\nu+1}}}{\sum_{j=1}^{m}a_{j}^{\frac{1}{\nu+1}}b_{j}^{\frac{\nu}{\nu+1}}}
        \qquad(i\in[m])
        .
    \]
    Moreover, the minimum value is
    \[
        \sum_{i=1}^{m}a_{i}(z_{i}^{\ast})^{-\nu}
        =\frac{1}{c^{\nu}}\left(\sum_{i=1}^{m}a_{i}^{\frac{1}{\nu+1}}b_{i}^{\frac{\nu}{\nu+1}}\right)^{\nu+1}
        .
    \]
\end{lemma}
\begin{proof}
    Since
    \[
        \frac{\partial^{2}}{\partial z_{i}^{2}}(a_{i}z_{i}^{-\nu})
        =\nu(\nu+1)a_{i}z_{i}^{-\nu-2}
        >0
    \]
    holds for each $i\in[m]$, the objective function is strictly convex on $(0,\infty)^{m}$. Hence, over the affine constraint set, any feasible point satisfying the first-order condition is the unique global minimizer. Consider the Lagrangian
    \[
        \mathcal{L}(z_{1},\ldots,z_{m},\lambda)
        =\sum_{i=1}^{m}a_{i}z_{i}^{-\nu}+\lambda\left(\sum_{i=1}^{m}b_{i}z_{i}-c\right)
        .
    \]
    The first-order condition with respect to $z_{i}$ is
    \[
        \frac{\partial\mathcal{L}}{\partial z_{i}}
        =-\nu a_{i}z_{i}^{-\nu-1}+\lambda b_{i}
        =0
        ,
    \]
    which implies
    \[
        z_{i}
        =\left(\frac{\nu a_{i}}{\lambda b_{i}}\right)^{\frac{1}{\nu+1}}
        \qquad(i\in[m])
        .
    \]
    Substituting this into the constraint yields
    \[
        \sum_{i=1}^{m}b_{i}\left(\frac{\nu a_{i}}{\lambda b_{i}}\right)^{\frac{1}{\nu+1}}
        =\left(\frac{\nu}{\lambda}\right)^{\frac{1}{\nu+1}}\sum_{i=1}^{m}a_{i}^{\frac{1}{\nu+1}}b_{i}^{\frac{\nu}{\nu+1}}
        =c
        .
    \]
    Hence, we require
    \[
        \left(\frac{\nu}{\lambda}\right)^{\frac{1}{\nu+1}}
        =\frac{c}{\sum_{j=1}^{m}a_{j}^{\frac{1}{\nu+1}}b_{j}^{\frac{\nu}{\nu+1}}}
        .
    \]
    Combining the above results, we obtain
    \[
        z_{i}^{\ast}
        =\frac{c a_{i}^{\frac{1}{\nu+1}}b_{i}^{-\frac{1}{\nu+1}}}{\sum_{j=1}^{m}a_{j}^{\frac{1}{\nu+1}}b_{j}^{\frac{\nu}{\nu+1}}}
        \qquad(i\in[m])
        .
    \]
    Finally, the corresponding minimum value is given by
    \[
        \sum_{i=1}^{m}a_{i}(z_{i}^{\ast})^{-\nu}
        =\left(\frac{\sum_{j=1}^{m}a_{j}^{\frac{1}{\nu+1}}b_{j}^{\frac{\nu}{\nu+1}}}{c}\right)^{\nu}\sum_{i=1}^{m}a_{i}^{\frac{1}{\nu+1}}b_{i}^{\frac{\nu}{\nu+1}}
        =\frac{1}{c^{\nu}}\left(\sum_{i=1}^{m}a_{i}^{\frac{1}{\nu+1}}b_{i}^{\frac{\nu}{\nu+1}}\right)^{\nu+1}
        .
    \]
\end{proof}

\subsubsection{Covering numbers and generalization gap}

In this paper, we define covering numbers as follows.

\begin{definition}[Covering number of a general metric space]\label{def:CoveringNumber}
    Suppose $(\mathcal{A},\rho)$ is a metric space. $\mathcal{C}\subset\mathcal{A}$ is called an $\epsilon$-cover of $\mathcal{A}$ if for every $a\in\mathcal{A}$, there exists $a'\in\mathcal{C}$ such that $\rho(a,a')\le\epsilon$. The $\epsilon$-covering number $\mathcal{N}(\mathcal{A},\rho,\epsilon)$ is the minimum cardinality of any $\epsilon$-cover $\mathcal{C}$ of $\mathcal{A}$.
\end{definition}

\begin{definition}[Covering number of a function space]\label{def:CoveringNumberFunctionSpace}
    Suppose $\mathcal{F}$ is a class of maps from $\mathcal{X}$ to $\mathcal{Y}$, where $\mathcal{Y}$ is equipped with a metric $\rho$. For inputs $\{x_{i}\}_{i\in[n]}\subset\mathcal{X}$, the covering number $\mathcal{N}_{\infty}(\mathcal{F},\rho,\epsilon;\{x_{i}\}_{i\in[n]})$ is the minimum cardinality of any $\epsilon$-cover $\mathcal{C}\subset\mathcal{F}$ such that for all $f\in\mathcal{F}$, there exists $f'\in\mathcal{C}$ satisfying
    \[
        \max_{i\in[n]}\rho(f(x_{i}),f'(x_{i}))
        \le\epsilon
        .
    \]
\end{definition}

To convert the entropy bounds into generalization bounds, we recall a standard route through empirical Rademacher complexity.

\begin{definition}[Empirical Rademacher complexity]\label{def:EmpiricalRademacherComplexity}
    Let $\mathcal{F}$ be a class of real-valued functions on $\mathcal{X}$. For inputs $\{x_{i}\}_{i\in[n]}\subset\mathcal{X}$, define the empirical Rademacher complexity by
    \[
        \hat{\mathfrak{R}}_{n}(\mathcal{F};\{x_{i}\}_{i\in[n]})
        =\mathbb{E}\left[\left.\sup_{f\in\mathcal{F}}\frac{1}{n}\sum_{i=1}^{n}\sigma_{i}f(x_{i})\right|\{x_{i}\}_{i\in[n]}\right]
        ,
    \]
    where $\sigma=(\sigma_{1},\dots,\sigma_{n})$ is a vector of independent Rademacher random variables.
\end{definition}

\begin{lemma}[Dudley-type entropy integral; cf.\ \citet{edelman2022inductive}, Lemma~A.2; \citet{dudley1967sizes}]\label{lem:DudleyEntropyIntegral}
    Let $\mathcal{F}$ be a class of real-valued functions on $\mathcal{X}$. Suppose $|f(x)|\le A$ holds for all $f\in\mathcal{F}$ and all $x\in\mathcal{X}$. Then, we have
    \[
        \hat{\mathfrak{R}}_{n}(\mathcal{F};\{x_{i}\}_{i\in[n]})
        \lesssim\inf_{\alpha>0}\left(\alpha+\int_{\alpha}^{A}\sqrt{\frac{\log\mathcal{N}_{\infty}(\mathcal{F},|\cdot|,\epsilon;\{x_{i}\}_{i\in[n]})}{n}}d\epsilon\right)
        .
    \]
\end{lemma}

\begin{lemma}\label{lem:RademacherPowerLawSum}
    Fix an integer $J\ge1$, constants $C_{i}>0\ (i\in[J])$, $\nu_{i}\in[0,2)\ (i\in[J-1])$, and $\nu_{J}=2$. Suppose $\log\mathcal{N}_{\infty}(\mathcal{F},|\cdot|,\epsilon;\{x_{i}\}_{i\in[n]})\lesssim\sum_{i=1}^{J}C_{i}\epsilon^{-\nu_{i}}$ holds for all $\epsilon\in(0,A]$. Then, we have
    \begin{align*}
        \hat{\mathfrak{R}}_{n}(\mathcal{F};\{x_{i}\}_{i\in[n]})
        \lesssim\frac{1}{\sqrt{n}}\left(\sum_{i=1}^{J-1}\frac{A^{1-\nu_{i}/2}}{1-\nu_{i}/2}\sqrt{C_{i}}+\left[1+\log\left(1+\frac{A\sqrt{n}}{\sqrt{C_{J}}}\right)\right]\sqrt{C_{J}}\right)
        .
    \end{align*}
\end{lemma}
\begin{proof}
    Using $(\sum_{i=1}^{J}a_{i})^{1/2}\le\sum_{i=1}^{J}\sqrt{a_{i}}$ for $a_{i}\ge0$, we have
    \begin{align*}
        \hat{\mathfrak{R}}_{n}(\mathcal{F};\{x_{i}\}_{i\in[n]})
        &\lesssim\inf_{0<\alpha\le A}\left(\alpha+\int_{\alpha}^{A}\sqrt{\frac{\sum_{i=1}^{J}C_{i}\epsilon^{-\nu_{i}}}{n}}d\epsilon\right)\\
        &\le\inf_{0<\alpha\le A}\left(\alpha+\sum_{i=1}^{J}\sqrt{\frac{C_{i}}{n}}\int_{\alpha}^{A}\epsilon^{-\nu_{i}/2}d\epsilon\right)
        .
    \end{align*}
    Note that each integral can be bounded as
    \begin{align*}
        &\int_{\alpha}^{A}\epsilon^{-\nu_{i}/2}d\epsilon
        =\frac{A^{1-\nu_{i}/2}-\alpha^{1-\nu_{i}/2}}{1-\nu_{i}/2}
        \le\frac{A^{1-\nu_{i}/2}}{1-\nu_{i}/2}
        \qquad(i\in[J-1])
        ,\\
        &\int_{\alpha}^{A}\epsilon^{-\nu_{J}/2}d\epsilon
        =\int_{\alpha}^{A}\epsilon^{-1}d\epsilon
        =\log(A/\alpha)
        .
    \end{align*}
    Hence, we have
    \begin{align*}
        \hat{\mathfrak{R}}_{n}(\mathcal{F};\{x_{i}\}_{i\in[n]})
        \lesssim\inf_{0<\alpha\le A}\left[\alpha+\frac{1}{\sqrt{n}}\left(\sum_{i=1}^{J-1}\frac{A^{1-\nu_{i}/2}}{1-\nu_{i}/2}\sqrt{C_{i}}+\sqrt{C_{J}}\log(A/\alpha)\right)\right]
        .
    \end{align*}
    Taking $\alpha=\min\{A,\sqrt{C_{J}/n}\}$ proves the claim.
\end{proof}

\begin{lemma}[\citet{bartlett2002rademacher}]\label{lem:RademacherGeneralization}
    Let $\mathcal{D}$ be a probability distribution on $\mathcal{X}\times\mathbb{R}$, and let $\mathcal{L}:\mathbb{R}\times\mathbb{R}\to\mathbb{R}$ be a $B_{\mathcal{L}}$-bounded loss function that is $L_{\mathcal{L}}$-Lipschitz in its first argument. For each $f\in\mathcal{F}$, define the population risk and empirical risk by
    \[
        \mathcal{R}(f)
        =\mathbb{E}_{(X,Y)\sim\mathcal{D}}[\mathcal{L}(f(X),Y)]
        \qquad
        \text{and}\qquad
        \hat{\mathcal{R}}_{n}(f)
        =\frac{1}{n}\sum_{i=1}^{n}\mathcal{L}(f(x_{i}),Y_{i})
        ,
    \]
    respectively, where $\{(x_{i},Y_{i})\}_{i\in[n]}$ are i.i.d. samples from $\mathcal{D}$. Then, for any $\delta>0$, with probability at least $1-\delta$, the bound
    \[
        |\mathcal{R}(f)-\hat{\mathcal{R}}_{n}(f)|
        \lesssim L_{\mathcal{L}}\hat{\mathfrak{R}}_{n}(\mathcal{F};\{x_{i}\}_{i\in[n]})+B_{\mathcal{L}}\sqrt{\frac{\log(1/\delta)}{n}}
        ,
    \]
    holds simultaneously for all $f\in\mathcal{F}$.
\end{lemma}

\section{Metric entropy bounds}

\subsection{Basic covering number bounds for matrices}

We first collect elementary covering number bounds for matrix classes that will be used as ingredients in the later interpolation argument.

\begin{fact}[\citet{vershynin2018high}, Corollary~4.2.13]\label{fact:EntropyEuclideanBall}
    Let $\mathcal{B}=\left\{a\in\mathbb{R}^{\ell}\mid\|a\|_{2}\le1\right\}$ be the Euclidean unit ball. Then, we have
    \[
        \ell\log\frac{1}{\epsilon}
        \le\log\mathcal{N}\left(\mathcal{B},\|\cdot\|_{2},\epsilon\right)
        \le\ell\log\left(\frac{2}{\epsilon}+1\right)
        .
    \]
\end{fact}

By scaling, the same result implies that, for $\mathcal{A}=\left\{a\in\mathbb{R}^{\ell}\mid\|a\|_{2}\le C\right\}$,
\[
    \ell\log\frac{C}{\epsilon}
    \le\log\mathcal{N}\left(\mathcal{A},\|\cdot\|_{2},\epsilon\right)
    \le\ell\log\left(\frac{2C}{\epsilon}+1\right)
    .
\]

\begin{corollary}\label{cor:EntropyEuclideanBall}
    Define
    \[
        \mathcal{W}
        =\left\{W\in\mathbb{R}^{\ell\times m}\mid\|W\|_{F}\le C\right\}
        ,
    \]
    where $\|\cdot\|_{F}$ denotes the Frobenius norm. Then, we have
    \[
        \ell m\log\frac{C}{\epsilon}
        \le\log\mathcal{N}\left(\mathcal{W},\|\cdot\|_{F},\epsilon\right)
        \le\ell m\log\left(\frac{2C}{\epsilon}+1\right)
        .
    \]
\end{corollary}

\begin{proposition}[\citet{ledent2025generalization}, Proposition~F.1 (modified)]\label{prop:Ledent2025PropF1}
    Define
    \[
        \mathcal{W}=\left\{
        W\in\mathbb{R}^{\ell\times m}\mid\operatorname{rank}(W)\le r,\ \|W\|_{2}\le C\right\}
        ,
    \]
    where $\|\cdot\|_{2}$ denotes the operator norm and $r\le\min\{\ell,m\}$. Then, we have
    \[
        \log\mathcal{N}\left(\mathcal{W},\|\cdot\|_{F},\epsilon\right)
        \le(\ell+m)r\log\left(\frac{8C\sqrt{r}}{\epsilon}+1\right)
        .
    \]
\end{proposition}
\begin{proof}
    For $W\in\mathcal{W}$, we can write $W=W_{1}W_{2}^{\top}$ with $W_{1}\in\mathbb{R}^{\ell\times r}$ and $W_{2}\in\mathbb{R}^{m\times r}$ satisfying
    \[
        \|W_{1}\|_{2}\le\sqrt{C},\quad\|W_{2}\|_{2}\le\sqrt{C}
        .
    \]
    It follows that for any $W,W'\in\mathcal{W}$,
    \begin{align*}
        &\|W-W'\|_{F}
        =\|W_{1}W_{2}^{\top}-W_{1}'W_{2}'^{\top}\|_{F}
        \le\|W_{1}(W_{2}-W_{2}')^{\top}\|_{F}+\|(W_{1}-W_{1}')W_{2}'^{\top}\|_{F}\\
        &\le\|W_{1}\|_{2}\|W_{2}-W_{2}'\|_{F}+\|W_{1}-W_{1}'\|_{F}\|W_{2}'\|_{2}
        \le\sqrt{C}\left(\|W_{1}-W_{1}'\|_{F}+\|W_{2}-W_{2}'\|_{F}\right)
        ,
    \end{align*}
    where the first inequality follows from the triangle inequality and the second inequality follows from Lemma~\ref{lem:MatrixProduct2Frob}. Set
    \[
        \mathcal{W}_{1}
        =\{W\in\mathbb{R}^{\ell\times r}\mid\|W\|_{2}\le\sqrt{C}\}
        ,\quad
        \mathcal{W}_{2}
        =\{W\in\mathbb{R}^{m\times r}\mid\|W\|_{2}\le\sqrt{C}
        \}
        .
    \]
    Then, by Lemma~\ref{lem:EntropyMatrixFrobenius}, we have
    \[
        \mathcal{N}\left(\mathcal{W}_{1},\|\cdot\|_{F},\frac{\epsilon}{2\sqrt{C}}\right)
        \le\left(\frac{8C\sqrt{r}}{\epsilon}+1\right)^{\ell r}
        ,\quad
        \mathcal{N}\left(\mathcal{W}_{2},\|\cdot\|_{F},\frac{\epsilon}{2\sqrt{C}}\right)
        \le\left(\frac{8C\sqrt{r}}{\epsilon}+1\right)^{mr}
        .
    \]
    Thus, we have
    \[
        \mathcal{N}\left(\mathcal{W},\|\cdot\|_{F},\epsilon\right)
        \le
        \mathcal{N}\left(\mathcal{W}_{1},\|\cdot\|_{F},\frac{\epsilon}{2\sqrt{C}}\right)
        \mathcal{N}\left(\mathcal{W}_{2},\|\cdot\|_{F},\frac{\epsilon}{2\sqrt{C}}\right)
        \le
        \left(\frac{8C\sqrt{r}}{\epsilon}+1\right)^{(\ell+m)r}
        .
    \]
\end{proof}

\begin{lemma}\label{lem:EntropyMatrixFrobenius}
    Define
    \[
        \mathcal{W}
        =\{W\in\mathbb{R}^{\ell\times m}\mid\|W\|_{2}\le C,\ \operatorname{rank}(W)\le r\}
        ,
    \]
    where $\|\cdot\|_{2}$ denotes the operator norm. Then, it holds that
    \[
        \log\mathcal{N}\left(\mathcal{W},\|\cdot\|_{F},\epsilon\right)
        \le \ell m\log\left(\frac{4C\sqrt{r}}{\epsilon}+1\right)
        .
    \]
\end{lemma}
\begin{proof}
    For any $W\in\mathcal{W}$, we have
    \[
        \|W\|_{F}
        =\left(\sum_{i=1}^{\operatorname{rank}(W)}\sigma_{i}^{2}(W)\right)^{1/2}
        \le C\sqrt{r}
        .
    \]
    Hence, $\mathcal{W}$ is contained in the Frobenius ball $\mathcal{B}_{F}=\{W\in\mathbb{R}^{\ell\times m}\mid\|W\|_{F}\le C\sqrt{r}\}$. By Corollary~\ref{cor:EntropyEuclideanBall}, there exists an $\epsilon/2$-cover
    $\bar{\mathcal{C}}$ of $\mathcal{B}_{F}$ with respect to the Frobenius norm such that
    \[
        |\bar{\mathcal{C}}|
        \le \left(\frac{4C\sqrt{r}}{\epsilon}+1\right)^{\ell m}
        .
    \]
    Since $\mathcal{W}\subset\mathcal{B}_{F}$, the same set $\bar{\mathcal{C}}$ is an external $\epsilon/2$-cover of $\mathcal{W}$ with respect to the Frobenius norm. Applying Lemma~\ref{lem:ProperCover} to this external cover yields a proper $\epsilon$-cover $\mathcal{C}\subset\mathcal{W}$ satisfying $|\mathcal{C}|\le|\bar{\mathcal{C}}|$. Therefore, we have
    \[
        \mathcal{N}\left(\mathcal{W},\|\cdot\|_{F},\epsilon\right)
        \le \left(\frac{4C\sqrt{r}}{\epsilon}+1\right)^{\ell m}
        .
    \]
\end{proof}

\begin{lemma}[Properization of external empirical covers]\label{lem:ProperCover}
    Let $\mathcal{F}$ be a function class. Suppose that $\bar{\mathcal{C}}$ is a finite set of functions, not necessarily contained in $\mathcal{F}$, such that for every $f\in\mathcal{F}$ there exists $\bar{f}\in\bar{\mathcal{C}}$ satisfying $\max_{i\in[n]}\rho(f(x_{i}),\bar{f}(x_{i}))\le\epsilon$. Then there exists a proper $2\epsilon$-cover $\mathcal{C}\subset \mathcal{F}$ of $\mathcal{F}$ with $|\mathcal{C}|\le|\bar{\mathcal{C}}|$.
\end{lemma}
\begin{proof}
Discard every $\bar{f}\in\bar{\mathcal{C}}$ for which $\{f\in\mathcal{F}\mid\max_{i\in[n]}\rho(f(x_{i}),\bar{f}(x_{i}))\le\epsilon\}$ is empty. For each remaining $\bar{f}$, choose one representative $f_{\bar{f}}\in\mathcal{F}$ satisfying $\max_{i\in[n]}\rho(f_{\bar{f}}(x_{i}),\bar{f}(x_{i}))\le\epsilon$, and set $\mathcal{C}=\{f_{\bar{f}}\mid\bar{f}\in\bar{\mathcal{C}}\}$. Then, for any $f\in\mathcal{F}$, choosing $\bar{f}\in\bar{\mathcal{C}}$ with $\max_{i\in[n]}\rho(f(x_{i}),\bar{f}(x_{i}))\le\epsilon$ gives
\[
    \max_{i\in[n]}\rho(f(x_{i}),f_{\bar{f}}(x_{i}))
    \le\max_{i\in[n]}\rho(f(x_{i}),\bar{f}(x_{i}))+\max_{i\in[n]}\rho(\bar{f}(x_{i}),f_{\bar{f}}(x_{i}))
    \le2\epsilon
    .
\]
Thus $\mathcal{C}$ is a proper $2\epsilon$-cover of $\mathcal{F}$.
\end{proof}

\begin{remark}\label{rem:NormInequalityCovering}
    Since $\|W\|_{2\to\infty}\le\|W\|_{2}\le\|W\|_{F}$ holds for any matrix $W$, we have
    \[
        \mathcal{N}\left(\mathcal{W},\|\cdot\|_{2\to\infty},\epsilon\right)\le\mathcal{N}\left(\mathcal{W},\|\cdot\|_{2},\epsilon\right)
        \le\mathcal{N}\left(\mathcal{W},\|\cdot\|_{F},\epsilon\right)
        .
    \]
\end{remark}

\begin{lemma}[\citet{vershynin2018high}, Lemma~4.4.1]\label{lem:OperatorNormNet1}
    Let $\mathcal{N}_{m}\subset\mathcal{S}^{m-1}$ be an $\epsilon$-net of the Euclidean unit sphere. Then, for any $W\in\mathbb{R}^{\ell\times m}$ and $\epsilon\in[0,1)$, we have
    \[
        \|W\|_{2}
        \le\frac{1}{1-\epsilon}\sup_{x\in\mathcal{N}_{m}}\|Wx\|_{2}
        .
    \]
\end{lemma}

\subsection{Entropy bounds for linear functions}

We next derive entropy bounds for linear function classes evaluated on a fixed sample. These results convert norm constraints on the weight matrix into uniform covering bounds for the induced matrix-valued maps under the $\|\cdot\|_{2\to\infty}$ metric.

\begin{proposition}[\citet{zhang2002covering}, Theorem~4]\label{prop:Zhang2002Thm4}
    Suppose $p\in[2,\infty)$ and $q\in\left[1,2\right]$ satisfy $\frac{1}{p}+\frac{1}{q}=1$. Define a class of linear functions by
    \[
        \mathcal{F}
        =\{f:\mathbb{R}^{d}\to\mathbb{R}\mid f(x)=x^{\top}w,\ w\in\mathbb{R}^{d},\ \|w\|_{q}\le a,\ \|x\|_{p}\le b\}
        .
    \]
    Then, for any $\epsilon>0$, we have
    \[
        \log\mathcal{N}_{\infty}(\mathcal{F},|\cdot|,\epsilon;\{x_{i}\}_{i\in[n]})
        \le36(p-1)\frac{a^{2}b^{2}}{\epsilon^{2}}\log\left(2\left\lceil\frac{4ab}{\epsilon}+2\right\rceil n+1\right)
        .
    \]
\end{proposition}

\begin{corollary}\label{cor:Zhang2002Thm4TwogInf}
    Define a class of matrix-valued functions by
    \[
        \mathcal{F}
        =\{f:\mathbb{R}^{d\times\ell}\to\mathbb{R}^{d\times m}\mid f(X)=XW,\ W\in\mathcal{W}\}
        ,
    \]
    where
    \[
        \mathcal{W}
        =\{W\in\mathbb{R}^{\ell\times m}\mid\|W\|_{F}\le a\}
        .
    \]
    Suppose $\max_{i\in[n]}\|X_{i}\|_{2\to\infty}\le B_{n,(2\to\infty)}$ holds. Then, for any $\epsilon>0$ and $\delta\in(0,1)$, we have
    \begin{align*}
        &\log\mathcal{N}_{\infty}(\mathcal{F},\|\cdot\|_{2\to\infty},\epsilon;\{X_{i}\}_{i\in[n]})\\
        &\le\frac{36a^{2}B_{n,(2\to\infty)}^{2}}{(1-\delta)^{2}\epsilon^{2}}\log\left(2\left\lceil\frac{4aB_{n,(2\to\infty)}}{(1-\delta)\epsilon}+2\right\rceil\left(1+\frac{2}{\delta}\right)^{m}nd+1\right)
        .
    \end{align*}
    In particular, setting $\delta=1/2$ gives
    \begin{align*}
        \log\mathcal{N}_{\infty}(\mathcal{F},\|\cdot\|_{2\to\infty},\epsilon;\{X_{i}\}_{i\in[n]})
        &\le\frac{144a^{2}B_{n,(2\to\infty)}^{2}}{\epsilon^{2}}\log\left(4\left\lceil\frac{8aB_{n,(2\to\infty)}}{\epsilon}+2\right\rceil5^{m}nd\right)\\
        &\le\frac{144a^{2}B_{n,(2\to\infty)}^{2}m}{\epsilon^{2}}\log\left(20\left\lceil\frac{8aB_{n,(2\to\infty)}}{\epsilon}+2\right\rceil nd\right)
        .
    \end{align*}
\end{corollary}
\begin{proof}
    Fix $\delta\in(0,1)$. Let $X_{i,j\cdot}$ denote the $j$-th row of $X_{i}$. By Fact~\ref{fact:EntropyEuclideanBall}, one may choose a $\delta$-net $\mathcal{N}_{m}\subset\mathcal{S}^{m-1}$ of the Euclidean sphere satisfying $|\mathcal{N}_{m}|\le(1+2/\delta)^{m}$. For such $\mathcal{N}_{m}$, we can apply Lemma~\ref{lem:OperatorNormNet1} to obtain
    \begin{align*}
        &\|X_{i}(W-W')\|_{2\to\infty}
        =\max_{j\in[d]}\|X_{i,j\cdot}(W-W')\|_{2}\\
        &\le\frac{1}{1-\delta}\max_{j\in[d]}\max_{u\in\mathcal{N}_{m}}|X_{i,j\cdot}(W-W')u|
        =\frac{1}{1-\delta}\max_{j\in[d]}\max_{u\in\mathcal{N}_{m}}\left|\langle W-W',(X_{i,j\cdot})^{\top}u^{\top}\rangle_{F}\right|\\
        &=\frac{1}{1-\delta}\max_{j\in[d]}\max_{u\in\mathcal{N}_{m}}\left|\mathrm{vec}(W-W')^{\top}\mathrm{vec}((X_{i,j\cdot})^{\top}u^{\top})\right|
    \end{align*}
    for any $W,W'\in\mathcal{W}$ and $i\in[n]$. Define the finite set of vectors
    \[
        \mathcal{Y}
        =\{\mathrm{vec}((X_{i,j\cdot})^{\top}u^{\top})\mid i\in[n],j\in[d],u\in\mathcal{N}_{m}\}
        .
    \]
    Then, any vector in $\mathcal{Y}$ is bounded in norm as
    \[
        \|\mathrm{vec}((X_{i,j\cdot})^{\top}u^{\top})\|_{2}
        =\|(X_{i,j\cdot})^{\top}u^{\top}\|_{F}
        =\|X_{i,j\cdot}\|_{2}\|u\|_{2}
        =\|X_{i,j\cdot}\|_{2}
        \le\|X_{i}\|_{2\to\infty}
        \le B_{n,(2\to\infty)}
        .
    \]
    Thus, by applying Proposition~\ref{prop:Zhang2002Thm4} with $p=q=2$ and the linear function class
    \[
        \tilde{\mathcal{F}}
        =\{f:\mathbb{R}^{\ell m}\to\mathbb{R}\mid f(y)=w^{\top}y,\ w\in\mathbb{R}^{\ell m},\ \|w\|_{2}\le a\}
        ,
    \]
    we obtain
    \begin{align*}
        &\log\mathcal{N}_{\infty}(\mathcal{F},\|\cdot\|_{2\to\infty},\epsilon;\{X_{i}\}_{i\in[n]})\\
        &\le\log\mathcal{N}_{\infty}(\tilde{\mathcal{F}},|\cdot|,(1-\delta)\epsilon;\mathcal{Y})\\
        &\le36\frac{a^{2}B_{n,(2\to\infty)}^{2}}{(1-\delta)^{2}\epsilon^{2}}\log\left(2\left\lceil\frac{4aB_{n,(2\to\infty)}}{(1-\delta)\epsilon}+2\right\rceil nd|\mathcal{N}_{m}|+1\right)\\
        &\le36\frac{a^{2}B_{n,(2\to\infty)}^{2}}{(1-\delta)^{2}\epsilon^{2}}\log\left(2\left\lceil\frac{4aB_{n,(2\to\infty)}}{(1-\delta)\epsilon}+2\right\rceil nd\left(1+\frac{2}{\delta}\right)^{m}+1\right)
        .
    \end{align*}
    Furthermore, setting $\delta=1/2$ gives
    \begin{align*}
        &\log\mathcal{N}_{\infty}(\mathcal{F},\|\cdot\|_{2\to\infty},\epsilon;\{X_{i}\}_{i\in[n]})\\
        &\le\frac{144a^{2}B_{n,(2\to\infty)}^{2}}{\epsilon^{2}}\log\left(2\left\lceil\frac{8aB_{n,(2\to\infty)}}{\epsilon}+2\right\rceil
        5^{m}nd+1\right)\\
        &\le\frac{144a^{2}B_{n,(2\to\infty)}^{2}}{\epsilon^{2}}\log\left(4\left\lceil\frac{8aB_{n,(2\to\infty)}}{\epsilon}+2\right\rceil5^{m}nd\right)\\
        &=\frac{144a^{2}B_{n,(2\to\infty)}^{2}}{\epsilon^{2}}\left(m\log5+\log\left(4\left\lceil\frac{8aB_{n,(2\to\infty)}}{\epsilon}+2\right\rceil nd\right)\right)\\
        &\le\frac{144a^{2}B_{n,(2\to\infty)}^{2}m}{\epsilon^{2}}\log\left(20\left\lceil\frac{8aB_{n,(2\to\infty)}}{\epsilon}+2\right\rceil nd\right)
        .
    \end{align*}
\end{proof}

\subsection{Parametric interpolation}

We now prove a matrix-valued version of the parametric interpolation argument of \citet{ledent2025generalization}. The main additional point is that the output metric is $\|\cdot\|_{2\to\infty}$ evaluated on the sample $\{X_{i}\}_{i\in[n]}$; hence, the Frobenius-controlled tail must be converted into uniform rowwise output bounds.

\begin{theorem}\label{thm:ParametricInterpolation2inf}
    Fix an arbitrary $p\in[0,2]$. Consider a class of matrix-valued functions
    \[
        \mathcal{F}
        =\{f:\mathbb{R}^{d\times\ell}\to\mathbb{R}^{d\times m}\mid f(X)=XW,\ W\in\mathcal{W}\}
        ,
    \]
    where the parameter set is defined by
    \[
        \mathcal{W}
        =\{W\in\mathbb{R}^{\ell\times m}\mid\|W\|_{\mathrm{s},p}^{p}\le C_{\mathrm{s}},\ \|W\|_{2}\le C_{2}\}
        .
    \]
    Suppose $\max_{i\in[n]}\|X_{i}\|_{2\to\infty}\le B_{n,(2\to\infty)}$ holds. Then, for any $\epsilon>0$, we have
    \[
        \log\mathcal{N}_{\infty}(\mathcal{F},\|\cdot\|_{2\to\infty},\epsilon;\{X_{i}\}_{i\in[n]})
        \lesssim\left(\frac{[C_{\mathrm{s}}(\ell+m)]^{2}(B_{n,(2\to\infty)}^{2}\min\{\ell,m\}m)^{p}}{\epsilon^{2p}}\right)^{\frac{1}{p+2}}\log C_{2\to\infty}
        ,
    \]
    where $C_{2\to\infty}$ is given by
    \begin{align*}
        C_{2\to\infty}
        &=\left(32C_{\mathrm{s}}^{\frac{1}{p+2}}C_{2}\left(\frac{B_{n,(2\to\infty)}}{\epsilon}\right)^{\frac{2p+2}{p+2}}\left(\frac{\min\{\ell,m\}m}{\ell+m}\right)^{\frac{p}{2(p+2)}}+1\right)\\
        &\quad\times\left[\left(640\left(\frac{C_{\mathrm{s}}B_{n,(2\to\infty)}^{p}\min\{\ell,m\}^{p/2}(\ell+m)}{m\epsilon^{p}}\right)^{\frac{1}{p+2}}+60\right)nd\right]
        .
    \end{align*}
\end{theorem}
\begin{proof}
    Fix $p\in[0,2]$ and use the convention $\|W\|_{\mathrm{s},0}^{0}=\operatorname{rank}(W)$ when $p=0$. For any $W\in\mathcal{W}$, consider its singular value decomposition
    \[
        W
        =\sum_{i=1}^{\min\{\ell,m\}}\sigma_{i}u_{i}v_{i}^{\top}
        ,
        \qquad
        \sigma_{1}\ge\sigma_{2}\ge\cdots\ge0
        .
    \]
    Fix $\tau>0$ and define $r=\max\{i\mid\sigma_{i}>\tau\}$, with the convention $r=0$ if this set is empty. Consider the decomposition of $W$ as $W=W_{1}+W_{2}$, where
    \[
        W_{1}=\sum_{i=1}^{r}\sigma_{i}u_{i}v_{i}^{\top},
        \qquad
        W_{2}=W-W_{1}
        .
    \]
    It holds that $r\tau^{p}\le\sum_{i=1}^{r}\sigma_{i}^{p}\le\sum_{i=1}^{\min\{\ell,m\}}\sigma_{i}^{p}=\|W\|_{\mathrm{s},p}^{p}\le C_{\mathrm{s}}$, which implies
    \[
        \operatorname{rank}(W_{1})\le r\le\frac{C_{\mathrm{s}}}{\tau^{p}}
        .
    \]
    Therefore, we have $W_{1}\in\mathcal{W}_{1}=\{W\in\mathcal{W}\mid\operatorname{rank}(W)\le C_{\mathrm{s}}/\tau^{p}\}$.\footnote{Note that this argument holds for both $p>0$ and $p=0$.} Furthermore, for any $W,W'\in\mathcal{W}_{1}$, we have
    \[
        \max_{i\in[n]}\|X_{i}(W-W')\|_{2\to\infty}
        \le\max_{i\in[n]}\|X_{i}\|_{2\to\infty}\|W-W'\|_{2}
        \le B_{n,(2\to\infty)}\|W-W'\|_{2}
        .
    \]
    Define the function class $\mathcal{F}_{1}=\{f:\mathbb{R}^{d\times\ell}\to\mathbb{R}^{d\times m}\mid f(X)=XW,\ W\in\mathcal{W}_{1}\}$. By applying Proposition~\ref{prop:Ledent2025PropF1} and Remark~\ref{rem:NormInequalityCovering}, there exists a $\epsilon/4$-cover $\mathcal{C}_{\mathcal{F}_{1}}$ of $\mathcal{F}_{1}$ on $\{X_{i}\}_{i\in[n]}$ that satisfies
    \[
        \log|\mathcal{C}_{\mathcal{F}_{1}}|
        \le\log\mathcal{N}\left(\mathcal{W}_{1},\|\cdot\|_{2},\frac{\epsilon}{4B_{n,(2\to\infty)}}\right)
        \le(\ell+m)\frac{C_{\mathrm{s}}}{\tau^{p}}\log\left(\frac{32C_{\mathrm{s}}^{\frac{1}{2}}C_{2}B_{n,(2\to\infty)}}{\tau^{\frac{p}{2}}\epsilon}+1\right)
        .
    \]
    On the other hand, since
    \[
        \|W_{2}\|_{F}^{2}
        =\sum_{i=r+1}^{\min\{\ell,m\}}\sigma_{i}^{2}
        \le\sum_{i=r+1}^{\min\{\ell,m\}}\tau^{2}
        \le\tau^{2}\min\{\ell,m\}
    \]
    holds, we have $W_{2}\in\{W\in\mathcal{W}\mid\|W\|_{F}\le\tau\min\{\ell,m\}^{1/2}\}$. By Corollary~\ref{cor:Zhang2002Thm4TwogInf}, there exists a $\epsilon/4$-cover $\mathcal{C}_{\mathcal{F}_{2}}$ of $\mathcal{F}_{2}$ on $\{X_{i}\}_{i\in[n]}$ that satisfies
    \[
        \log|\mathcal{C}_{\mathcal{F}_{2}}|
        \lesssim\frac{\tau^{2}\min\{\ell,m\}B_{n,(2\to\infty)}^{2}m}{\epsilon^{2}}\log\left[\left(\frac{640\tau\min\{\ell,m\}^{1/2}B_{n,(2\to\infty)}}{\epsilon}+60\right)nd\right]
        .
    \]
    Set $\bar{\mathcal{C}}_{\mathcal{F}}=\{f_{1}+f_{2}\mid f_{1}\in\mathcal{C}_{\mathcal{F}_{1}},f_{2}\in\mathcal{C}_{\mathcal{F}_{2}}\}$. Then, for any $f=f_{1}+f_{2}\in\mathcal{F}$ with $f_{1}\in\mathcal{F}_{1}$ and $f_{2}\in\mathcal{F}_{2}$, there exist $\tilde{f}_{1}\in\mathcal{C}_{\mathcal{F}_{1}}$ and $\tilde{f}_{2}\in\mathcal{C}_{\mathcal{F}_{2}}$ that satisfy
    \[
        \max_{i\in[n]}\|f_{1}(X_{i})-\tilde{f}_{1}(X_{i})\|_{2\to\infty}
        \le\frac{\epsilon}{4}
        ,\qquad
        \max_{i\in[n]}\|f_{2}(X_{i})-\tilde{f}_{2}(X_{i})\|_{2\to\infty}
        \le\frac{\epsilon}{4}
        .
    \]
    Hence, it holds for $\tilde{f}=\tilde{f}_{1}+\tilde{f}_{2}$ that
    \begin{align*}
        &\max_{i\in[n]}\|f(X_{i})-\tilde{f}(X_{i})\|_{2\to\infty}
        \le\max_{i\in[n]}\|f_{1}(X_{i})-\tilde{f}_{1}(X_{i})\|_{2\to\infty}+\max_{i\in[n]}\|f_{2}(X_{i})-\tilde{f}_{2}(X_{i})\|_{2\to\infty}\\
        &\le\frac{\epsilon}{2}
        ,
    \end{align*}
    which implies that $\bar{\mathcal{C}}_{\mathcal{F}}$ is a (possibly improper) $\epsilon/2$-cover of $\mathcal{F}$ on $\{X_{i}\}_{i\in[n]}$. Finally, by Lemma~\ref{lem:ProperCover}, we have
    \begin{align*}
        &\log\mathcal{N}_{\infty}(\mathcal{F},\|\cdot\|_{2\to\infty},\epsilon;\{X_{i}\}_{i\in[n]})\\
        &\le\log|\bar{\mathcal{C}}_{\mathcal{F}}|
        \le\log|\bar{\mathcal{C}}_{\mathcal{F}_{1}}|+\log|\bar{\mathcal{C}}_{\mathcal{F}_{2}}|\\
        &\lesssim(\ell+m)\frac{C_{\mathrm{s}}}{\tau^{p}}\log\left(\frac{32C_{\mathrm{s}}^{\frac{1}{2}}C_{2}B_{n,(2\to\infty)}}{\tau^{\frac{p}{2}}\epsilon}+1\right)\\
        &\quad+\frac{\tau^{2}\min\{\ell,m\}B_{n,(2\to\infty)}^{2}m}{\epsilon^{2}}\log\left[\left(\frac{640\tau\min\{\ell,m\}^{1/2}B_{n,(2\to\infty)}}{\epsilon}+60\right)nd\right]
        .
    \end{align*}
    Choosing $\tau=\left(\frac{C_{\mathrm{s}}(\ell+m)\epsilon^{2}}{B_{n,(2\to\infty)}^{2}\min\{\ell,m\}m}\right)^{\frac{1}{p+2}}$, we have
    \[
        \log\mathcal{N}_{\infty}(\mathcal{F},\|\cdot\|_{2\to\infty},\epsilon;\{X_{i}\}_{i\in[n]})
        \lesssim\left(\frac{[C_{\mathrm{s}}(\ell+m)]^{2}(B_{n,(2\to\infty)}^{2}\min\{\ell,m\}m)^{p}}{\epsilon^{2p}}\right)^{\frac{1}{p+2}}\log C_{2\to\infty}
        .
    \]
\end{proof}

\subsection{Entropy bounds for the class of composite functions}

We finally record a simple composition rule for covering numbers. This lemma allows us to construct a cover of a composed function class by combining a cover of the inner class with sample-dependent covers of the outer class, while keeping track of the Lipschitz propagation of approximation errors.

\begin{lemma}\label{lem:CoveringCompositeFunctions}
    Let $(S_{1},d_{1})$ and $(S_{2},d_{2})$ be metric spaces. Let $\mathcal{F}$ be a class of functions from $S$ to $(S_{1},d_{1})$, and let $\mathcal{G}$ be a class of functions from $(S_{1},d_{1})$ to $(S_{2},d_{2})$. Suppose that there exists a constant $L_{\mathcal{G}}>0$ such that, for every $g\in\mathcal{G}$ and every $y,y'\in\{f(x_{i})\mid f\in\mathcal{F},i\in[n]\}$,
    \[
        d_{2}(g(y),g(y'))
        \le L_{\mathcal{G}}d_{1}(y,y')
    \]
    holds. Suppose $\mathcal{C}_{\mathcal{F}}$ is an $\epsilon_{\mathcal{F}}$-cover of $\mathcal{F}$ on $\{x_{i}\}_{i\in[n]}$. For each $\tilde{f}\in\mathcal{C}_{\mathcal{F}}$, suppose $\mathcal{C}_{\mathcal{G}}(\tilde{f})$ is an $\epsilon_{\mathcal{G}}$-cover of $\mathcal{G}$ on $\{\tilde{f}(x_{i})\}_{i\in[n]}$. Then, $\mathcal{C}=\{\tilde{g}\circ\tilde{f}\mid\tilde{f}\in\mathcal{C}_{\mathcal{F}},\tilde{g}\in\mathcal{C}_{\mathcal{G}}(\tilde{f})\}$ is an $(L_{\mathcal{G}}\epsilon_{\mathcal{F}}+\epsilon_{\mathcal{G}})$-cover of $\mathcal{G}\circ\mathcal{F}$.
\end{lemma}
\begin{proof}
    For any $f\in\mathcal{F}$ and $g\in\mathcal{G}$, there exist $\tilde{f}\in\mathcal{C}_{\mathcal{F}}$ and $\tilde{g}\in\mathcal{C}_{\mathcal{G}}(\tilde{f})$ that satisfy
    \[
        d_{1}(f(x_{i}),\tilde{f}(x_{i}))
        \le\epsilon_{\mathcal{F}}
        ,\qquad
        d_{2}(g(\tilde{f}(x_{i})),\tilde{g}(\tilde{f}(x_{i})))
        \le\epsilon_{\mathcal{G}}
    \]
    for every $i\in[n]$. Therefore, by the triangle inequality and the Lipschitz continuity of $g$, we have
    \begin{align*}
        d_{2}(g\circ f(x_{i}),\tilde{g}\circ\tilde{f}(x_{i}))
        &\le d_{2}(g\circ f(x_{i}),g\circ\tilde{f}(x_{i}))+d_{2}(g\circ\tilde{f}(x_{i}),\tilde{g}\circ\tilde{f}(x_{i}))\\
        &\le L_{\mathcal{G}}d_{1}(f(x_{i}),\tilde{f}(x_{i}))+d_{2}(g(\tilde{f}(x_{i})),\tilde{g}(\tilde{f}(x_{i})))\\
        &\le L_{\mathcal{G}}\epsilon_{\mathcal{F}}+\epsilon_{\mathcal{G}}
        .
    \end{align*}
    Thus, $\mathcal{C}$ is an $(L_{\mathcal{G}}\epsilon_{\mathcal{F}}+\epsilon_{\mathcal{G}})$-cover of $\mathcal{G}\circ\mathcal{F}$.
\end{proof}

\section{Proofs of the main results}\label{sec:ProofsMainResults}

Throughout this section, the $O(\cdot)$ notation suppresses all logarithmic factors except those involving $n$ and $T$.

\subsection{Covering number bounds for Transformer heads}

Consider a Transformer head $f_{\mathrm{head}}(\ \cdot\ ;W^{QK},W^{V}):\mathbb{R}^{T\times N}\to\mathbb{R}^{T\times N}$ defined in Eq.~\eqref{eq:DefHead}. We define the function class of the Transformer head as
\[
    \mathcal{F}_{\mathrm{head}}
    =\{f_{\mathrm{head}}(\ \cdot\ ;W^{QK},W^{V})\mid W^{QK}\in\mathcal{W}^{QK}(p^{QK},C_{\mathrm{s}}^{QK}),\ W^{V}\in\mathcal{W}^{V}(p^{V},C_{\mathrm{s}}^{V})\}
    ,
\]
where, for fixed $p^{QK},p^{V}\in[0,2]$, the parameter sets $\mathcal{W}^{QK},\mathcal{W}^{V}$ are defined by
\begin{equation}\label{eq:HeadParameterClass}
    \begin{aligned}
        &\mathcal{W}^{QK}(p^{QK},C_{\mathrm{s}}^{QK})
        =\{W\in\mathbb{R}^{N\times N}\mid\|W\|_{\mathrm{s},p^{QK}}^{p^{QK}}\le C_{\mathrm{s}}^{QK},\ \|W\|_{2}\le C_{2}^{QK}\}
        ,\\
        &\mathcal{W}^{V}(p^{V},C_{\mathrm{s}}^{V})
        =\{W\in\mathbb{R}^{N\times N}\mid \|W\|_{\mathrm{s},p^{V}}^{p^{V}}\le C_{\mathrm{s}}^{V},\ \|W\|_{2}\le C_{2}^{V}\}
        .
    \end{aligned}
\end{equation}

\begin{proposition}\label{prop:EntropyTFHead}
    Suppose $\max_{i\in[n]}\|X_{i}\|_{2\to\infty}\le B_{n,(2\to\infty)}$ holds. Then, for any $\epsilon^{QK},\epsilon^{V}>0$, we have
    \begin{align*}
        &\log\mathcal{N}_{\infty}(\mathcal{F}_{\mathrm{head}},\|\cdot\|_{2\to\infty},2C_{2}^{V}B_{n,(2\to\infty)}^{2}\epsilon^{QK}+\epsilon^{V};\{X_{i}\}_{i\in[n]})\\
        &\lesssim\left[\left(\frac{(C_{\mathrm{s}}^{QK})^{2}B_{n,(2\to\infty)}^{2p^{QK}}N^{p^{QK}}}{(\epsilon^{QK})^{2p^{QK}}}\right)^{\frac{1}{p^{QK}+2}}+\left(\frac{(C_{\mathrm{s}}^{V})^{2}B_{n,(2\to\infty)}^{2p^{V}}N^{p^{V}}}{(\epsilon^{V})^{2p^{V}}}\right)^{\frac{1}{p^{V}+2}}\right]N\log(nT)
        ,
    \end{align*}
    where we omit the logarithmic factors except $\log(nT)$. Therefore, for any $\epsilon>0$, setting $\epsilon^{QK}=\epsilon/(4C_{2}^{V}B_{n,(2\to\infty)}^{2})$ and $\epsilon^{V}=\epsilon/2$ yields
    \begin{align*}
        &\log\mathcal{N}_{\infty}(\mathcal{F}_{\mathrm{head}},\|\cdot\|_{2\to\infty},\epsilon;\{X_{i}\}_{i\in[n]})\\
        &\lesssim\left[\left(\frac{(C_{\mathrm{s}}^{QK})^{2}(C_{2}^{V})^{2p^{QK}}B_{n,(2\to\infty)}^{6p^{QK}}N^{p^{QK}}}{\epsilon^{2p^{QK}}}\right)^{\frac{1}{p^{QK}+2}}+\left(\frac{(C_{\mathrm{s}}^{V})^{2}B_{n,(2\to\infty)}^{2p^{V}}N^{p^{V}}}{\epsilon^{2p^{V}}}\right)^{\frac{1}{p^{V}+2}}\right]\\
        &\quad\times N\log(nT)
        .
    \end{align*}
\end{proposition}
\begin{proof}
    Define linear function classes $\mathcal{F}_{QK}$ and $\mathcal{F}_{V}$ by
    \begin{align*}
        &\mathcal{F}_{QK}
        =\{f:\mathbb{R}^{T\times N}\to\mathbb{R}^{T\times N}\mid f(X)=XW,\ W\in\mathcal{W}^{QK}\}
        ,\\
        &\mathcal{F}_{V}
        =\{f:\mathbb{R}^{T\times N}\to\mathbb{R}^{T\times N}\mid f(X)=XW,\ W\in\mathcal{W}^{V}\}
        .
    \end{align*}
    Suppose $\mathcal{C}_{\mathcal{F}_{QK}}$ is an $\epsilon^{QK}$-cover of $\mathcal{F}_{QK}$, and $\mathcal{C}_{\mathcal{F}_{V}}$ is an $\epsilon^{V}$-cover of $\mathcal{F}_{V}$. By Theorem~\ref{thm:ParametricInterpolation2inf}, we can take $\mathcal{C}_{\mathcal{F}_{QK}}$ and $\mathcal{C}_{\mathcal{F}_{V}}$ with cardinalities satisfying
    \begin{align*}
        &\log|\mathcal{C}_{\mathcal{F}_{QK}}|
        \lesssim\left(\frac{(C_{\mathrm{s}}^{QK})^{2}B_{n,(2\to\infty)}^{2p^{QK}}N^{p^{QK}}}{(\epsilon^{QK})^{2p^{QK}}}\right)^{\frac{1}{p^{QK}+2}}N\log(nT)
        ,\\
        &\log|\mathcal{C}_{\mathcal{F}_{V}}|
        \lesssim\left(\frac{(C_{\mathrm{s}}^{V})^{2}B_{n,(2\to\infty)}^{2p^{V}}N^{p^{V}}}{(\epsilon^{V})^{2p^{V}}}\right)^{\frac{1}{p^{V}+2}}N\log(nT)
        .
    \end{align*}
    We denote the corresponding parameter sets of $\mathcal{C}_{\mathcal{F}_{QK}}$ and $\mathcal{C}_{\mathcal{F}_{V}}$ by $\mathcal{C}_{\mathcal{W}^{QK}}$ and $\mathcal{C}_{\mathcal{W}^{V}}$, respectively. We show that
    \[
        \mathcal{C}_{\mathcal{F}_{\mathrm{head}}}
        =\{f_{\mathrm{head}}(\ \cdot\ ;\tilde{W}^{QK},\tilde{W}^{V})\mid\tilde{W}^{QK}\in\mathcal{C}_{\mathcal{W}^{QK}},\tilde{W}^{V}\in\mathcal{C}_{\mathcal{W}^{V}}\}
    \]
    is a $(2C_{2}^{V}B_{n,(2\to\infty)}^{2}\epsilon^{QK}+\epsilon^{V})$-cover of $\mathcal{F}_{\mathrm{head}}$ on $\{X_{i}\}_{i\in[n]}$.
    
    By the definition of $\mathcal{C}_{\mathcal{F}_{QK}}$ and $\mathcal{C}_{\mathcal{F}_{V}}$, for any $W^{QK}\in\mathcal{W}^{QK}$ and $W^{V}\in\mathcal{W}^{V}$, there exist $\tilde{W}^{QK}\in\mathcal{C}_{\mathcal{W}^{QK}}$ and $\tilde{W}^{V}\in\mathcal{C}_{\mathcal{W}^{V}}$ that satisfy
    \[
        \max_{i\in[n]}\|X_{i}(W^{QK}-\tilde{W}^{QK})\|_{2\to\infty}
        \le\epsilon^{QK}
        ,\qquad
        \max_{i\in[n]}\|X_{i}(W^{V}-\tilde{W}^{V})\|_{2\to\infty}
        \le\epsilon^{V}
        .
    \]
    For any $X\in\mathbb{R}^{T\times N}$, we first bound the distance between $f_{\mathrm{head}}(X;W^{QK},W^{V})$ and $f_{\mathrm{head}}(X;\tilde{W}^{QK},\tilde{W}^{V})$ as
    \begin{align*}
        &\|f_{\mathrm{head}}(X;W^{QK},W^{V})-f_{\mathrm{head}}(X;\tilde{W}^{QK},\tilde{W}^{V})\|_{2\to\infty}\\
        &=\|\mathrm{SoftMax}(XW^{QK}X^{\top})XW^{V}-\mathrm{SoftMax}(X\tilde{W}^{QK}X^{\top})X\tilde{W}^{V}\|_{2\to\infty}\\
        &\le\|\mathrm{SoftMax}(XW^{QK}X^{\top})X(W^{V}-\tilde{W}^{V})\|_{2\to\infty}\\
        &\quad+\|(\mathrm{SoftMax}(XW^{QK}X^{\top})-\mathrm{SoftMax}(X\tilde{W}^{QK}X^{\top}))X\tilde{W}^{V}\|_{2\to\infty}
        .
    \end{align*}
    The first term on the right-hand side can be bounded by
    \begin{align*}
        &\|\mathrm{SoftMax}(XW^{QK}X^{\top})X(W^{V}-\tilde{W}^{V})\|_{2\to\infty}\\
        &\le\|\mathrm{SoftMax}(XW^{QK}X^{\top})\|_{\infty}\|X(W^{V}-\tilde{W}^{V})\|_{2\to\infty}
        \le\|X(W^{V}-\tilde{W}^{V})\|_{2\to\infty}
        ,
    \end{align*}
    where the last inequality follows from Lemma~\ref{lem:SoftmaxMatrixInfinity}. For the second term, it holds that
    \begin{align*}
        &\|(\mathrm{SoftMax}(XW^{QK}X^{\top})-\mathrm{SoftMax}(X\tilde{W}^{QK}X^{\top}))X\tilde{W}^{V}\|_{2\to\infty}\\
        &=\max_{t\in[T]}\|(\mathrm{SoftMax}_{t\cdot}(XW^{QK}X^{\top})-\mathrm{SoftMax}_{t\cdot}(X\tilde{W}^{QK}X^{\top}))X\tilde{W}^{V}\|_{2}\\
        &\le\max_{t\in[T]}\|(\mathrm{SoftMax}_{t\cdot}(XW^{QK}X^{\top})-\mathrm{SoftMax}_{t\cdot}(X\tilde{W}^{QK}X^{\top}))^{\top}\|_{1}\|X\tilde{W}^{V}\|_{2\to\infty}\\
        &\le2\max_{t\in[T]}\|((XW^{QK}X^{\top})_{t\cdot}-(X\tilde{W}^{QK}X^{\top})_{t\cdot})^{\top}\|_{\infty}\|X\|_{2\to\infty}\|\tilde{W}^{V}\|_{2}
        ,
    \end{align*}
    where the first inequality follows from Lemma~\ref{lem:VecMatrixProduct} and the second inequality follows from Lemmas~\ref{lem:SoftmaxLipschitzVector1}~and~\ref{lem:MatrixProduct2Inf2}. Using Lemma~\ref{lem:MatrixProduct2Inf2} once again, we further bound as
    \begin{align*}
        &\max_{t\in[T]}\|((XW^{QK}X^{\top})_{t\cdot}-(X\tilde{W}^{QK}X^{\top})_{t\cdot})^{\top}\|_{\infty}
        =\max_{t\in[T]}\max_{s\in[T]}|X_{t\cdot}(W^{QK}-\tilde{W}^{QK})(X_{s\cdot})^{\top}|\\
        &=\max_{s\in[T]}\|X(W^{QK}-\tilde{W}^{QK})(X_{s\cdot})^{\top}\|_{2\to\infty}\\
        &\le\max_{s\in[T]}\|X(W^{QK}-\tilde{W}^{QK})\|_{2\to\infty}\|(X_{s\cdot})^{\top}\|_{2}
        =\|X(W^{QK}-\tilde{W}^{QK})\|_{2\to\infty}\|X\|_{2\to\infty}
        ,
    \end{align*}
    where the last equality uses $\|v\|_{2}=\|v^{\top}\|_{2}$ for any vector $v$. Combining the results, we have
    \begin{align*}
        &\max_{i\in[n]}\|f_{\mathrm{head}}(X_{i};W^{QK},W^{V})-f_{\mathrm{head}}(X_{i};\tilde{W}^{QK},\tilde{W}^{V})\|_{2\to\infty}\\
        &\le\max_{i\in[n]}\|X_{i}(W^{V}-\tilde{W}^{V})\|_{2\to\infty}+2\max_{i\in[n]}\|X_{i}(W^{QK}-\tilde{W}^{QK})\|_{2\to\infty}\|X_{i}\|_{2\to\infty}^{2}\|\tilde{W}^{V}\|_{2}\\
        &\le\epsilon^{V}+2\epsilon^{QK}B_{n,(2\to\infty)}^{2}C_{2}^{V}
        ,
    \end{align*}
    which implies that $\mathcal{C}_{\mathcal{F}_{\mathrm{head}}}$ is a $(2C_{2}^{V}B_{n,(2\to\infty)}^{2}\epsilon^{QK}+\epsilon^{V})$-cover of $\mathcal{F}_{\mathrm{head}}$ on $\{X_{i}\}_{i\in[n]}$. Thus, we obtain the entropy bound
    \begin{align*}
        &\log\mathcal{N}_{\infty}(\mathcal{F}_{\mathrm{head}},\|\cdot\|_{2\to\infty},2C_{2}^{V}B_{n,(2\to\infty)}^{2}\epsilon^{QK}+\epsilon^{V};\{X_{i}\}_{i\in[n]})\\
        &\le\log|\mathcal{C}_{\mathcal{F}_{\mathrm{head}}}|\\
        &\le\log|\mathcal{C}_{\mathcal{F}_{QK}}|+\log|\mathcal{C}_{\mathcal{F}_{V}}|\\
        &\lesssim\left[\left(\frac{(C_{\mathrm{s}}^{QK})^{2}B_{n,(2\to\infty)}^{2p^{QK}}N^{p^{QK}}}{(\epsilon^{QK})^{2p^{QK}}}\right)^{\frac{1}{p^{QK}+2}}+\left(\frac{(C_{\mathrm{s}}^{V})^{2}B_{n,(2\to\infty)}^{2p^{V}}N^{p^{V}}}{(\epsilon^{V})^{2p^{V}}}\right)^{\frac{1}{p^{V}+2}}\right]N\log(nT)
        .
    \end{align*}
\end{proof}

\subsection{Covering number bounds for Transformer blocks}

We define the Transformer block $f_{\mathrm{block}}(\ \cdot\ ;W^{QK},W^{V},W^{M}):\mathbb{R}^{T\times N}\to\mathbb{R}^{T\times N}$ by Eq.~\eqref{eq:DefBlock}. Define $\mathcal{W}^{QK}(p^{QK},C_{\mathrm{s}}^{QK})$ and $\mathcal{W}^{V}(p^{V},C_{\mathrm{s}}^{V})$ by Eq.~\eqref{eq:HeadParameterClass}, and define
\[
    \mathcal{W}^{M}(p^{M},C_{\mathrm{s}}^{M})
    =\{W\in\mathbb{R}^{N\times N}\mid \|W\|_{\mathrm{s},p^{M}}^{p^{M}}\le C_{\mathrm{s}}^{M},\ \|W\|_{2}\le C_{2}^{M}\}
    .
\]
We introduce the class of Transformer blocks as
\[
    \mathcal{F}_{\mathrm{block}}
    =\{f_{\mathrm{block}}(\ \cdot\ ;W^{QK},W^{V},W^{M})\mid W^{\star}\in\mathcal{W}^{\star}(p^{\star},C_{\mathrm{s}}^{\star}),\ \star\in\{QK,V,M\}\}
    .
\]

We first note that by the Lipschitz property of the activation function $\phi:\mathbb{R}^{N}\to\mathbb{R}^{N}$ in Assumption~\ref{asm:Activation}, the same Lipschitz property holds for the rowwise extension of $\phi$ with respect to the $\|\cdot\|_{2\to\infty}$ norm. Indeed, for all $Z,Z'\in\mathbb{R}^{T\times N}$, it holds that
\begin{equation}\label{eq:ActivationMatrixLipschitz}
    \|\phi(Z)-\phi(Z')\|_{2\to\infty}
    =\max_{t\in[T]}\|\phi(Z_{t\cdot})-\phi(Z'_{t\cdot})\|_{2}
    \le L_{\phi}\max_{t\in[T]}\|Z_{t\cdot}-Z'_{t\cdot}\|_{2}
    =L_{\phi}\|Z-Z'\|_{2\to\infty}
    .
\end{equation}

\begin{proposition}\label{prop:EntropyTfBlock}
    Suppose $\max_{i\in[n]}\|X_{i}\|_{2\to\infty}\le B_{n,(2\to\infty)}$ holds. Then, for any $\epsilon^{QK},\epsilon^{V},\epsilon^{M}>0$, we have
    \begin{align*}
        &\log\mathcal{N}_{\infty}(\mathcal{F}_{\mathrm{block}},\|\cdot\|_{2\to\infty},2L_{\phi}C_{2}^{V}C_{2}^{M}B_{n,(2\to\infty)}^{2}\epsilon^{QK}+L_{\phi}C_{2}^{M}\epsilon^{V}+\epsilon^{M};\{X_{i}\}_{i\in[n]})\\
        &\lesssim(\Upsilon^{QK}(\epsilon^{QK})+\Upsilon^{V}(\epsilon^{V})+\Upsilon^{M}(\epsilon^{M}))N\log(nT)
        ,
    \end{align*}
    where
    \[
        \Upsilon^{\star}(\epsilon^{\star})
        =\left(\frac{(C_{\mathrm{s}}^{\star})^{2}B_{n,(2\to\infty)}^{2p^{\star}}N^{p^{\star}}}{(\epsilon^{\star})^{2p^{\star}}}\right)^{\frac{1}{p^{\star}+2}}
        ,\qquad
        \Upsilon^{M}(\epsilon^{M})
        =\left(\frac{(C_{\mathrm{s}}^{M})^{2}L_{\phi}^{2p^{M}}N^{p^{M}}}{(\epsilon^{M})^{2p^{M}}}\right)^{\frac{1}{p^{M}+2}}
    \]
    with $\star\in\{QK,V\}$.
\end{proposition}
\begin{proof}
    Define $\epsilon_{\mathrm{head}}=2C_{2}^{V}B_{n,(2\to\infty)}^{2}\epsilon^{QK}+\epsilon^{V}$, and suppose $\mathcal{C}_{\mathcal{F}_{\mathrm{head}}}\subset\mathcal{F}_{\mathrm{head}}$ is an $\epsilon_{\mathrm{head}}$-cover of $\mathcal{F}_{\mathrm{head}}$ on $\{X_{i}\}_{i\in[n]}$. Note that by Proposition~\ref{prop:EntropyTFHead}, we can choose $\mathcal{C}_{\mathcal{F}_{\mathrm{head}}}$ that satisfies
    \begin{align*}
        \log|\mathcal{C}_{\mathcal{F}_{\mathrm{head}}}|
        \lesssim[\Upsilon^{QK}(\epsilon^{QK})+\Upsilon^{V}(\epsilon^{V})]N\log(nT)
        .
    \end{align*}
    Consider a linear function class
    \[
        \mathcal{F}_{M}
        =\{f:\mathbb{R}^{T\times N}\to\mathbb{R}^{T\times N}\mid f(X)=XW^{M},\ W^{M}\in\mathcal{W}^{M}\}
        ,
    \]
    and define $\mathcal{C}_{\mathcal{F}_{M}}(\tilde{f}_{\mathrm{head}})$ as an $\epsilon^{M}$-cover of $\mathcal{F}_{M}$ on $\{\phi(\Pi_{\mathrm{norm}}(\tilde{f}_{\mathrm{head}}(X_{i})))\}_{i\in[n]}$. Note that the input $\{\phi(\Pi_{\mathrm{norm}}(\tilde{f}_{\mathrm{head}}(X_{i})))\}_{i\in[n]}$ has the norm bound
    \begin{align*}
        &\max_{i\in[n]}\|\phi(\Pi_{\mathrm{norm}}(\tilde{f}_{\mathrm{head}}(X_{i})))\|_{2\to\infty}
        =\max_{i\in[n]}\|\phi(\Pi_{\mathrm{norm}}(\tilde{f}_{\mathrm{head}}(X_{i})))-\phi(0)\|_{2\to\infty}\\
        &\le L_{\phi}\max_{i\in[n]}\|\Pi_{\mathrm{norm}}(\tilde{f}_{\mathrm{head}}(X_{i}))\|_{2\to\infty}
        \le L_{\phi}
        ,
    \end{align*}
    thanks to the normalization $\Pi_{\mathrm{norm}}$. Thus, by Theorem~\ref{thm:ParametricInterpolation2inf}, we can choose such a cover with cardinality
    \[
        \log|\mathcal{C}_{\mathcal{F}_{M}}(\tilde{f}_{\mathrm{head}})|
        \lesssim\Upsilon^{M}(\epsilon^{M})N\log(nT)
        .
    \]
    Define
    \[
        \mathcal{C}_{\mathcal{F}_{\mathrm{block}}}
        =\{X\mapsto\Pi_{\mathrm{norm}}\circ\tilde{f}_{M}\circ\phi\circ\Pi_{\mathrm{norm}}\circ\tilde{f}_{\mathrm{head}}(X)\mid\tilde{f}_{\mathrm{head}}\in\mathcal{C}_{\mathcal{F}_{\mathrm{head}}},\tilde{f}_{M}\in\mathcal{C}_{\mathcal{F}_{M}}(\tilde{f}_{\mathrm{head}})\}
        .
    \]
    Then, for any $f_{\mathrm{block}}=\Pi_{\mathrm{norm}}\circ f_{M}\circ\phi\circ\Pi_{\mathrm{norm}}\circ f_{\mathrm{head}}\in\mathcal{F}_{\mathrm{block}}$, there exist $\tilde{f}_{\mathrm{head}}\in\mathcal{C}_{\mathcal{F}_{\mathrm{head}}}$ and $\tilde{f}_{M}\in\mathcal{C}_{\mathcal{F}_{M}}(\tilde{f}_{\mathrm{head}})$ that satisfy
    \begin{align*}
        &\|f_{\mathrm{head}}(X_{i})-\tilde{f}_{\mathrm{head}}(X_{i})\|_{2\to\infty}
        \le\epsilon_{\mathrm{head}}
        ,\\
        &\|f_{M}(\phi(\Pi_{\mathrm{norm}}(\tilde{f}_{\mathrm{head}}(X_{i}))))-\tilde{f}_{M}(\phi(\Pi_{\mathrm{norm}}(\tilde{f}_{\mathrm{head}}(X_{i}))))\|_{2\to\infty}
        \le\epsilon^{M}
        .
    \end{align*}
    By Lemma~\ref{lem:Edelman2022LemA9} and Eq.~\eqref{eq:ActivationMatrixLipschitz}, we have
    \begin{align*}
        &\|f_{\mathrm{block}}(X_{i})-\tilde{f}_{\mathrm{block}}(X_{i})\|_{2\to\infty}\\
        &=\|\Pi_{\mathrm{norm}}\circ f_{M}\circ\phi\circ\Pi_{\mathrm{norm}}\circ f_{\mathrm{head}}(X_{i})-\Pi_{\mathrm{norm}}\circ\tilde{f}_{M}\circ\phi\circ\Pi_{\mathrm{norm}}\circ\tilde{f}_{\mathrm{head}}(X_{i})\|_{2\to\infty}\\
        &\le\|f_{M}(\phi(\Pi_{\mathrm{norm}}(f_{\mathrm{head}}(X_{i}))))-\tilde{f}_{M}(\phi(\Pi_{\mathrm{norm}}(\tilde{f}_{\mathrm{head}}(X_{i}))))\|_{2\to\infty}
        .
    \end{align*}
    By a computation similar to the proof of Lemma~\ref{lem:CoveringCompositeFunctions}, we can bound the norm on the right-hand side as
    \begin{align*}
        &\|f_{M}(\phi(\Pi_{\mathrm{norm}}(f_{\mathrm{head}}(X_{i}))))-\tilde{f}_{M}(\phi(\Pi_{\mathrm{norm}}(\tilde{f}_{\mathrm{head}}(X_{i}))))\|_{2\to\infty}\\
        &\le\|f_{M}(\phi(\Pi_{\mathrm{norm}}(f_{\mathrm{head}}(X_{i}))))-f_{M}(\phi(\Pi_{\mathrm{norm}}(\tilde{f}_{\mathrm{head}}(X_{i}))))\|_{2\to\infty}\\
        &\quad+\|f_{M}(\phi(\Pi_{\mathrm{norm}}(\tilde{f}_{\mathrm{head}}(X_{i}))))-\tilde{f}_{M}(\phi(\Pi_{\mathrm{norm}}(\tilde{f}_{\mathrm{head}}(X_{i}))))\|_{2\to\infty}\\
        &\le\|[\phi(\Pi_{\mathrm{norm}}(f_{\mathrm{head}}(X_{i})))-\phi(\Pi_{\mathrm{norm}}(\tilde{f}_{\mathrm{head}}(X_{i})))]W^{M}\|_{2\to\infty}+\epsilon^{M}
        .
    \end{align*}
    By Lemma~\ref{lem:MatrixProduct2Inf2} and Assumption~\ref{asm:Activation}, we further bound the first term as
    \begin{align*}
        &\|[\phi(\Pi_{\mathrm{norm}}(f_{\mathrm{head}}(X_{i})))-\phi(\Pi_{\mathrm{norm}}(\tilde{f}_{\mathrm{head}}(X_{i})))]W^{M}\|_{2\to\infty}\\
        &\le L_{\phi}\|f_{\mathrm{head}}(X_{i})-\tilde{f}_{\mathrm{head}}(X_{i})\|_{2\to\infty}\|W^{M}\|_{2}\\
        &\le L_{\phi}C_{2}^{M}\epsilon_{\mathrm{head}}
        .
    \end{align*}
    Hence, we have
    \[
        \|f_{\mathrm{block}}(X_{i})-\tilde{f}_{\mathrm{block}}(X_{i})\|_{2\to\infty}
        \le L_{\phi}C_{2}^{M}\epsilon_{\mathrm{head}}+\epsilon^{M}
        ,
    \]
    which implies that $\mathcal{C}_{\mathcal{F}_{\mathrm{block}}}$ is an $L_{\phi}C_{2}^{M}\epsilon_{\mathrm{head}}+\epsilon^{M}$-cover of $\mathcal{F}_{\mathrm{block}}$ on $\{X_{i}\}_{i\in[n]}$. Finally, we obtain the entropy bound
    \begin{align*}
        &\log\mathcal{N}_{\infty}(\mathcal{F}_{\mathrm{block}},\|\cdot\|_{2\to\infty},2L_{\phi}C_{2}^{V}C_{2}^{M}B_{n,(2\to\infty)}^{2}\epsilon^{QK}+L_{\phi}C_{2}^{M}\epsilon^{V}+\epsilon^{M};\{X_{i}\}_{i\in[n]})\\
        &\le\log|\mathcal{C}_{\mathcal{F}_{\mathrm{block}}}|\\
        &\le\log|\mathcal{C}_{\mathcal{F}_{\mathrm{head}}}|+\sup_{\tilde{f}_{\mathrm{head}}\in\mathcal{C}_{\mathcal{F}_{\mathrm{head}}}}\log|\mathcal{C}_{\mathcal{F}_{M}}(\tilde{f}_{\mathrm{head}})|\\
        &\lesssim(\Upsilon^{QK}(\epsilon^{QK})+\Upsilon^{V}(\epsilon^{V})+\Upsilon^{M}(\epsilon^{M}))N\log(nT)
        .
    \end{align*}
\end{proof}

\subsection{Covering number bounds for multi-layer Transformers}

We next consider a multi-layer Transformer in Eq.~\eqref{eq:DefMultiLayerTf}. For each $\star\in\{QK,V,M\}$ and $\ell\in[L]$, set
\[
    \mathcal{W}^{\star,(\ell)}(p^{\star,(\ell)},C_{\mathrm{s}}^{\star,(\ell)})
    =\{W\in\mathbb{R}^{N\times N}\mid\|W\|_{\mathrm{s},p^{\star,(\ell)}}^{p^{\star,(\ell)}}\le C_{\mathrm{s}}^{\star,(\ell)},\ \|W\|_{2}\le C_{2}^{\star,(\ell)}\}
    .
\]
We also define
\[
    \mathcal{W}^{(\ell)}(\bm{p}^{(\ell)},\bm{C}_{\mathrm{s}}^{(\ell)})
    =\mathcal{W}^{QK,(\ell)}(p^{QK,(\ell)},C_{\mathrm{s}}^{QK,(\ell)})\times\mathcal{W}^{V,(\ell)}(p^{V,(\ell)},C_{\mathrm{s}}^{V,(\ell)})\times\mathcal{W}^{M,(\ell)}(p^{M,(\ell)},C_{\mathrm{s}}^{M,(\ell)})
\]
and
\[
    \mathcal{W}^{(1:\ell)}(\bm{p}^{(1:\ell)},\bm{C}_{\mathrm{s}}^{(1:\ell)})
    =\mathcal{W}^{(1)}(\bm{p}^{(1)},\bm{C}_{\mathrm{s}}^{(1)})\times\cdots\times\mathcal{W}^{(\ell)}(\bm{p}^{(\ell)},\bm{C}_{\mathrm{s}}^{(\ell)})
    \qquad(\ell=2,\dots,L)
    .
\]
We then define the class of $\ell$-layer Transformers by
\[
    \mathcal{F}_{\mathrm{tf}}^{(\ell)}
    =\{f:\mathbb{R}^{T\times N}\to\mathbb{R}^{T\times N}\mid f(X)=f_{\mathrm{tf}}^{(\ell)}(X;W^{(1:\ell)}),\ W^{(1:\ell)}\in\mathcal{W}^{(1:\ell)}(\bm{p}^{(1:\ell)},\bm{C}_{\mathrm{s}}^{(1:\ell)})\}
    .
\]
For each $\ell\in[L]$, we also define
\[
    \mathcal{F}_{\mathrm{block}}^{(\ell)}
    =\{f:\mathbb{R}^{T\times N}\to\mathbb{R}^{T\times N}\mid f(X)=f_{\mathrm{block}}(X;W^{(\ell)}),W^{(\ell)}\in\mathcal{W}^{(\ell)}(\bm{p}^{(\ell)},\bm{C}_{\mathrm{s}}^{(\ell)})\}
    .
\]

\begin{lemma}\label{lem:MultiLayerTFLipschitz}
    For every $\ell\in\{2,\dots,L\}$, consider any $f\in\mathcal{F}_{\mathrm{block}}^{(\ell)}$. Suppose $Z,Z'\in\mathbb{R}^{T\times N}$ satisfy
    \[
        \|Z\|_{2\to\infty}
        \le B_{2\to\infty}
        ,\qquad
        \|Z'\|_{2\to\infty}
        \le B_{2\to\infty}
        .
    \]
    Then, we have
    \[
        \|f(Z)-f(Z')\|_{2\to\infty}
        \le L_{\phi}C_{2}^{M,(\ell)}C_{2}^{V,(\ell)}(1+4C_{2}^{QK,(\ell)}B_{2\to\infty}^{2})\|Z-Z'\|_{2\to\infty}
        .
    \]
\end{lemma}
\begin{proof}
    Take any $f(\cdot)=f_{\mathrm{block}}(\ \cdot\ ;W^{QK,(\ell)},W^{V,(\ell)},W^{M,(\ell)})\in\mathcal{F}_{\mathrm{block}}^{(\ell)}$. Then, by Lemmas~\ref{lem:Edelman2022LemA9},~\ref{lem:MatrixProduct2Inf2}, and Eq.~\eqref{eq:ActivationMatrixLipschitz}, we have
    \begin{align*}
        &\|f(Z)-f(Z')\|_{2\to\infty}\\
        &\le\|\phi(\Pi_{\mathrm{norm}}(f_{\mathrm{head}}(Z;W^{QK,(\ell)},W^{V,(\ell)})))-\phi(\Pi_{\mathrm{norm}}(f_{\mathrm{head}}(Z';W^{QK,(\ell)},W^{V,(\ell)})))\|_{2\to\infty}C_{2}^{M,(\ell)}\\
        &\le L_{\phi}C_{2}^{M,(\ell)}\|f_{\mathrm{head}}(Z;W^{QK,(\ell)},W^{V,(\ell)})-f_{\mathrm{head}}(Z';W^{QK,(\ell)},W^{V,(\ell)})\|_{2\to\infty}
        .
    \end{align*}
    Hence it suffices to bound the distance between the head outputs. Write
    \[
        A
        =\mathrm{SoftMax}(ZW^{QK,(\ell)}Z^{\top})
        ,\qquad
        A'
        =\mathrm{SoftMax}(Z'W^{QK,(\ell)}Z'^{\top})
        .
    \]
    Then, we have
    \begin{align}
        \notag
        &\|f_{\mathrm{head}}(Z;W^{QK,(\ell)},W^{V,(\ell)})-f_{\mathrm{head}}(Z';W^{QK,(\ell)},W^{V,(\ell)})\|_{2\to\infty}\\
        \notag
        &=\|AZW^{V,(\ell)}-A'Z'W^{V,(\ell)}\|_{2\to\infty}\\
        \label{eq:HeqdLipschitz}
        &\le\|A(Z-Z')W^{V,(\ell)}\|_{2\to\infty}+\|(A-A')Z'W^{V,(\ell)}\|_{2\to\infty}
        .
    \end{align}
    For the first term, using Lemmas~\ref{lem:SoftmaxMatrixInfinity}, \ref{lem:MatrixProduct2InfInf} and \ref{lem:MatrixProduct2Inf2}, we have
    \begin{equation}\label{eq:HeqdLipschitz1}
        \|A(Z-Z')W^{V,(\ell)}\|_{2\to\infty}
        \le\|(Z-Z')W^{V,(\ell)}\|_{2\to\infty}
        \le C_{2}^{V,(\ell)}\|Z-Z'\|_{2\to\infty}
        .
    \end{equation}
    For the second term, we first note that by Lemmas~\ref{lem:VecMatrixProduct} and \ref{lem:MatrixProduct2Inf2}, we have
    \begin{align*}
        \|(A-A')Z'W^{V,(\ell)}\|_{2\to\infty}
        &=\max_{t\in[T]}\|(A_{t\cdot}-A'_{t\cdot})Z'W^{V,(\ell)}\|_{2}\\
        &\le\max_{t\in[T]}\|(A_{t\cdot}-A'_{t\cdot})^{\top}\|_{1}\|Z'W^{V,(\ell)}\|_{2\to\infty}\\
        &\le\max_{t\in[T]}\|(A_{t\cdot}-A'_{t\cdot})^{\top}\|_{1}\|Z'\|_{2\to\infty}\|W^{V,(\ell)}\|_{2}\\
        &\le\max_{t\in[T]}\|(A_{t\cdot}-A'_{t\cdot})^{\top}\|_{1}B_{2\to\infty}C_{2}^{V,(\ell)}
        .
    \end{align*}
    It remains to bound the row difference of the attention matrices. The rows $A_{t\cdot}$ and $A'_{t\cdot}$ are written as
    \[
        A_{t\cdot}
        =\mathrm{SoftMax}((ZW^{QK,(\ell)}Z^{\top})_{t\cdot})
        ,\qquad
        A'_{t\cdot}
        =\mathrm{SoftMax}((Z'W^{QK,(\ell)}Z'^{\top})_{t\cdot})
        .
    \]
    Thus, Lemma~\ref{lem:SoftmaxLipschitzVector1} yields
    \begin{align*}
        \|(A_{t\cdot}-A'_{t\cdot})^{\top}\|_{1}
        &\le2\|((ZW^{QK,(\ell)}Z^{\top})_{t\cdot}-(Z'W^{QK,(\ell)}Z'^{\top})_{t\cdot})^{\top}\|_{\infty}\\
        &=2\max_{s\in[T]}|Z_{t\cdot}W^{QK,(\ell)}(Z_{s\cdot})^{\top}-Z'_{t\cdot}W^{QK,(\ell)}(Z'_{s\cdot})^{\top}|
        .
    \end{align*}
    For each $s,t\in[T]$, we decompose
    \begin{align*}
        &|Z_{t\cdot}W^{QK,(\ell)}(Z_{s\cdot})^{\top}-Z'_{t\cdot}W^{QK,(\ell)}(Z'_{s\cdot})^{\top}|\\
        &\le|Z_{t\cdot}W^{QK,(\ell)}\left(Z_{s\cdot}-Z'_{s\cdot}\right)^{\top}|+|\left(Z_{t\cdot}-Z'_{t\cdot}\right)W^{QK,(\ell)}(Z'_{s\cdot})^{\top}|\\
        &\le\|Z_{t\cdot}\|_{2}\|W^{QK,(\ell)}\left(Z_{s\cdot}-Z'_{s\cdot}\right)^{\top}\|_{2}+\|Z_{t\cdot}-Z'_{t\cdot}\|_{2}\|W^{QK,(\ell)}(Z'_{s\cdot})^{\top}\|_{2}\\
        &\le \|Z_{t\cdot}\|_{2}\|W^{QK,(\ell)}\|_{2}\|Z_{s\cdot}-Z'_{s\cdot}\|_{2}+\|Z_{t\cdot}-Z'_{t\cdot}\|_{2}\|W^{QK,(\ell)}\|_{2}\|Z'_{s\cdot}\|_{2}\\
        &\le B_{2\to\infty}C_{2}^{QK,(\ell)}(\|Z_{s\cdot}-Z'_{s\cdot}\|_{2}+\|Z_{t\cdot}-Z'_{t\cdot}\|_{2})
        .
    \end{align*}
    Substituting these results into the previous inequality yields
    \begin{align*}
        \|(A-A')Z'W^{V,(\ell)}\|_{2\to\infty}
        &\le2C_{2}^{QK,(\ell)}C_{2}^{V,(\ell)}B_{2\to\infty}^{2}\max_{t\in[T]}\max_{s\in[T]}(\|Z_{s\cdot}-Z'_{s\cdot}\|_{2}+\|Z_{t\cdot}-Z'_{t\cdot}\|_{2})\\
        &\le4C_{2}^{QK,(\ell)}C_{2}^{V,(\ell)}B_{2\to\infty}^{2}\|Z-Z'\|_{2\to\infty}
        .
    \end{align*}
    Combining this with Eq.~\eqref{eq:HeqdLipschitz} and \eqref{eq:HeqdLipschitz1}, we have
    \begin{align*}
        &\|f_{\mathrm{head}}(Z;W^{QK,(\ell)},W^{V,(\ell)})-f_{\mathrm{head}}(Z';W^{QK,(\ell)},W^{V,(\ell)})\|_{2\to\infty}\\
        &\le C_{2}^{V,(\ell)}(1+4C_{2}^{QK,(\ell)}B_{2\to\infty}^{2})\|Z-Z'\|_{2\to\infty}
        ,
    \end{align*}
    which implies
    \[
        \|f(Z)-f(Z')\|_{2\to\infty}
        \le L_{\phi}C_{2}^{M,(\ell)}C_{2}^{V,(\ell)}(1+4C_{2}^{QK,(\ell)}B_{2\to\infty}^{2})\|Z-Z'\|_{2\to\infty}
    \]
    as desired.
\end{proof}

\begin{proposition}\label{prop:EntropyMultiLayerTf}
    Suppose $\max_{i\in[n]}\|X_{i}\|_{2\to\infty}\le B_{n,(2\to\infty)}$ holds. For each $\ell\in[L]$, fix arbitrary $\epsilon^{QK,(\ell)},\epsilon^{V,(\ell)},\epsilon^{M,(\ell)}>0$, and define
    \[
        \epsilon^{(\ell)}
        =2L_{\phi}C_{2}^{V,(\ell)}C_{2}^{M,(\ell)}B_{n,(2\to\infty)}^{2\delta_{\ell=1}}\epsilon^{QK,(\ell)}+L_{\phi}C_{2}^{M,(\ell)}\epsilon^{V,(\ell)}+\epsilon^{M,(\ell)}
        ,
    \]
    where $\delta_{\ell=1}$ is an indicator function that equals $1$ if $\ell=1$ and $0$ otherwise. We also define
    \begin{align*}
        &\Upsilon^{\star,(\ell)}(\epsilon^{\star,(\ell)})
        =\left(\frac{(C_{\mathrm{s}}^{\star,(\ell)})^{2}B_{n,(2\to\infty)}^{2\delta_{\ell=1}p^{\star,(\ell)}}N^{p^{\star,(\ell)}}}{(\epsilon^{\star,(\ell)})^{2p^{\star,(\ell)}}}\right)^{\frac{1}{p^{\star,(\ell)}+2}}
        ,\\
        &\Upsilon^{M,(\ell)}(\epsilon^{M,(\ell)})
        =\left(\frac{(C_{\mathrm{s}}^{M,(\ell)})^{2}L_{\phi}^{2p^{M,(\ell)}}N^{p^{M,(\ell)}}}{(\epsilon^{M,(\ell)})^{2p^{M,(\ell)}}}\right)^{\frac{1}{p^{M,(\ell)}+2}}
    \end{align*}
    with $\star\in\{QK,V\}$. Then, we have
    \begin{align*}
        &\log\mathcal{N}_{\infty}(\mathcal{F}_{\mathrm{tf}}^{(L)},\|\cdot\|_{2\to\infty},\eta^{(L)};\{X_{i}\}_{i\in[n]})\\
        &\lesssim\sum_{k=1}^{L}\left(\Upsilon^{QK,(k)}(\epsilon^{QK,(k)})+\Upsilon^{V,(k)}(\epsilon^{V,(k)})+\Upsilon^{M,(k)}(\epsilon^{M,(k)})\right)N\log(nT)
        ,
    \end{align*}
    where $\eta^{(L)}$ is defined by
    \[
        \eta^{(L)}
        =\sum_{j=1}^{L-1}L_{\phi}^{L-j}\epsilon^{(j)}\prod_{k=j+1}^{L}C_{2}^{M,(k)}C_{2}^{V,(k)}(1+4C_{2}^{QK,(k)})+\epsilon^{(L)}
        .
    \]
\end{proposition}
\begin{proof}
    First, note that for all $\ell\in[L] $ and $f_{\mathrm{tf}}^{(\ell)}\in\mathcal{F}_{\mathrm{tf}}^{(\ell)}$, we have
    \begin{equation}\label{eq:TfEllNormBound}
        \max_{i\in[n]}\|f_{\mathrm{tf}}^{(\ell)}(X_{i})\|_{2\to\infty}
        \le1
        .
    \end{equation}
    Set
    \begin{align*}
        &\eta^{(1)}
        =\epsilon^{(1)}
        ,\\
        &\eta^{(\ell)}
        =\sum_{j=1}^{\ell-1}L_{\phi}^{\ell-j}\epsilon^{(j)}\prod_{k=j+1}^{\ell}C_{2}^{M,(k)}C_{2}^{V,(k)}(1+4C_{2}^{QK,(k)})+\epsilon^{(\ell)}
        \quad(\ell=2,\dots,L)
        .
    \end{align*}
    We show by induction that for each $\ell\in[L]$, there exists a proper $\eta^{(\ell)}$-cover $\mathcal{C}_{\mathcal{F}_{\mathrm{tf}}^{(\ell)}}$ of $\mathcal{F}_{\mathrm{tf}}^{(\ell)}$ on $\{X_{i}\}_{i\in[n]}$ with cardinality
    \begin{equation}\label{eq:TfEllEntropy}
        \log|\mathcal{C}_{\mathcal{F}_{\mathrm{tf}}^{(\ell)}}|
        \lesssim\sum_{k=1}^{\ell}\left(\Upsilon^{QK,(k)}(\epsilon^{QK,(k)})+\Upsilon^{V,(k)}(\epsilon^{V,(k)})+\Upsilon^{M,(k)}(\epsilon^{M,(k)})\right)N\log(nT)
        .
    \end{equation}
    For the base case $\ell=1$, the claim follows from Proposition~\ref{prop:EntropyTfBlock}. Now, suppose the claim holds for some $\ell\in[L-1]$. We show that it also holds for $\ell+1$.

    Suppose $\mathcal{C}_{\mathcal{F}_{\mathrm{tf}}^{(\ell)}}$ is a $\eta^{(\ell)}$-cover of $\mathcal{F}_{\mathrm{tf}}^{(\ell)}$ satisfying Eq.~\eqref{eq:TfEllEntropy}. For each $\tilde{f}_{\mathrm{tf}}^{(\ell)}\in\mathcal{C}_{\mathcal{F}_{\mathrm{tf}}^{(\ell)}}$, let $\mathcal{C}_{\mathcal{F}_{\mathrm{block}}^{(\ell+1)}}(\tilde{f}_{\mathrm{tf}}^{(\ell)})$ be an $\epsilon^{(\ell+1)}$-cover of $\mathcal{F}_{\mathrm{block}}^{(\ell+1)}$ on $\{\tilde{f}_{\mathrm{tf}}^{(\ell)}(X_{i})\}_{i\in[n]}$. Note that by Proposition~\ref{prop:EntropyTfBlock} and Eq.~\eqref{eq:TfEllNormBound}, we can choose such a cover with cardinality
    \begin{align*}
        &\log|\mathcal{C}_{\mathcal{F}_{\mathrm{block}}^{(\ell+1)}}(\tilde{f}_{\mathrm{tf}}^{(\ell)})|\\
        &\lesssim(\Upsilon^{QK,(\ell+1)}(\epsilon^{QK,(\ell+1)})+\Upsilon^{V,(\ell+1)}(\epsilon^{V,(\ell+1)})+\Upsilon^{M,(\ell+1)}(\epsilon^{M,(\ell+1)}))N\log(nT)
        .
    \end{align*}
    By Lemma~\ref{lem:MultiLayerTFLipschitz} and Eq.~\eqref{eq:TfEllNormBound}, we have for any $i\in[n]$, $\ell\in[L-1]$, $f_{\mathrm{block}}^{(\ell+1)}\in\mathcal{F}_{\mathrm{block}}^{(\ell+1)}$ and $f_{\mathrm{tf}}^{(\ell)},\tilde{f}_{\mathrm{tf}}^{(\ell)}\in\mathcal{F}_{\mathrm{tf}}^{(\ell)}$ that
    \begin{align*}
        &\|f_{\mathrm{block}}^{(\ell+1)}(f_{\mathrm{tf}}^{(\ell)}(X_{i}))-f_{\mathrm{block}}^{(\ell+1)}(\tilde{f}_{\mathrm{tf}}^{(\ell)}(X_{i}))\|_{2\to\infty}\\
        &\le L_{\phi}C_{2}^{M,(\ell+1)}C_{2}^{V,(\ell+1)}(1+4C_{2}^{QK,(\ell+1)})\|f_{\mathrm{tf}}^{(\ell)}(X_{i})-\tilde{f}_{\mathrm{tf}}^{(\ell)}(X_{i})\|_{2\to\infty}
        .
    \end{align*}
    The recursive relation
    \[
        \eta^{(\ell+1)}
        =L_{\phi}C_{2}^{M,(\ell+1)}C_{2}^{V,(\ell+1)}(1+4C_{2}^{QK,(\ell+1)})\eta^{(\ell)}+\epsilon^{(\ell+1)}
        .
    \]
    holds. Therefore, Lemma~\ref{lem:CoveringCompositeFunctions} implies that
    \[
        \mathcal{C}_{\mathcal{F}_{\mathrm{tf}}^{(\ell+1)}}
        =\{\tilde{f}_{\mathrm{block}}^{(\ell+1)}\circ\tilde{f}_{\mathrm{tf}}^{(\ell)}\mid\tilde{f}_{\mathrm{tf}}^{(\ell)}\in\mathcal{C}_{\mathcal{F}_{\mathrm{tf}}^{(\ell)}},\tilde{f}_{\mathrm{block}}^{(\ell+1)}\in\mathcal{C}_{\mathcal{F}_{\mathrm{block}}^{(\ell+1)}}(\tilde{f}_{\mathrm{tf}}^{(\ell)})\}
    \]
    is an $\eta^{(\ell+1)}$-cover of $\mathcal{F}_{\mathrm{tf}}^{(\ell+1)}$. Moreover, we can bound the cardinality of $\mathcal{C}_{\mathcal{F}_{\mathrm{tf}}^{(\ell+1)}}$ as
    \begin{align*}
        &\log|\mathcal{C}_{\mathcal{F}_{\mathrm{tf}}^{(\ell+1)}}|\\
        &\le\log|\mathcal{C}_{\mathcal{F}_{\mathrm{tf}}^{(\ell)}}|+\sup_{\tilde{f}_{\mathrm{tf}}^{(\ell)}\in\mathcal{C}_{\mathcal{F}_{\mathrm{tf}}^{(\ell)}}}\log|\mathcal{C}_{\mathcal{F}_{\mathrm{block}}^{(\ell+1)}}(\tilde{f}_{\mathrm{tf}}^{(\ell)})|\\
        &\lesssim\sum_{k=1}^{\ell+1}\left(\Upsilon^{QK,(k)}(\epsilon^{QK,(k)})+\Upsilon^{V,(k)}(\epsilon^{V,(k)})+\Upsilon^{M,(k)}(\epsilon^{M,(k)})\right)N\log(nT)
        .
    \end{align*}
\end{proof}

\subsection{Covering number bounds for the scalar outputs of multi-layer Transformers}

For a vector of Schatten indices
\[
    \bm{p}
    =(p^{QK,(1)},p^{V,(1)},p^{M,(1)},\dots,p^{QK,(L)},p^{V,(L)},p^{M,(L)})
    \in[0,2]^{3L}
\]
and a vector of Schatten-quantity radii
\[
    \bm{C}_{\mathrm{s}}
    =(C_{\mathrm{s}}^{QK,(1)},C_{\mathrm{s}}^{V,(1)},C_{\mathrm{s}}^{M,(1)},\dots,C_{\mathrm{s}}^{QK,(L)},C_{\mathrm{s}}^{V,(L)},C_{\mathrm{s}}^{M,(L)})
    \in(0,\infty)^{3L},
\]
consider the class of scalar outputs
\[
    \mathcal{F}_{\mathrm{out}}^{(\bm{p},\bm{C}_{\mathrm{s}})}
    =\{f_{\mathrm{out}}(\ \cdot\ ;W^{(1:L)},w):\mathbb{R}^{T\times N}\to\mathbb{R}\mid W^{(1:L)}\in\mathcal{W}^{(1:L)}(\bm{p},\bm{C}_{\mathrm{s}}),\ w\in\mathbb{R}^{N},\ \|w\|_{2}\le C_{2}^{\mathrm{out}}\}
    ,
\]
where $f_{\mathrm{out}}(\ \cdot\ ;W^{(1:L)},w)$ is defined in Eq.~\eqref{eq:DefTfOut}. For $\ell\in[L]$, we also define
\begin{equation}\label{eq:DefAlpha}
    \alpha^{(\ell)}
    =\prod_{k=\ell+1}^{L}L_{\phi}C_{2}^{M,(k)}C_{2}^{V,(k)}(1+4C_{2}^{QK,(k)})
    ,
\end{equation}
where we adopt the convention that an empty product is equal to $1$.

\begin{proposition}\label{prop:EntropyTfOut}
    Assume the same setting as in Proposition~\ref{prop:EntropyMultiLayerTf} for the localized class $\mathcal{F}_{\mathrm{out}}^{(\bm{p},\bm{C}_{\mathrm{s}})}$. Then, for any $\epsilon^{\mathrm{out}}>0$, it holds that
    \begin{align*}
        &\log\mathcal{N}_{\infty}(\mathcal{F}_{\mathrm{out}}^{(\bm{p},\bm{C}_{\mathrm{s}})},|\cdot|,C_{2}^{\mathrm{out}}\eta^{(L)}+\epsilon^{\mathrm{out}};\{X_{i}\}_{i\in[n]})\\
        &\lesssim\sum_{k=1}^{L}\left(\Upsilon^{QK,(k)}(\epsilon^{QK,(k)})+\Upsilon^{V,(k)}(\epsilon^{V,(k)})+\Upsilon^{M,(k)}(\epsilon^{M,(k)})\right)N\log(nT)+\left(\frac{C_{2}^{\mathrm{out}}}{\epsilon^{\mathrm{out}}}\right)^{2}\log(n)
        .
    \end{align*}
\end{proposition}
\begin{proof}
    Proposition~\ref{prop:EntropyMultiLayerTf} implies that there exists an $\eta^{(L)}$-cover $\mathcal{C}_{\mathcal{F}_{\mathrm{tf}}^{(L)}}$ of $\mathcal{F}_{\mathrm{tf}}^{(L)}$ on $\{X_{i}\}_{i\in[n]}$ that satisfies
    \[
        \log|\mathcal{C}_{\mathcal{F}_{\mathrm{tf}}^{(L)}}|
        \lesssim\sum_{k=1}^{L}\left(\Upsilon^{QK,(k)}(\epsilon^{QK,(k)})+\Upsilon^{V,(k)}(\epsilon^{V,(k)})+\Upsilon^{M,(k)}(\epsilon^{M,(k)})\right)N\log(nT)
        .
    \]
    Moreover, for every $\tilde{f}_{\mathrm{tf}}^{(L)}\in\mathcal{C}_{\mathcal{F}_{\mathrm{tf}}^{(L)}}$, it holds that
    \[
        \max_{i\in[n]}\|[\tilde{f}_{\mathrm{tf}}^{(L)}(X_{i})]_{[\mathrm{CLS}]}\|_{2}
        \le\max_{i\in[n]}\|\tilde{f}_{\mathrm{tf}}^{(L)}(X_{i})\|_{2\to\infty}
        \le1
        .
    \]
    Fix $\tilde{f}_{\mathrm{tf}}^{(L)}\in\mathcal{C}_{\mathcal{F}_{\mathrm{tf}}^{(L)}}$. Consider the linear function class
    \[
        \mathcal{F}_{w}
        =\{f:\mathbb{R}^{N}\to\mathbb{R}\mid f(x)=w^{\top}x,w\in\mathbb{R}^{N},\|w\|_{2}\le C_{2}^{\mathrm{out}}\}
        .
    \]
    Applying Proposition~\ref{prop:Zhang2002Thm4}, we obtain an $\epsilon^{\mathrm{out}}$-cover $\mathcal{C}_{\mathcal{F}_{w}}(\tilde{f}_{\mathrm{tf}}^{(L)})$ of $\mathcal{F}_{w}$ on $\{[\tilde{f}_{\mathrm{tf}}^{(L)}(X_{i})]_{[\mathrm{CLS}]}\}_{i\in[n]}$ satisfying
    \[
        \log|\mathcal{C}_{\mathcal{F}_{w}}(\tilde{f}_{\mathrm{tf}}^{(L)})|
        \lesssim\left(\frac{C_{2}^{\mathrm{out}}}{\epsilon^{\mathrm{out}}}\right)^{2}\log(n)
        .
    \]
    Now define
    \[
        \mathcal{C}_{\mathcal{F}_{\mathrm{out}}^{(\bm{p},\bm{C}_{\mathrm{s}})}}
        =\{f:\mathbb{R}^{T\times N}\to\mathbb{R}\mid f(X)=\tilde{w}^{\top}[\tilde{f}_{\mathrm{tf}}^{(L)}(X)]_{[\mathrm{CLS}]},\ \tilde{f}_{\mathrm{tf}}^{(L)}\in\mathcal{C}_{\mathcal{F}_{\mathrm{tf}}^{(L)}},\ \tilde{w}^{\top}(\cdot)\in\mathcal{C}_{\mathcal{F}_{w}}(\tilde{f}_{\mathrm{tf}}^{(L)})\}
        .
    \]
    Take any $f_{\mathrm{out}}(\ \cdot\ ;W^{(1:L)},w)\in\mathcal{F}_{\mathrm{out}}^{(\bm{p},\bm{C}_{\mathrm{s}})}$. By the definition of $\mathcal{C}_{\mathcal{F}_{\mathrm{tf}}^{(L)}}$, there exists $\tilde{f}_{\mathrm{tf}}^{(L)}\in\mathcal{C}_{\mathcal{F}_{\mathrm{tf}}^{(L)}}$
    that satisfies
    \[
        \max_{i\in[n]}\|f_{\mathrm{tf}}^{(L)}(X_{i};W^{(1:L)})-\tilde{f}_{\mathrm{tf}}^{(L)}(X_{i})\|_{2\to\infty}
        \le\eta^{(L)}
        .
    \]
    Also, by the definition of $\mathcal{C}_{\mathcal{F}_{w}}(\tilde{f}_{\mathrm{tf}}^{(L)})$, there exists $\tilde{w}^{\top}(\cdot)\in\mathcal{C}_{\mathcal{F}_{w}}(\tilde{f}_{\mathrm{tf}}^{(L)})$
    that satisfies
    \[
        \max_{i\in[n]}|(w-\tilde{w})^{\top}[\tilde{f}_{\mathrm{tf}}^{(L)}(X_{i})]_{[\mathrm{CLS}]}|
        \le\epsilon^{\mathrm{out}}
        .
    \]
    For such $\tilde{f}_{\mathrm{tf}}^{(L)}$ and $\tilde{w}$, we have
    \begin{align*}
        &|f_{\mathrm{out}}(X_{i};W^{(1:L)},w)-\tilde{w}^{\top}[\tilde{f}_{\mathrm{tf}}^{(L)}(X_{i})]_{[\mathrm{CLS}]}|\\
        &\le|w^{\top}([f_{\mathrm{tf}}^{(L)}(X_{i};W^{(1:L)})]_{[\mathrm{CLS}]}-[\tilde{f}_{\mathrm{tf}}^{(L)}(X_{i})]_{[\mathrm{CLS}]})|+|(w-\tilde{w})^{\top}[\tilde{f}_{\mathrm{tf}}^{(L)}(X_{i})]_{[\mathrm{CLS}]}|\\
        &\le\|w\|_{2}\|[f_{\mathrm{tf}}^{(L)}(X_{i};W^{(1:L)})]_{[\mathrm{CLS}]}-[\tilde{f}_{\mathrm{tf}}^{(L)}(X_{i})]_{[\mathrm{CLS}]}\|_{2}+|(w-\tilde{w})^{\top}[\tilde{f}_{\mathrm{tf}}^{(L)}(X_{i})]_{[\mathrm{CLS}]}|\\
        &\le C_{2}^{\mathrm{out}}\|f_{\mathrm{tf}}^{(L)}(X_{i};W^{(1:L)})-\tilde{f}_{\mathrm{tf}}^{(L)}(X_{i})\|_{2\to\infty}+\epsilon^{\mathrm{out}}\\
        &\le C_{2}^{\mathrm{out}}\eta^{(L)}+\epsilon^{\mathrm{out}}
        .
    \end{align*}
    Therefore, $\mathcal{C}_{\mathcal{F}_{\mathrm{out}}^{(\bm{p},\bm{C}_{\mathrm{s}})}}$ is a $(C_{2}^{\mathrm{out}}\eta^{(L)}+\epsilon^{\mathrm{out}})$-cover of
    $\mathcal{F}_{\mathrm{out}}^{(\bm{p},\bm{C}_{\mathrm{s}})}$. Finally, its cardinality is bounded by
    \[
        \log|\mathcal{C}_{\mathcal{F}_{\mathrm{out}}^{(\bm{p},\bm{C}_{\mathrm{s}})}}|
        \le\log|\mathcal{C}_{\mathcal{F}_{\mathrm{tf}}^{(L)}}|+\sup_{\tilde{f}_{\mathrm{tf}}^{(L)}\in\mathcal{C}_{\mathcal{F}_{\mathrm{tf}}^{(L)}}}\log|\mathcal{C}_{\mathcal{F}_{w}}(\tilde{f}_{\mathrm{tf}}^{(L)})|
        ,
    \]
    and substituting the above two entropy bounds proves the claim.
\end{proof}

\subsection{$\epsilon$-entropy bounds for the scalar outputs of multi-layer Transformers}

We now derive the $\epsilon$-entropy bounds for $\mathcal{F}_{\mathrm{out}}^{(\bm{p},\bm{C}_{\mathrm{s}})}$. For $\ell\in[L]$, define
\[
    \alpha^{(\ell)}
    =L_{\phi}^{L-\ell}\prod_{k=\ell+1}^{L}C_{2}^{M,(k)}C_{2}^{V,(k)}(1+4C_{2}^{QK,(k)})
    ,
\]
where we adopt the convention that an empty product is equal to $1$. We also define
\begin{align*}
    &\beta^{QK,(\ell)}
    =L_{\phi}C_{2}^{\mathrm{out}}\alpha^{(\ell)}2C_{2}^{V,(\ell)}C_{2}^{M,(\ell)}B_{n,(2\to\infty)}^{2\delta_{\ell=1}}
    ,\\
    &
    \beta^{V,(\ell)}
    =L_{\phi}C_{2}^{\mathrm{out}}\alpha^{(\ell)}C_{2}^{M,(\ell)}
    ,\qquad
    \beta^{M,(\ell)}
    =C_{2}^{\mathrm{out}}\alpha^{(\ell)}
    .
\end{align*}
Then, we can write
\[
    C_{2}^{\mathrm{out}}\eta^{(L)}
    =\sum_{\ell=1}^{L}(\beta^{QK,(\ell)}\epsilon^{QK,(\ell)}+\beta^{V,(\ell)}\epsilon^{V,(\ell)}+\beta^{M,(\ell)}\epsilon^{M,(\ell)})
    .
\]
Thus, we need to choose $\epsilon^{QK,(\ell)},\epsilon^{V,(\ell)},\epsilon^{M,(\ell)}$ for $\ell\in[L]$ and $\epsilon^{\mathrm{out}}>0$ such that
\begin{equation}\label{eq:CoveringRadius}
    \epsilon
    =\sum_{\ell=1}^{L}(\beta^{QK,(\ell)}\epsilon^{QK,(\ell)}+\beta^{V,(\ell)}\epsilon^{V,(\ell)}+\beta^{M,(\ell)}\epsilon^{M,(\ell)})+\epsilon^{\mathrm{out}}
\end{equation}
holds. Furthermore, setting
\begin{align*}
    &\Lambda^{\star,(\ell)}
    =((C_{\mathrm{s}}^{\star,(\ell)})^{2}B_{n,(2\to\infty)}^{2\delta_{\ell=1}p^{\star,(\ell)}}N^{p^{\star,(\ell)}})^{\frac{1}{p^{\star,(\ell)}+2}}
    \quad(\star\in\{QK,V\})
    ,\\
    &\Lambda^{M,(\ell)}
    =((C_{\mathrm{s}}^{M,(\ell)})^{2}L_{\phi}^{2p^{M,(\ell)}}N^{p^{M,(\ell)}})^{\frac{1}{p^{M,(\ell)}+2}}
    ,\\
    &\nu^{\star,(\ell)}
    =\frac{2p^{\star,(\ell)}}{p^{\star,(\ell)}+2}
    \quad(\star\in\{QK,V,M\})
    ,
\end{align*}
for each $\ell\in[L]$, we can write
\[
    \Upsilon^{\star,(\ell)}(\epsilon^{\star,(\ell)})
    =\Lambda^{\star,(\ell)}(\epsilon^{\star,(\ell)})^{-\nu^{\star,(\ell)}}
    \quad(\star\in\{QK,V,M\})
    .
\]
Thus, by Proposition~\ref{prop:EntropyTfOut}, we obtain the general form
\begin{equation}\label{eq:EpsilonObjectiveFunction}
    \log\mathcal{N}_{\infty}(\mathcal{F}_{\mathrm{out}}^{(\bm{p},\bm{C}_{\mathrm{s}})},|\cdot|,\epsilon;\{X_{i}\}_{i\in[n]})
    \lesssim\sum_{\ell=1}^{L}\sum_{\star\in\{QK,V,M\}}\frac{\Lambda^{\star,(\ell)}}{(\epsilon^{\star,(\ell)})^{\nu^{\star,(\ell)}}}N\log(nT)+\left(\frac{C_{2}^{\mathrm{out}}}{\epsilon^{\mathrm{out}}}\right)^{2}\log(n)
    ,
\end{equation}
where $\{\epsilon^{QK,(\ell)},\epsilon^{V,(\ell)},\epsilon^{M,(\ell)}\}_{\ell\in[L]}$ and $\epsilon^{\mathrm{out}}>0$ satisfy Eq.~\eqref{eq:CoveringRadius}.

We now state two entropy bounds for $\mathcal{F}_{\mathrm{out}}^{(\bm{p},\bm{C}_{\mathrm{s}})}$. The first allows for different $p^{QK,(\ell)},p^{V,(\ell)},p^{M,(\ell)}$ to vary across layers, whereas the second assumes a common $p$ for all layers and all weight matrices.

\begin{theorem}\label{thm:EntropyTfScalarGeneralP}
    Suppose $\max_{i\in[n]}\|X_{i}\|_{2\to\infty}\le B_{n,(2\to\infty)}$ holds. Fix any $\bm{p}\in[0,2]^{3L}$ and $\bm{C}_{\mathrm{s}}\in(0,\infty)^{3L}$. For each $\ell\in[L]$, define
    \[
        \gamma^{QK,(\ell)}
        =2C_{2}^{V,(\ell)}C_{2}^{M,(\ell)}B_{n,(2\to\infty)}^{3\delta_{\ell=1}}
        ,\qquad
        \gamma^{V,(\ell)}
        =C_{2}^{M,(\ell)}B_{n,(2\to\infty)}^{\delta_{\ell=1}}
        ,\qquad
        \gamma^{M,(\ell)}
        =1
        .
    \]
    Then, for any $\epsilon>0$, it holds that
    \begin{align*}
        &\log\mathcal{N}_{\infty}(\mathcal{F}_{\mathrm{out}}^{(\bm{p},\bm{C}_{\mathrm{s}})},|\cdot|,\epsilon;\{X_{i}\}_{i\in[n]})\\
        &\lesssim\sum_{\ell=1}^{L}\sum_{\star\in\{QK,V,M\}}(C_{\mathrm{s}}^{\star,(\ell)})^{\frac{2}{p^{\star,(\ell)}+2}}\left(\frac{L_{\phi}C_{2}^{\mathrm{out}}\gamma^{\star,(\ell)}\alpha^{(\ell)}L}{\epsilon}\right)^{\frac{2p^{\star,(\ell)}}{p^{\star,(\ell)}+2}}N^{1+\frac{p^{\star,(\ell)}}{p^{\star,(\ell)}+2}}\log(nT)\\
        &\quad+\left(\frac{C_{2}^{\mathrm{out}}}{\epsilon}\right)^{2}\log(n)
        .
    \end{align*}
\end{theorem}
\begin{proof}
    For any $\epsilon>0$, set $\epsilon^{\mathrm{out}}=\epsilon/2$. For each $\ell\in[L]$, define
    \begin{align*}
        &\epsilon^{QK,(\ell)}
        =\frac{\epsilon}{12LL_{\phi}C_{2}^{\mathrm{out}}\alpha^{(\ell)}C_{2}^{V,(\ell)}C_{2}^{M,(\ell)}B_{n,(2\to\infty)}^{2\delta_{\ell=1}}}
        ,\\
        &\epsilon^{V,(\ell)}
        =\frac{\epsilon}{6LL_{\phi}C_{2}^{\mathrm{out}}\alpha^{(\ell)}C_{2}^{M,(\ell)}}
        ,\qquad
        \epsilon^{M,(\ell)}
        =\frac{\epsilon}{6LC_{2}^{\mathrm{out}}\alpha^{(\ell)}}
        .
    \end{align*}
    Note that for each $\ell\in[L]$, we have
    \[
        \beta^{QK,(\ell)}\epsilon^{QK,(\ell)}
        =\beta^{V,(\ell)}\epsilon^{V,(\ell)}
        =\beta^{M,(\ell)}\epsilon^{M,(\ell)}
        =\frac{\epsilon}{6L}
        ,
    \]
    which shows that the within-layer allocation is balanced in the sense prescribed above. This also implies that Eq.~\eqref{eq:CoveringRadius} is satisfied. We now substitute this admissible choice into Eq.~\eqref{eq:EpsilonObjectiveFunction}. We obtain
    \begin{align*}
        \frac{\Lambda^{\star,(\ell)}}{(\epsilon^{\star,(\ell)})^{\nu^{\star,(\ell)}}}
        =(C_{\mathrm{s}}^{\star,(\ell)})^{\frac{2}{p^{\star,(\ell)}+2}}\left(\frac{6L_{\phi}C_{2}^{\mathrm{out}}\gamma^{\star,(\ell)}\alpha^{(\ell)}L}{\epsilon}\right)^{\frac{2p^{\star,(\ell)}}{p^{\star,(\ell)}+2}}N^{\frac{p^{\star,(\ell)}}{p^{\star,(\ell)}+2}}
        .
    \end{align*}
    Combining these bounds with Eq.~\eqref{eq:EpsilonObjectiveFunction} proves the claim.
\end{proof}

\begin{theorem}\label{thm:EntropyTfScalarCommonP}
    Suppose $\max_{i\in[n]}\|X_{i}\|_{2\to\infty}\le B_{n,(2\to\infty)}$ holds. For some $p\in[0,2]$, suppose $\bm{p}=(p,\dots,p)$ holds. Fix any $\bm{C}_{\mathrm{s}}\in(0,\infty)^{3L}$. For each $\ell\in[L]$, define $\alpha^{(\ell)}$ by Eq.~\eqref{eq:DefAlpha} and $\Gamma^{(\ell)}$ by
    \begin{align*}
        \Gamma^{(\ell)}
        =\sum_{\star\in\{QK,V,M\}}(\gamma^{\star,(\ell)})^{\frac{2p}{3p+2}}(C_{\mathrm{s}}^{\star,(\ell)})^{\frac{2}{3p+2}}
        ,
    \end{align*}
    where $\gamma^{\star,(\ell)}$ is defined in Theorem~\ref{thm:EntropyTfScalarGeneralP}. Then, for any $\epsilon>0$, we have
    \begin{align*}
        &\log\mathcal{N}_{\infty}(\mathcal{F}_{\mathrm{out}}^{(\bm{p},\bm{C}_{\mathrm{s}})},|\cdot|,\epsilon;\{X_{i}\}_{i\in[n]})\\
        &\lesssim\left(\frac{L_{\phi}C_{2}^{\mathrm{out}}}{\epsilon}\right)^{\frac{2p}{p+2}}\left(\sum_{\ell=1}^{L}(\alpha^{(\ell)})^{\frac{2p}{3p+2}}\Gamma^{(\ell)}\right)^{\frac{3p+2}{p+2}}N^{1+\frac{p}{p+2}}\log(nT)+\left(\frac{C_{2}^{\mathrm{out}}}{\epsilon}\right)^{2}\log(n)
        .
    \end{align*}
\end{theorem}
\begin{proof}
    When $p^{\star,(\ell)}=p$ holds, we can rewrite
    \begin{align*}
        &\Lambda^{\star,(\ell)}
        =((C_{\mathrm{s}}^{\star,(\ell)})^{2}B_{n,(2\to\infty)}^{2\delta_{\ell=1}p}N^{p})^{\frac{1}{p+2}}
        \quad(\star\in\{QK,V\})
        ,\qquad
        \Lambda^{M,(\ell)}
        =((C_{\mathrm{s}}^{M,(\ell)})^{2}L_{\phi}^{2p}N^{p})^{\frac{1}{p+2}}
        ,\\
        &\nu^{\star,(\ell)}
        =\frac{2p}{p+2}
        =\nu
        \quad(\star\in\{QK,V,M\})
        .
    \end{align*}
    For a fixed $\epsilon^{\mathrm{out}}<\epsilon$, we first consider the following optimization problem:
    \begin{align*}
        &\min_{\{\epsilon^{\star,(\ell)}\}_{\star\in\{QK,V,M\},\ell\in[L]}}\sum_{\ell=1}^{L}\sum_{\star\in\{QK,V,M\}}\Lambda^{\star,(\ell)}(\epsilon^{\star,(\ell)})^{-\nu}\\
        &\textnormal{subject to}\quad
        \sum_{\ell=1}^{L}\sum_{\star\in\{QK,V,M\}}\beta^{\star,(\ell)}\epsilon^{\star,(\ell)}
        =\epsilon-\epsilon^{\mathrm{out}}
        .
    \end{align*}
    If we take $p=0$, we have $\nu=0$. Thus, the first term on the right-hand side of Eq.~\eqref{eq:EpsilonObjectiveFunction} reduces to $\sum_{\ell=1}^{L}\sum_{\star\in\{QK,V,M\}}\Lambda^{\star,(\ell)}N\log(nT)$, which is independent of the allocation of $\{\epsilon^{\star,(\ell)}\}_{\ell\in[L],\star\in\{QK,V,M\}}$. Thus, setting $\epsilon^{\mathrm{out}}=\epsilon/2$ and choosing any $\{\epsilon^{\star,(\ell)}\}_{\ell\in[L],\star\in\{QK,V,M\}}$ such that Eq.~\eqref{eq:CoveringRadius} is satisfied, we obtain the desired bounds.

    Now, we consider the case $p>0$. By Lemma~\ref{lem:OptimizationPolynomial}, the optimal solution is given by
    \[
        \epsilon^{\star,(\ell)}
        =\frac{(\epsilon-\epsilon^{\mathrm{out}})(\Lambda^{\star,(\ell)})^{\frac{1}{\nu+1}}(\beta^{\star,(\ell)})^{-\frac{1}{\nu+1}}}{\sum_{k=1}^{L}\sum_{\star\in\{QK,V,M\}}(\Lambda^{\star,(k)})^{\frac{1}{\nu+1}}(\beta^{\star,(k)})^{\frac{\nu}{\nu+1}}}
        \qquad(\star\in\{QK,V,M\},\ell\in[L])
    \]
    and the corresponding minimum value is
    \begin{align*}
        &\frac{1}{(\epsilon-\epsilon^{\mathrm{out}})^{\nu}}\left(\sum_{\ell=1}^{L}\sum_{\star\in\{QK,V,M\}}(\Lambda^{\star,(\ell)})^{\frac{1}{\nu+1}}(\beta^{\star,(\ell)})^{\frac{\nu}{\nu+1}}\right)^{\nu+1}
        .
    \end{align*}
    Noting that we can write
    \begin{align*}
        (\Lambda^{QK,(\ell)})^{\frac{1}{\nu+1}}(\beta^{QK,(\ell)})^{\frac{\nu}{\nu+1}}
        &=(2L_{\phi}C_{2}^{\mathrm{out}}C_{2}^{V,(\ell)}C_{2}^{M,(\ell)}\alpha^{(\ell)})^{\frac{2p}{3p+2}}(C_{\mathrm{s}}^{QK,(\ell)})^{\frac{2}{3p+2}}(B_{n,(2\to\infty)}^{\delta_{\ell=1}})^{\frac{6p}{3p+2}}N^{\frac{p}{3p+2}}
        ,\\
        (\Lambda^{V,(\ell)})^{\frac{1}{\nu+1}}(\beta^{V,(\ell)})^{\frac{\nu}{\nu+1}}
        &=(L_{\phi}C_{2}^{\mathrm{out}}C_{2}^{M,(\ell)}\alpha^{(\ell)})^{\frac{2p}{3p+2}}(C_{\mathrm{s}}^{V,(\ell)})^{\frac{2}{3p+2}}(B_{n,(2\to\infty)}^{\delta_{\ell=1}})^{\frac{2p}{3p+2}}N^{\frac{p}{3p+2}}
        ,\\
        (\Lambda^{M,(\ell)})^{\frac{1}{\nu+1}}(\beta^{M,(\ell)})^{\frac{\nu}{\nu+1}}
        &=(L_{\phi}C_{2}^{\mathrm{out}}\alpha^{(\ell)})^{\frac{2p}{3p+2}}(C_{\mathrm{s}}^{M,(\ell)})^{\frac{2}{3p+2}}N^{\frac{p}{3p+2}}
        ,
    \end{align*}
    we have
    \begin{align*}
        &\frac{1}{(\epsilon-\epsilon^{\mathrm{out}})^{\nu}}\left(\sum_{\ell=1}^{L}\sum_{\star\in\{QK,V,M\}}(\Lambda^{\star,(\ell)})^{\frac{1}{\nu+1}}(\beta^{\star,(\ell)})^{\frac{\nu}{\nu+1}}\right)^{\nu+1}\\
        &=\left(\frac{L_{\phi}C_{2}^{\mathrm{out}}}{\epsilon-\epsilon^{\mathrm{out}}}\right)^{\frac{2p}{p+2}}\left(\sum_{\ell=1}^{L}(\alpha^{(\ell)})^{\frac{2p}{3p+2}}\Gamma^{(\ell)}\right)^{\frac{3p+2}{p+2}}N^{\frac{p}{p+2}}
        .
    \end{align*}
    Substituting $\epsilon^{\mathrm{out}}=\epsilon/2$ and combining the resulting bound with Eq.~\eqref{eq:EpsilonObjectiveFunction} completes the proof.
\end{proof}

\subsection{Generalization gap bounds for Transformers}\label{subsec:GenGapBounds}

\begin{theorem}\label{thm:GenGapGeneralP}
    Suppose $\max_{i\in[n]}\|X_{i}\|_{2\to\infty}\le B_{n,(2\to\infty)}$ holds. Fix any $\bm{p}\in[0,2]^{3L}$ and $\bm{C}_{\mathrm{s}}\in(0,\infty)^{3L}$. For each $\ell\in[L]$ and $\star\in\{QK,V,M\}$, define $\Psi^{\star,(\ell)}$ by
    \[
        \Psi^{\star,(\ell)}
        =(C_{\mathrm{s}}^{\star,(\ell)})^{\frac{1}{p^{\star,(\ell)}+2}}(L_{\phi}\gamma^{\star,(\ell)}\alpha^{(\ell)})^{\frac{p^{\star,(\ell)}}{p^{\star,(\ell)}+2}}
        ,
    \]
    where $\alpha^{(\ell)}$ is defined in Eq.~\eqref{eq:DefAlpha} and $\gamma^{\star,(\ell)}$ is defined in Theorem~\ref{thm:EntropyTfScalarGeneralP}. Suppose $n\ge3$ holds. Then, for any $\delta\in(0,1)$, with probability at least $1-\delta$, it holds simultaneously for all $f_{\mathrm{out}}\in\mathcal{F}_{\mathrm{out}}^{(\bm{p},\bm{C}_{\mathrm{s}})}$ that
    \begin{align*}
        \mathrm{GAP}(f_{\mathrm{out}})
        &\lesssim L_{\mathcal{L}}C_{2}^{\mathrm{out}}\sqrt{\frac{\log(nT)}{n}}\left(\sum_{\ell=1}^{L}\sum_{\star\in\{QK,V,M\}}\Psi^{\star,(\ell)}L^{\frac{p^{\star,(\ell)}}{p^{\star,(\ell)}+2}}N^{\frac{p^{\star,(\ell)}+1}{p^{\star,(\ell)}+2}}\right)\\
        &\quad+L_{\mathcal{L}}C_{2}^{\mathrm{out}}\frac{(\log(n))^{\frac{3}{2}}}{\sqrt{n}}+B_{\mathcal{L}}\sqrt{\frac{\log(1/\delta)}{n}}
        .
    \end{align*}
\end{theorem}
\begin{proof}
    For each $\ell\in[L]$ and $\star\in\{QK,V,M\}$, define $\nu^{\star,(\ell)}=2p^{\star,(\ell)}/(p^{\star,(\ell)}+2)\in[0,2)$. By Theorem~\ref{thm:EntropyTfScalarGeneralP}, for every $\epsilon\in(0,C_{2}^{\mathrm{out}}]$, we have
    \[
        \log\mathcal{N}_{\infty}(\mathcal{F}_{\mathrm{out}}^{(\bm{p},\bm{C}_{\mathrm{s}})},|\cdot|,\epsilon;\{X_{i}\}_{i\in[n]})
        \lesssim\sum_{\ell=1}^{L}\sum_{\star\in\{QK,V,M\}}c^{\star,(\ell)}\epsilon^{-\nu^{\star,(\ell)}}+c^{\mathrm{out}}\epsilon^{-2}
        ,
    \]
    where $c^{\star,(\ell)}$ and $c^{\mathrm{out}}$ are defined by
    \begin{align*}
        &c^{\star,(\ell)}
        =(C_{\mathrm{s}}^{\star,(\ell)})^{\frac{2}{p^{\star,(\ell)}+2}}(L_{\phi}C_{2}^{\mathrm{out}}\gamma^{\star,(\ell)}\alpha^{(\ell)}L)^{\nu^{\star,(\ell)}}N^{1+\frac{p^{\star,(\ell)}}{p^{\star,(\ell)}+2}}\log(nT)
        ,\\
        &c^{\mathrm{out}}
        =(C_{2}^{\mathrm{out}})^{2}\log(n)
        .
    \end{align*}
    Next, note that every $f\in\mathcal{F}_{\mathrm{out}}^{(\bm{p},\bm{C}_{\mathrm{s}})}$ satisfies $|f(X)|\le C_{2}^{\mathrm{out}}$ for all $X\in\mathbb{R}^{T\times N}$, because $\|[f_{\mathrm{tf}}^{(L)}(X)]_{[\mathrm{CLS}]}\|_{2}\le1$ holds by construction. Therefore, by Lemma~\ref{lem:RademacherPowerLawSum}, we have
    \begin{align*}
        &\hat{\mathfrak{R}}_{n}(\mathcal{F}_{\mathrm{out}}^{(\bm{p},\bm{C}_{\mathrm{s}})};\{X_{i}\}_{i\in[n]})\\
        &\lesssim\frac{1}{\sqrt{n}}\left(\sum_{\ell=1}^{L}\sum_{\star\in\{QK,V,M\}}(C_{2}^{\mathrm{out}})^{1-\nu^{\star,(\ell)}/2}\sqrt{c^{\star,(\ell)}}+\sqrt{c^{\mathrm{out}}}\left[1+\log\left(1+\frac{C_{2}^{\mathrm{out}}\sqrt{n}}{\sqrt{c^{\mathrm{out}}}}\right)\right]\right)\\
        &\lesssim C_{2}^{\mathrm{out}}\sqrt{\frac{\log(nT)}{n}}\left(\sum_{\ell=1}^{L}\sum_{\star\in\{QK,V,M\}}\Psi^{\star,(\ell)}L^{\frac{p^{\star,(\ell)}}{p^{\star,(\ell)}+2}}N^{\frac{p^{\star,(\ell)}+1}{p^{\star,(\ell)}+2}}\right)+C_{2}^{\mathrm{out}}\frac{(\log(n))^{\frac{3}{2}}}{\sqrt{n}}
        ,
    \end{align*}
    where the last inequality follows from $1+\log(1+\sqrt{n/\log(n)})\lesssim\log(n)$. Finally, applying Lemma~\ref{lem:RademacherGeneralization} completes the proof.
\end{proof}

\begin{theorem}\label{thm:GenGapCommonP}
    Suppose $\max_{i\in[n]}\|X_{i}\|_{2\to\infty}\le B_{n,(2\to\infty)}$ holds. For some $p\in[0,2]$, suppose $\bm{p}=(p,\dots,p)$ holds. Fix any $\bm{C}_{\mathrm{s}}\in(0,\infty)^{3L}$. Define
    \[
        \Xi^{(p)}
        =\left(\sum_{\ell=1}^{L}(\alpha^{(\ell)})^{\frac{2p}{3p+2}}\Gamma^{(\ell)}\right)^{\frac{3p+2}{2(p+2)}}N^{\frac{p+1}{p+2}}
        ,
    \]
    where $\alpha^{(\ell)}$ is defined in Eq.~\eqref{eq:DefAlpha} and $\Gamma^{(\ell)}$ is defined in Theorem~\ref{thm:EntropyTfScalarCommonP}. Suppose $n\ge3$ holds. Then, for any $\delta\in(0,1)$, with probability at least $1-\delta$, it holds simultaneously for all $f_{\mathrm{out}}\in\mathcal{F}_{\mathrm{out}}^{(\bm{p},\bm{C}_{\mathrm{s}})}$ that
    \[
        \mathrm{GAP}(f_{\mathrm{out}})
        \lesssim\frac{L_{\mathcal{L}}C_{2}^{\mathrm{out}}}{\sqrt{n}}[L_{\phi}^{\frac{p}{p+2}}\Xi^{(p)}\sqrt{\log(nT)}+(\log(n))^{\frac{3}{2}}]+B_{\mathcal{L}}\sqrt{\frac{\log(1/\delta)}{n}}
        .
    \]
\end{theorem}
\begin{proof}
    Define $\nu=2p/(p+2)\in[0,2)$. By Theorem~\ref{thm:EntropyTfScalarCommonP}, it holds for every $\epsilon\in(0,C_{2}^{\mathrm{out}}]$ that
    \[
        \log\mathcal{N}_{\infty}(\mathcal{F}_{\mathrm{out}}^{(\bm{p},\bm{C}_{\mathrm{s}})},|\cdot|,\epsilon;\{X_{i}\}_{i\in[n]})
        \lesssim c_{1}\epsilon^{-\nu}+c_{2}\epsilon^{-2}
        ,
    \]
    where
    \[
        c_{1}
        =(L_{\phi}C_{2}^{\mathrm{out}})^{\nu}(\Xi^{(p)})^{2}\log(nT)
        ,\qquad
        c_{2}
        =(C_{2}^{\mathrm{out}})^{2}\log(n)
        .
    \]
    Next, note that every $f\in\mathcal{F}_{\mathrm{out}}^{(\bm{p},\bm{C}_{\mathrm{s}})}$ satisfies $|f(X)|\le C_{2}^{\mathrm{out}}$ for all $X\in\mathbb{R}^{T\times N}$, because $\|[f_{\mathrm{tf}}^{(L)}(X)]_{[\mathrm{CLS}]}\|_{2}\le1$ holds by construction. Hence we may apply Lemma~\ref{lem:RademacherPowerLawSum} as
    \begin{align*}
        &\hat{\mathfrak{R}}_{n}(\mathcal{F}_{\mathrm{out}}^{(\bm{p},\bm{C}_{\mathrm{s}})};\{X_{i}\}_{i\in[n]})\\
        &\lesssim(C_{2}^{\mathrm{out}})^{1-\nu/2}\sqrt{\frac{c_{1}}{n}}+\frac{\sqrt{c_{2}}}{\sqrt{n}}\left[1+\log\left(1+\frac{C_{2}^{\mathrm{out}}\sqrt{n}}{\sqrt{c_{2}}}\right)\right]\\
        &\lesssim\frac{C_{2}^{\mathrm{out}}}{\sqrt{n}}\left(L_{\phi}^{\frac{p}{p+2}}\Xi^{(p)}\sqrt{\log(nT)}+(\log(n))^{\frac{3}{2}}\right)
        ,
    \end{align*}
    where the last inequality follows from $1+\log(1+\sqrt{n/\log(n)})\lesssim\log(n)$. Finally, applying Lemma~\ref{lem:RademacherGeneralization} completes the proof.
\end{proof}

\subsection{Obtaining the post hoc generalization bounds}\label{subsec:GenGapPostHoc}

In this section, we consider the scalar-output Transformer class $\mathcal{F}_{\mathrm{out}}$ in Eq.~\eqref{eq:TfOutClass}. In what follows, we derive a result in which the high-probability generalization gap bounds are uniform over the admissible choices of Schatten indices $p^{\star,(\ell)}$, so the indices may be selected after the trained weights have been observed. For that purpose, we partition the possible values of the realized Schatten quantity into dyadic shells. Define
\[
    \mathcal{J}=\mathbb{Z}\cup\{\bot\}
    ,\qquad
    R_{\bot}=0,\qquad
    R_{j}=2^{j}\quad(j\in\mathbb{Z})
    .
\]
Here, the symbol $\bot$ is an auxiliary index that does not belong to $\mathbb{Z}$, and it is used to represent the zero-radius shell. Thus, $j=\bot$ corresponds to $\|W\|_{\mathrm{s},p}^{p}=0$, whereas $j\in\mathbb{Z}$ corresponds to the positive dyadic shell $2^{j-1}<\|W\|_{\mathrm{s},p}^{p}\le2^{j}$. Set
\[
    Z_{\mathcal{J}}
    =1+\sum_{j\in\mathbb{Z}}\frac{1}{(1+|j|)^{2}}
    =\frac{\pi^{2}}{3}
    ,
    \qquad
    \omega_{\bot}
    =Z_{\mathcal{J}}^{-1}
    ,
    \qquad
    \omega_{j}
    =\frac{Z_{\mathcal{J}}^{-1}}{(1+|j|)^{2}}
    \quad(j\in\mathbb{Z})
\]
so that $\sum_{j\in\mathcal{J}}\omega_{j}=1$ holds. For $p\in[0,2]$ and $W\in\mathbb{R}^{N\times N}$, define
\[
    \kappa_{p}(W)
    =
    \begin{cases}
        \bot,
        &\|W\|_{\mathrm{s},p}^{p}=0
        ,\\
        \lceil\log_{2}\|W\|_{\mathrm{s},p}^{p}\rceil,
        &\|W\|_{\mathrm{s},p}^{p}>0
        .
    \end{cases}
\]
Thus, if $\|W\|_{\mathrm{s},p}^{p}$ is nonzero, it belongs to the $\kappa_{p}(W)$-th dyadic shell as
\[
    2^{\kappa_{p}(W)-1}
    <\|W\|_{\mathrm{s},p}^{p}
    \le2^{\kappa_{p}(W)}
    .
\]

For an integer $m\ge1$, define the grid
\[
    \mathcal{P}_{m}
    =\left\{0,\frac{1}{m},\frac{2}{m},\dots,\frac{2m-1}{m},2\right\}^{3L}.
\]
For $\bm{p}\in\mathcal{P}_{m}$, define the post hoc logarithmic penalty
\begin{equation}\label{eq:PosthocPenalty}
    \Omega_{\bm{p}}(W^{(1:L)})
    =3L\log(2m+1)+\sum_{\ell=1}^{L}\sum_{\star\in\{QK,V,M\}}\log(1/\omega_{\kappa_{p^{\star,(\ell)}}(W^{\star,(\ell)})})
    .
\end{equation}

For $\bm{p}\in[0,2]^{3L}$ and weights $W^{(1:L)}$, define
\begin{equation}\label{eq:PosthocComplexity}
    \mathfrak{B}_{\bm{p}}(W^{(1:L)})
    =\sum_{\ell=1}^{L}\sum_{\star\in\{QK,V,M\}}(\|W^{\star,(\ell)}\|_{\mathrm{s},p^{\star,(\ell)}}^{p^{\star,(\ell)}})^{\frac{1}{p^{\star,(\ell)}+2}}(L_{\phi}\gamma^{\star,(\ell)}\alpha^{(\ell)}L)^{\frac{p^{\star,(\ell)}}{p^{\star,(\ell)}+2}}N^{\frac{p^{\star,(\ell)}+1}{p^{\star,(\ell)}+2}}
    .
\end{equation}
Here $\alpha^{(\ell)}$ is defined in Eq.~\eqref{eq:DefAlpha}, and
$\gamma^{\star,(\ell)}$ is the quantity defined in Theorem~\ref{thm:EntropyTfScalarGeneralP}. When $p^{\star,(\ell)}=0$, we interpret $\|W^{\star,(\ell)}\|_{\mathrm{s},0}^{0}$ as $\operatorname{rank}(W^{\star,(\ell)})$.

\begin{lemma}[Post hoc selection of the Schatten indices]\label{lem:GenGapPostHocGrid}
    Suppose $\max_{i\in[n]}\|X_{i}\|_{2\to\infty}\le B_{n,(2\to\infty)}$ and $n\ge3$ hold. Fix an integer $m\ge1$. Then, for any $\delta\in(0,1)$, with probability at least $1-\delta$, every $f_{\mathrm{out}}(\ \cdot\ ;W^{(1:L)},w)\in\mathcal{F}_{\mathrm{out}}$ satisfies
    \begin{align*}
        \mathrm{GAP}(f_{\mathrm{out}})
        &\lesssim\inf_{\bm{p}\in\mathcal{P}_{m}}\left(L_{\mathcal{L}}C_{2}^{\mathrm{out}}\sqrt{\frac{\log(nT)}{n}}\mathfrak{B}_{\bm{p}}(W^{(1:L)})+B_{\mathcal{L}}\sqrt{\frac{\log(1/\delta)+\Omega_{\bm{p}}(W^{(1:L)})}{n}}\right)\\
        &\quad+L_{\mathcal{L}}C_{2}^{\mathrm{out}}\frac{(\log(n))^{\frac{3}{2}}}{\sqrt{n}}
        .
    \end{align*}
\end{lemma}
\begin{proof}
    Write the index sets as $\mathcal{I}=\{QK,V,M\}\times[L]$. For $a=(\star,\ell)\in\mathcal{I}$, we write $W^{a}=W^{\star,(\ell)}$ and $p^{a}=p^{\star,(\ell)}$.
    For $\bm{p}\in\mathcal{P}_{m}$ and $\bm{j}=(j_{a})_{a\in\mathcal{I}}\in\mathcal{J}^{3L}$, define
    \[
        \mathcal{F}_{\bm{p},\bm{j}}
        =\{f_{\mathrm{out}}(\ \cdot\ ;W^{(1:L)},w)\in\mathcal{F}_{\mathrm{out}}\mid\|W^{a}\|_{\mathrm{s},p^{a}}^{p^{a}}\le R_{j_{a}}\ (a\in\mathcal{I})\}.
    \]
    When $j_{a}=\bot$, this condition is interpreted as $\|W^{a}\|_{\mathrm{s},p^{a}}^{p^{a}}=0$. For each $a=(\star,\ell)\in\mathcal{I}$, define
    \[
        M_{a}
        =N(1\vee(C_{2}^{\star,(\ell)})^{2})
        ,\qquad
        \bar{R}_{j_{a}}^{a}
        =
        \begin{cases}
            0,&j_{a}=\bot
            ,\\
            R_{j_{a}}\wedge M_{a},&j_{a}\in\mathbb{Z}
            .
        \end{cases}
    \]
    Since every element of $\mathcal{F}_{\mathrm{out}}$ satisfies $\|W^{a}\|_2\le C_{2}^{\star,(\ell)}$, we have $\|W^{a}\|_{\mathrm{s},p^{a}}^{p^{a}}\le M_{a}$ for every $p^{a}\in[0,2]$. Therefore, under the spectral norm constraints already included in $\mathcal{F}_{\mathrm{out}}$, the class $\mathcal{F}_{\bm{p},\bm{j}}$ is equivalently described by the constraints
    \[
        \|W^{a}\|_{\mathrm{s},p^{a}}^{p^{a}}\le\bar{R}_{j_{a}}^{a}
        \qquad(a\in\mathcal{I})
        .
    \]

    The localized fixed-radius version of Theorem~\ref{thm:GenGapGeneralP}, applied to $\mathcal{F}_{\bm{p},\bm{j}}$ with Schatten-quantity radii $\bar{R}_{j_{a}}^{a}$, gives an event $E_{\bm{p},\bm{j}}$ such that
    \[
        \mathrm{P}(E_{\bm{p},\bm{j}})
        \ge1-\delta_{\bm{p},\bm{j}}
        ,\qquad
        \delta_{\bm{p},\bm{j}}
        =\frac{\delta}{(2m+1)^{3L}}\prod_{a\in\mathcal{I}}\omega_{j_{a}}
        ,
    \]
    and, on $E_{\bm{p},\bm{j}}$, the following bounds hold simultaneously for all $f_{\mathrm{out}}(\ \cdot\ ;W^{(1:L)},w)\in\mathcal{F}_{\bm{p},\bm{j}}$:
    \begin{align*}
        \mathrm{GAP}(f_{\mathrm{out}})
        &\lesssim L_{\mathcal{L}}C_{2}^{\mathrm{out}}\sqrt{\frac{\log(nT)}{n}}\sum_{a=(\star,\ell)\in\mathcal{I}}(\bar{R}_{j_{a}}^{a})^{\frac{1}{p^{a}+2}}(L_{\phi}\gamma^{\star,(\ell)}\alpha^{(\ell)}L)^{\frac{p^{a}}{p^{a}+2}}N^{\frac{p^{a}+1}{p^{a}+2}}\\
        &\quad+L_{\mathcal{L}}C_{2}^{\mathrm{out}}\frac{(\log(n))^{\frac{3}{2}}}{\sqrt{n}}+B_{\mathcal{L}}\sqrt{\frac{\log(1/\delta_{\bm{p},\bm{j}})}{n}}
        .
    \end{align*}
    The logarithmic factors suppressed in this application are uniform over $\bm{p}\in\mathcal{P}_{m}$ and $\bm{j}\in\mathcal{J}^{3L}$, because the fixed-radius entropy bounds are used only with the truncated radii $\bar{R}_{j_{a}}^{a}\le M_{a}$. Thus any radius-dependent logarithmic factor is absorbed into the notation $\lesssim$. If $\bar{R}_{j_{a}}^{a}=0$ for some $a$, the same conclusion follows by fixing the corresponding weight matrix to zero and applying the same proof with that coordinate omitted. Since $\bar{R}_{j_{a}}^{a}\le R_{j_{a}}$ holds, the preceding display implies the coarser bound
    \begin{align}
        \notag
        \mathrm{GAP}(f_{\mathrm{out}})
        &\lesssim L_{\mathcal{L}}C_{2}^{\mathrm{out}}\sqrt{\frac{\log(nT)}{n}}\sum_{a=(\star,\ell)\in\mathcal{I}}R_{j_{a}}^{\frac{1}{p^{a}+2}}(L_{\phi}\gamma^{\star,(\ell)}\alpha^{(\ell)}L)^{\frac{p^{a}}{p^{a}+2}}N^{\frac{p^{a}+1}{p^{a}+2}}\\
        \label{eq:PostHocShellBound}
        &\quad+L_{\mathcal{L}}C_{2}^{\mathrm{out}}\frac{(\log(n))^{\frac{3}{2}}}{\sqrt{n}}+B_{\mathcal{L}}\sqrt{\frac{\log(1/\delta_{\bm{p},\bm{j}})}{n}}
        .
    \end{align}

    By the union bound, the event $E=\bigcap_{\bm{p}\in\mathcal{P}_{m}}\bigcap_{\bm{j}\in\mathcal{J}^{3L}}E_{\bm{p},\bm{j}}$ satisfies
    \begin{align*}
        &\mathrm{P}(E)
        =1-\mathrm{P}(E^{C})
        =1-\mathrm{P}\left(\bigcup_{\bm{p}\in\mathcal{P}_{m}}\bigcup_{\bm{j}\in\mathcal{J}^{3L}}E_{\bm{p},\bm{j}}^{C}\right)
        \ge1-\sum_{\bm{p}\in\mathcal{P}_{m}}\sum_{\bm{j}\in\mathcal{J}^{3L}}\mathrm{P}(E_{\bm{p},\bm{j}}^{C})\\
        &\ge1-\sum_{\bm{p}\in\mathcal{P}_{m}}\sum_{\bm{j}\in\mathcal{J}^{3L}}\delta_{\bm{p},\bm{j}}
        =1-\delta
        ,
    \end{align*}
    where the last equality follows from
    \[
        \sum_{\bm{p}\in\mathcal{P}_{m}}\sum_{\bm{j}\in\mathcal{J}^{3L}}\delta_{\bm{p},\bm{j}}
        =\delta\sum_{\bm{j}\in\mathcal{J}^{3L}}\prod_{a\in\mathcal{I}}\omega_{j_{a}}
        =\delta\left(\sum_{j\in\mathcal{J}}\omega_{j}\right)^{3L}
        =\delta
        .
    \]

    We now work on the event $E$. Fix an arbitrary $f_{\mathrm{out}}(\ \cdot\ ;W^{(1:L)},w)\in\mathcal{F}_{\mathrm{out}}$ and an arbitrary $\bm{p}\in\mathcal{P}_{m}$. For each $a=(\star,\ell)\in\mathcal{I}$, set $j_{a}=\kappa_{p^{a}}(W^{a})$. Then we have $f_{\mathrm{out}}\in\mathcal{F}_{\bm{p},\bm{j}}$, and Eq.~\eqref{eq:PostHocShellBound} is applicable. Moreover, by the definition of $\kappa_{p^{a}}$, we have
    \[
        R_{j_{a}}
        \le2\|W^{a}\|_{\mathrm{s},p^{a}}^{p^{a}}
    \]
    whenever $\|W^{a}\|_{\mathrm{s},p^{a}}^{p^{a}}>0$, while both sides are zero when $\|W^{a}\|_{\mathrm{s},p^{a}}^{p^{a}}=0$ and $j_{a}=\bot$. Since $p^{a}\in[0,2]$ holds, this implies
    \[
        R_{j_{a}}^{\frac{1}{p^{a}+2}}
        \le2^{\frac{1}{p^{a}+2}}(\|W^{a}\|_{\mathrm{s},p^{a}}^{p^{a}})^{\frac{1}{p^{a}+2}}
        \le\sqrt{2}(\|W^{a}\|_{\mathrm{s},p^{a}}^{p^{a}})^{\frac{1}{p^{a}+2}}
        .
    \]
    The factor $\sqrt{2}$ is absorbed into the universal constant. We also have
    \[
        \log(1/\delta_{\bm{p},\bm{j}})
        =\log(1/\delta)+3L\log(2m+1)+\sum_{a\in\mathcal{I}}\log\left(\frac{1}{\omega_{j_{a}}}\right)
        =\log(1/\delta)+\Omega_{\bm{p}}(W^{(1:L)}).
    \]
    Substituting these two estimates into Eq.~\eqref{eq:PostHocShellBound} gives
    \begin{align*}
        \mathrm{GAP}(f_{\mathrm{out}})
        &\lesssim L_{\mathcal{L}}C_{2}^{\mathrm{out}}\sqrt{\frac{\log(nT)}{n}}\mathfrak{B}_{\bm{p}}(W^{(1:L)})+L_{\mathcal{L}}C_{2}^{\mathrm{out}}\frac{(\log(n))^{\frac{3}{2}}}{\sqrt{n}}\\
        &\quad+B_{\mathcal{L}}\sqrt{\frac{\log(1/\delta)+\Omega_{\bm{p}}(W^{(1:L)})}{n}}.
    \end{align*}
    This holds for every $\bm{p}\in\mathcal{P}_{m}$ on the same event $E$. Taking the infimum over $\bm{p}\in\mathcal{P}_{m}$ proves the claim.
\end{proof}

\begin{theorem}[Post hoc generalization bounds]\label{thm:GenGapPostHoc}
    Suppose $\max_{i\in[n]}\|X_{i}\|_{2\to\infty}\le B_{n,(2\to\infty)}$ and $n\ge3$ hold. Suppose further that $\gamma^{\star,(\ell)}\alpha^{(\ell)}>0$ holds for all $\star\in\{QK,V,M\}$ and $\ell\in[L]$, and $L_{\phi}>0$ holds. For $\bm{p}\in[0,2]^{3L}$ and weights $W^{(1:L)}$, define $\Omega_{\bm{p}}(W^{(1:L)})$ and $\mathfrak{B}_{\bm{p}}(W^{(1:L)})$ by Eq.~\eqref{eq:PosthocPenalty} and Eq.~\eqref{eq:PosthocComplexity}, respectively. Fix an integer $m\ge1$, and define the coordinatewise upward grid projection $\Pi_{m}:[0,2]^{3L}\to\mathcal{P}_{m}$ as follows. For $t\in[0,2]$, define $\pi_{m}(t)$ by $\pi_{m}(t)=\lceil mt\rceil/m$ and set
    \[
        \Pi_{m}(\bm{p})
        =(\pi_{m}(p^{\star,(\ell)}))_{\star\in\{QK,V,M\},\ell\in[L]}
        .
    \]
    For weights $W^{(1:L)}$, define
    \[
        \chi(W^{(1:L)})
        =\max_{\ell\in[L],\star\in\{QK,V,M\}:W^{\star,(\ell)}\ne0}(|\log\|W^{\star,(\ell)}\|_{2}|+|\log(L_{\phi}\gamma^{\star,(\ell)}\alpha^{(\ell)})|)
        ,
    \]
    with the convention that the maximum over an empty set is zero. Then, for any $\delta\in(0,1)$, with probability at least $1-\delta$, every $f_{\mathrm{out}}(\ \cdot\ ;W^{(1:L)},w)\in\mathcal{F}_{\mathrm{out}}$ satisfies
    \begin{align*}
        \mathrm{GAP}(f_{\mathrm{out}})
        &\lesssim\inf_{\bm{p}\in[0,2]^{3L}}\Bigg[L_{\mathcal{L}}C_{2}^{\mathrm{out}}\sqrt{\frac{\log(nT)}{n}}\exp\left(\frac{\chi(W^{(1:L)})}{m}\right)L^{\frac{1}{2m}}N^{\frac{1}{4m}}\mathfrak{B}_{\bm{p}}(W^{(1:L)})\\
        &\hspace{30mm}+B_{\mathcal{L}}\sqrt{\frac{\log(1/\delta)+\Omega_{\Pi_{m}(\bm{p})}(W^{(1:L)})}{n}}\Bigg]+L_{\mathcal{L}}C_{2}^{\mathrm{out}}\frac{(\log(n))^{\frac{3}{2}}}{\sqrt{n}}
        .
    \end{align*}
\end{theorem}
\begin{proof}
    We first prove a deterministic rounding inequality. Fix $0\le p\le q\le2$ and a nonzero matrix $W\in\mathbb{R}^{N\times N}$. For $H>0$, define
    \[
        F_{p}(W;H)
        =(\|W\|_{\mathrm{s},p}^{p})^{\frac{1}{p+2}}H^{\frac{p}{p+2}}N^{\frac{p+1}{p+2}}
        .
    \]
    Set $\rho_{p}(W)=\|W\|_{\mathrm{s},p}^{p}/\|W\|_{2}^{p}$. Since the singular values normalized by $\|W\|_{2}$ belong to $[0,1]$ and the largest normalized singular value is one, we have
    \begin{equation}\label{eq:IneqSchattenSpectral}
        1\le\rho_{q}(W)\le\rho_{p}(W)
        \le N
        .
    \end{equation}
    Hence, it holds that
    \begin{align*}
        \log\frac{F_{q}(W;H)}{F_{p}(W;H)}
        &=\frac{q\log\|W\|_{2}+\log\rho_{q}(W)}{q+2}-\frac{p\log\|W\|_{2}+\log\rho_{p}(W)}{p+2}\\
        &\quad+\left(\frac{q}{q+2}-\frac{p}{p+2}\right)\log H+\left(\frac{q+1}{q+2}-\frac{p+1}{p+2}\right)\log N\\
        &\le(q-p)\left(\frac{|\log\|W\|_{2}|}{2}+\frac{|\log H|}{2}+\frac{\log N}{4}\right)
        .
    \end{align*}
    Therefore, we have
    \[
        F_{q}(W;H)
        \le\exp\left((q-p)\left(\frac{|\log\|W\|_{2}|}{2}+\frac{|\log H|}{2}+\frac{\log N}{4}\right)\right)F_{p}(W;H)
        .
    \]
    The same inequality is trivial when $W=0$, because both sides are zero.

    Now fix $\bm{p}\in[0,2]^{3L}$ and set $\tilde{\bm{p}}=\Pi_{m}(\bm{p})$. For every coordinate
    $(\star,\ell)$, we have
    \[
        \tilde{p}^{\star,(\ell)}-p^{\star,(\ell)}
        =\frac{\lceil mp^{\star,(\ell)}\rceil}{m}-p^{\star,(\ell)}
        \in\left[0,\frac{1}{m}\right]
        .
    \]
    Applying the preceding deterministic inequality with $H=L_{\phi}\gamma^{\star,(\ell)}\alpha^{(\ell)}L$ to each summand in $\mathfrak{B}_{\bm{p}}(W^{(1:L)})$ gives
    \begin{equation}\label{eq:GridRoundingContinuum}
        \mathfrak{B}_{\Pi_{m}(\bm{p})}(W^{(1:L)})
        \le\exp\left(\frac{\chi(W^{(1:L)})}{2m}\right)L^{\frac{1}{2m}}N^{\frac{1}{4m}}\mathfrak{B}_{\bm{p}}(W^{(1:L)})
        .
    \end{equation}

    We next apply Lemma~\ref{lem:GenGapPostHocGrid}. On an event of probability at least $1-\delta$, every $f_{\mathrm{out}}(\ \cdot\ ;W^{(1:L)},w)\in\mathcal{F}_{\mathrm{out}}$ satisfies
    \begin{align*}
        \mathrm{GAP}(f_{\mathrm{out}})
        &\lesssim\inf_{\bm{q}\in\mathcal{P}_{m}}\left(L_{\mathcal{L}}C_{2}^{\mathrm{out}}\sqrt{\frac{\log(nT)}{n}}\mathfrak{B}_{\bm{q}}(W^{(1:L)})+B_{\mathcal{L}}\sqrt{\frac{\log(1/\delta)+\Omega_{\bm{q}}(W^{(1:L)})}{n}}\right)\\
        &\quad+L_{\mathcal{L}}C_{2}^{\mathrm{out}}\frac{(\log(n))^{\frac{3}{2}}}{\sqrt{n}}
        .
    \end{align*}
    Since $\Pi_{m}(\bm{p})\in\mathcal{P}_{m}$ for every $\bm{p}\in[0,2]^{3L}$, we may upper bound the finite-grid infimum by evaluating it at $\bm{q}=\Pi_{m}(\bm{p})$. Using Eq.~\eqref{eq:GridRoundingContinuum}, we obtain, for every $\bm{p}\in[0,2]^{3L}$,
    \begin{align*}
        \mathrm{GAP}(f_{\mathrm{out}})
        &\lesssim L_{\mathcal{L}}C_{2}^{\mathrm{out}}\sqrt{\frac{\log(nT)}{n}}\exp\left(\frac{\chi(W^{(1:L)})}{m}\right)L^{\frac{1}{2m}}N^{\frac{1}{4m}}\mathfrak{B}_{\bm{p}}(W^{(1:L)})\\
        &\quad+B_{\mathcal{L}}\sqrt{\frac{\log(1/\delta)+\Omega_{\Pi_{m}(\bm{p})}(W^{(1:L)})}{n}}+L_{\mathcal{L}}C_{2}^{\mathrm{out}}\frac{(\log(n))^{\frac{3}{2}}}{\sqrt{n}}
        .
    \end{align*}
    Taking the infimum over $\bm{p}\in[0,2]^{3L}$ proves the claim.
\end{proof}

\begin{remark}[Derivation of Theorem~\ref{thm:GenGapPostHocSimplified}]\label{rem:GenGapPostHocSimplified}
    We derive Theorem~\ref{thm:GenGapPostHocSimplified} from Theorem~\ref{thm:GenGapPostHoc}. Choose $m=\lceil L+\log(N)\rceil$. Since $\alpha^{(\ell)}$ is a product of at most $L-\ell$ fixed spectral norm and Lipschitz factors, there exists a constant $C>0$ such that $L_{\phi}\gamma^{\star,(\ell)}\alpha^{(\ell)}\le C^{L}$ holds for all $\star$ and $\ell$. Hence, we have
    \[
        \mathfrak{B}_{\bm{p}}(W^{(1:L)})
        \lesssim\sum_{\ell=1}^{L}\sum_{\star\in\{QK,V,M\}}(\|W^{\star,(\ell)}\|_{\mathrm{s},p^{\star,(\ell)}}^{p^{\star,(\ell)}})^{\frac{1}{p^{\star,(\ell)}+2}}C^{\frac{Lp^{\star,(\ell)}}{p^{\star,(\ell)}+2}}L^{\frac{p^{\star,(\ell)}}{p^{\star,(\ell)}+2}}N^{\frac{p^{\star,(\ell)}+1}{p^{\star,(\ell)}+2}}
        .
    \]

    We now simplify the two additional terms, $\exp(\chi(W^{(1:L)})/m)L^{\frac{1}{2m}}N^{\frac{1}{4m}}$ and $\Omega_{\Pi_{m}(\bm{p})}(W^{(1:L)})$, introduced by the post hoc argument. First, we control the rounding factor $\exp(\chi(W^{(1:L)})/m)L^{\frac{1}{2m}}N^{\frac{1}{4m}}$. Suppose that the trained weights are not exponentially small, in the sense that there exists a universal constant $c_{0}>0$ that satisfies $\|W^{\star,(\ell)}\|_{2}\ge\exp[-c_{0}(L+\log(N))]$ for every nonzero $W^{\star,(\ell)}$. Under the simplified spectral norm assumptions, the quantity $\chi(W^{(1:L)})$ appearing in Theorem~\ref{thm:GenGapPostHoc} satisfies $\chi(W^{(1:L)})\lesssim L+\log(N)$. Consequently, our choice of $m$ gives
    \[
        \exp\left(\frac{\chi(W^{(1:L)})}{m}\right)L^{\frac{1}{2m}}N^{\frac{1}{4m}}
        \lesssim L^{\frac{1}{2m}}N^{\frac{1}{4m}}
        =\exp\left(\frac{\log(L)}{2m}+\frac{\log(N)}{4m}\right)
        \le\exp\left(\frac{1}{2}+\frac{1}{4}\right)
        \lesssim1
        .
    \]
    It remains to simplify the post hoc logarithmic penalty $\Omega_{\Pi_{m}(\bm{p})}(W^{(1:L)})$. Note that for $j\in\mathbb{Z}$, we have $\log(1/\omega_{j})=\log(Z_{\mathcal{J}})+2\log(1+|j|)$. For a nonzero matrix $W$ and $p\in[0,2]$, we have $\|W\|_{2}^{p}\le\|W\|_{\mathrm{s},p}^{p}\le N\|W\|_{2}^{p}$ (see Eq.~\eqref{eq:IneqSchattenSpectral}). Together with the upper spectral norm constraint and the lower bound on nonzero spectral norms, this implies
    \[
        |\kappa_{p}(W)|
        =|\lceil\log_{2}\|W\|_{\mathrm{s},p}^{p}\rceil|
        \lesssim L+\log(N)
        .
    \]
    Thus, if $W^{\star,(\ell)}\neq0$ and $\kappa_{p^{\star,(\ell)}}(W^{\star,(\ell)})\neq0$ hold, we have
    \[
        \log(1/\omega_{\kappa_{p^{\star,(\ell)}}(W^{\star,(\ell)})})
        \lesssim\log(|\kappa_{p^{\star,(\ell)}}(W^{\star,(\ell)})|)
        \lesssim\log(L+\log(N))
        .
    \]
    On the other hand, if either $\kappa_{p^{\star,(\ell)}}(W^{\star,(\ell)})=0$ or $W^{\star,(\ell)}=0$ holds, we have by definition that $\log(1/\omega_{\kappa_{p^{\star,(\ell)}}(W^{\star,(\ell)})})=\log(Z_{\mathcal{J}})$. Hence, uniformly over $p^{\star,(\ell)}\in[0,2]$ and $W^{\star,(\ell)}$, we have
    \[
        \log(1/\omega_{\kappa_{p^{\star,(\ell)}}(W^{\star,(\ell)})})
        \lesssim\log(L+\log(N))
        .
    \]
    Therefore, we obtain
    \[
        \Omega_{\Pi_{m}(\bm{p})}(W^{(1:L)})
        \lesssim L\log(L+\log(N))
    \]
    uniformly over $\bm{p}\in[0,2]^{3L}$.
\end{remark}

\section{Choice of $p^{\star,(\ell)}$ under polynomial singular value decay}\label{sec:PolynomialDecay}

In this section, we illustrate how the post hoc parameter $p^{\star,(\ell)}$ in Theorem~\ref{thm:GenGapPostHocSimplified} can be selected from the observed singular-value profile. For each $\star\in\{QK,V,M\},\ell\in[L]$ and $p^{\star,(\ell)}\in[0,2]$, define
\begin{align*}
    \tilde{\mathfrak{B}}^{\star,(\ell)}(p^{\star,(\ell)})
    &=(\|W^{\star,(\ell)}\|_{\mathrm{s},p^{\star,(\ell)}}^{p^{\star,(\ell)}})^{\frac{1}{p^{\star,(\ell)}+2}}C^{\frac{Lp^{\star,(\ell)}}{p^{\star,(\ell)}+2}}L^{\frac{p^{\star,(\ell)}}{p^{\star,(\ell)}+2}}N^{\frac{p^{\star,(\ell)}+1}{p^{\star,(\ell)}+2}}
    .
\end{align*}
Then, the first term on the right-hand side of Eq.~\eqref{eq:GenGapPostHocSimplified} is expressed as
\[
    \sqrt{\frac{\log(nT)}{n}}\sum_{\ell=1}^{L}\sum_{\star\in\{QK,V,M\}}\inf_{p^{\star,(\ell)}\in[0,2]}\tilde{\mathfrak{B}}^{\star,(\ell)}(p^{\star,(\ell)})
    .
\]
Thus, it suffices to choose $p^{\star,(\ell)}\in[0,2]$ to minimize $\tilde{\mathfrak{B}}^{\star,(\ell)}(p^{\star,(\ell)})$.

Fix $\star\in\{QK,V,M\}$ and $\ell\in[L]$. Suppose that the singular values of $W^{\star,(\ell)}$ satisfy the polynomial decay condition
\[
    \sigma_{i}(W^{\star,(\ell)})
    \asymp S^{\star,(\ell)}i^{-\tau^{\star,(\ell)}}
    \qquad(i\in[N])
    ,
\]
with dimension-independent $S^{\star,(\ell)}$ and $\tau^{\star,(\ell)}\ge1/2$. For notational simplicity, we write $p=p^{\star,(\ell)}$, $\tau=\tau^{\star,(\ell)}$, $S=S^{\star,(\ell)}$, and $A=C^{L}L$. By the singular value decay assumption, we have
\[
    \|W^{\star,(\ell)}\|_{\mathrm{s},p}^{p}
    =
    \begin{cases}
        \sum_{i=1}^{N}\sigma_{i}(W^{\star,(\ell)})^{p}
        \lesssim S^{p}\sum_{i=1}^{N}i^{-\tau p},
        &\text{if }p>0
        ,\\
        \operatorname{rank}(W^{\star,(\ell)})
        \le N,
        &\text{if }p=0
        .\\
    \end{cases}
\]
Next, note that for any $q>0$ and $N\ge2$, we have
\[
    \sum_{i=1}^{N}i^{-q}
    \le
    \begin{cases}
        N^{1-q}/(1-q),
        &\text{if }q<1
        ,\\
        \log(N)+1
        \lesssim \log(N),
        &\text{if }q=1
        ,\\
        q/(q-1),
        &\text{if }q>1
        .
    \end{cases}
\]
Therefore, we have
\begin{equation}\label{eq:PolyDecayUpper}
    \tilde{\mathfrak{B}}^{\star,(\ell)}(p)
    \lesssim
    \begin{cases}
        N
        &\text{if }p=0
        ,\\
        (1-\tau p)^{-\frac{1}{p+2}}A^{\frac{p}{p+2}}N^{1-\frac{\tau p}{p+2}},
        &\text{if }0<p<1/\tau
        ,\\
        A^{\frac{p}{p+2}}N^{\frac{p+1}{p+2}}(\log(N))^{\frac{1}{p+2}},
        &\text{if }p=1/\tau
        ,\\
        (\tau p/(\tau p-1))^{\frac{1}{p+2}}A^{\frac{p}{p+2}}N^{\frac{p+1}{p+2}},
        &\text{if }p>1/\tau
        .
    \end{cases}
    .
\end{equation}

To obtain an explicit polynomial-order description, suppose that $A\asymp N^{\kappa}$ holds for some $\kappa\ge0$. For example, this includes the case $L=O(1)$ with $\kappa=0$, or the case $C=1$, $L\asymp N^{\kappa}$. For simplicity, we assume $N\gg1$, and we only optimize in the polynomial order in $N$ and suppress constants and logarithmic factors.

If $0<p<1/\tau$, the leading polynomial order in $N$ of the right-hand side of Eq.~\eqref{eq:PolyDecayUpper} is determined by $1+(\kappa-\tau)p/(p+2)$. This value decreases in $p$ if $\kappa<\tau$ and increases in $p$ if $\kappa>\tau$. On the other hand, if $p>1/\tau$, the leading polynomial order in $N$ is determined by $(p+1+\kappa p)/(p+2)$, which increases in $p$. Therefore, if $\kappa<\tau$, the choice $p=1/\tau$ is optimal and we have
\[
    \tilde{\mathfrak{B}}^{\star,(\ell)}\left(\frac{1}{\tau}\right)
    \lesssim A^{\frac{1}{2\tau+1}}N^{\frac{\tau+1}{2\tau+1}}(\log(N))^{\frac{\tau}{2\tau+1}}
    ,
\]
whereas if $\kappa\ge\tau$, the optimal choice is $p=0$ and we have $\tilde{\mathfrak{B}}^{\star,(\ell)}(0)\lesssim N$.

This calculation gives the following interpretation. When the singular values decay rapidly, the optimal admissible Schatten index is small, and the resulting bounds behave closer to rank-based bounds. In contrast, when the singular values decay slowly, the optimal index moves toward larger values, and the bounds become closer to norm-based bounds. Therefore, Theorem~\ref{thm:GenGapPostHocSimplified} can exploit heterogeneous spectral behavior across layers and matrix types by selecting $p^{\star,(\ell)}$ separately after training.

\section{Empirical comparison and discussion}\label{sec:EmpiricalComparisonDiscussion}

In this section, we first discuss a practical limitation of existing norm-based bounds: their leading factors involve fixed norm radii whose dimension-independent scaling is difficult to justify for large learned weight matrices. Then, we compare the growth of the leading-factor proxies induced by our bounds and by existing norm-based bounds. For illustration, we use the BERT Miniatures checkpoints of \citet{turc2019well}. Since these checkpoints use multi-head attention and two-layer feedforward sublayers, we construct BERT-adapted proxies from the leading polynomial and spectral factors of the theoretical generalization gap bounds. The code used for the empirical analyses in Section~\ref{sec:experiments} and this appendix is available at  \url{https://github.com/manasakai/spectrum-adaptive-transformers}.

\subsection{Existing norm-based bounds}

We first recall the two existing norm-based bounds. The bounds of \citet{edelman2022inductive} assume mixed $(2,1)$-norm constraints and obtain a logarithmic dependence on the token length and hidden dimension. The bounds of \citet{trauger2024sequence} replace these assumptions with mixed $(1,1)$-norm constraints and remove the explicit dependence on the token length.

\begin{theorem}[\citet{edelman2022inductive}, Theorem~A.17]\label{thm:Edelman2022ThmA.17}
    Consider the parameter classes
    \begin{align*}
        \mathcal{W}^{QK,(\ell)}
        &=\{W\in\mathbb{R}^{N\times N}\mid\|W^{\top}\|_{2,1}\le C_{2,1}^{QK,(\ell)},\|W\|_{2}\le C_{2}^{QK,(\ell)}\}
        ,\\
        \mathcal{W}^{V,(\ell)}
        &=\{W\in\mathbb{R}^{N\times N}\mid\|W\|_{2,1}\le C_{2,1}^{V,(\ell)},\|W\|_{2}\le C_{2}^{V,(\ell)}\}
        ,\\
        \mathcal{W}^{M,(\ell)}
        &=\{W\in\mathbb{R}^{N\times N}\mid\|W\|_{2,1}\le C_{2,1}^{M,(\ell)},\|W\|_{2}\le C_{2}^{M,(\ell)}\}
        .
    \end{align*}
    Then, for any $\delta\in(0,1)$, with probability at least $1-\delta$, it holds simultaneously for all $f_{\mathrm{out}}\in\mathcal{F}_{\mathrm{out}}$ that
    \[
        \mathrm{GAP}(f_{\mathrm{out}})
        \lesssim L_{\mathcal{L}}C_{2}^{\mathrm{out}}\left(1+\sum_{\ell=1}^{L}(\alpha^{(\ell)})^{\frac{2}{3}}\xi^{(\ell)}\right)^{\frac{3}{2}}\sqrt{\frac{\log(NnT)}{n}}+B_{\mathcal{L}}\sqrt{\frac{\log(1/\delta)}{n}}
    \]
    with $\xi^{(\ell)}=(C_{2,1}^{M,(\ell)})^{\frac{2}{3}}+(2L_{\phi}C_{2}^{M,(\ell)}C_{2}^{V,(\ell)}C_{2,1}^{QK,(\ell)})^{\frac{2}{3}}+(L_{\phi}C_{2}^{M,(\ell)}C_{2,1}^{V,(\ell)})^{\frac{2}{3}}$.
\end{theorem}

\begin{theorem}[\citet{trauger2024sequence}, Corollary~4.2.1]
    Consider parameter classes of the form
    \[
        \mathcal{W}^{\star,(\ell)}
        =\{W\in\mathbb{R}^{N\times N}\mid\|W\|_{1,1}\le C_{1,1},\|W\|_{2}\le C_{2}^{\star}\}
        \quad(\star\in\{QK,V,M\},\ell\in[L])
        .
    \]
    Then, for any $\delta\in(0,1)$, with probability at least $1-\delta$, it holds simultaneously for all $f_{\mathrm{out}}\in\mathcal{F}_{\mathrm{out}}$ that
    \begin{align*}
        \mathrm{GAP}(f_{\mathrm{out}})
        &\lesssim L_{\mathcal{L}}C_{2}^{\mathrm{out}}C_{1,1}\left(1+(L_{\phi}C_{2}^{V})^{\frac{2}{3}}+\sum_{\ell=1}^{L}(\alpha^{(\ell)})^{\frac{2}{3}}\upsilon^{(\ell)}\right)^{\frac{3}{2}}\sqrt{\frac{\log(2N^{2}+1)}{n}}\\
        &\quad+B_{\mathcal{L}}\sqrt{\frac{\log(1/\delta)}{n}}
        ,
    \end{align*}
    where
    \[
        \upsilon^{(\ell)}
        =
        \begin{cases}
            (2L_{\phi}C_{2}^{M}C_{2}^{V}B_{n,(2\to\infty)})^{\frac{2}{3}}
            &\text{if }\ell=1
            \\
            1+(2L_{\phi}C_{2}^{M}C_{2}^{V})^{\frac{2}{3}}+(L_{\phi}C_{2}^{V})^{\frac{2}{3}}
            &\text{if }\ell>1
            .
        \end{cases}
    \]
\end{theorem}

The apparent dimension-favorable behavior of these bounds is tied to their norm assumptions. Recall that for every $W\in\mathbb{R}^{N\times N}$, we have
\[
    \|W\|_{2,1}
    \le\sqrt{N\operatorname{rank}(W)}\|W\|_{2}
    ,
    \qquad
    \|W\|_{1,1}
    \le N\sqrt{\operatorname{rank}(W)}\|W\|_{2}
    .
\]
Thus, treating $C_{2,1}$ or $C_{1,1}$ as independent of $N$ is substantially stronger than imposing only spectral norm constraints. This issue is visible in trained BERT weights: the mixed norms increase with the hidden dimension $N$, as shown in Figure~\ref{fig:NormScalingBERT}.

\begin{figure}[t]
    \centering
    \includegraphics[width=0.49\linewidth]{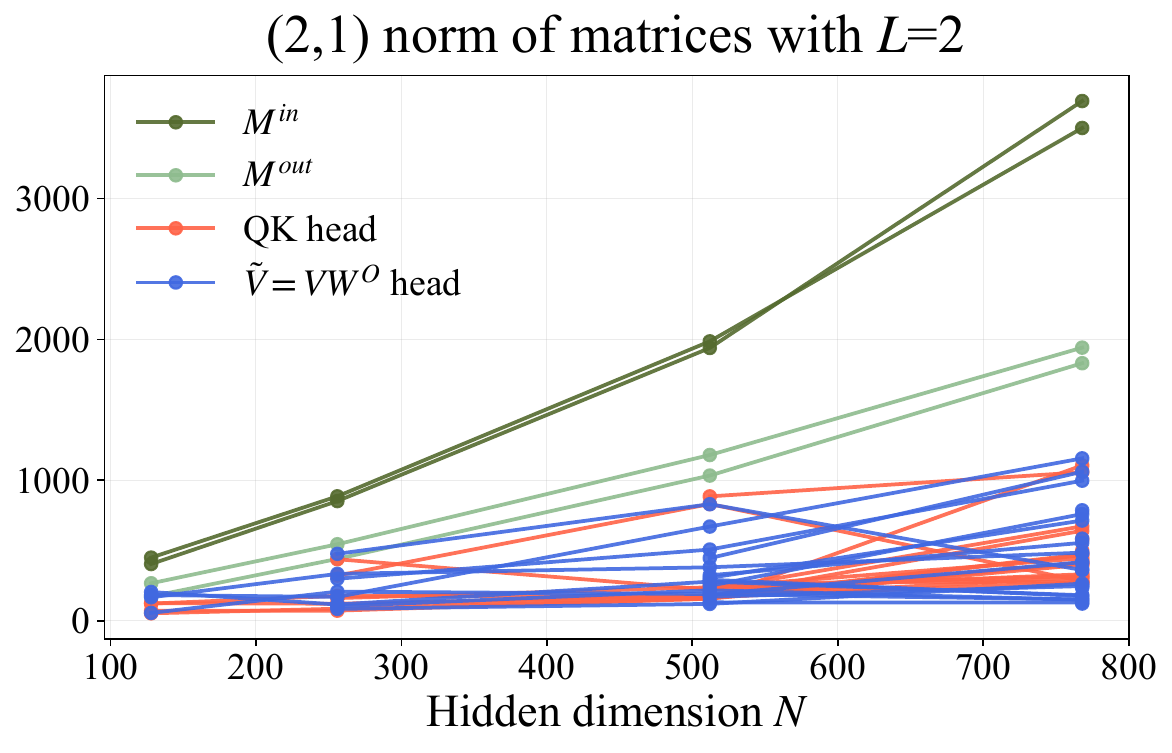}
    \includegraphics[width=0.49\linewidth]{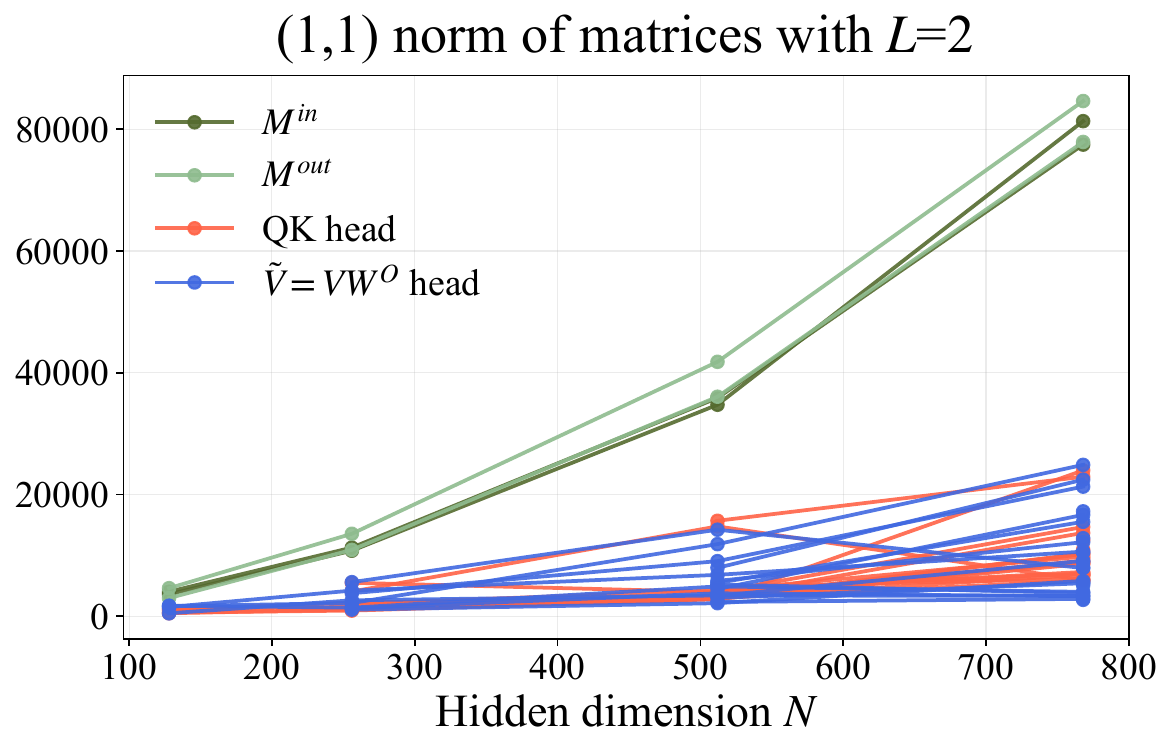}
    \caption{Scaling of the mixed $(2,1)$- and $(1,1)$-norms of trained BERT weights at depth $L=2$. The observed growth with the hidden dimension $N$ suggests that treating the corresponding mixed-norm radii as dimension-independent is not well supported by these checkpoints.}
    \label{fig:NormScalingBERT}
\end{figure}

We focus on the bounds of \citet{edelman2022inductive} in the numerical comparison. The present experiment suppresses constants and logarithmic factors and compares only the leading polynomial growth of BERT-adapted proxies. The main advantage of \citet{trauger2024sequence} is the removal of explicit token-length dependence, but this advantage is not visible in this polynomial-only comparison. Moreover, their bounds are based on mixed $(1,1)$-norm radii, which can grow more severely with the hidden dimension than the mixed $(2,1)$-norm radii used in \citet{edelman2022inductive}. Thus, Edelman-type proxies provide a direct baseline for assessing the polynomial scaling behavior of our post hoc bounds. This does not diminish the importance of \citet{trauger2024sequence} in settings where the token-length dependence is the central quantity of interest.

\subsection{BERT checkpoints and construction of proxies}

We used the publicly released BERT Miniatures checkpoints of \citet{turc2019well}, hosted on Hugging Face under the Google organization. The checkpoints are identified by the model names \texttt{bert\_uncased\_L-\{2,4,6,8,10,12\}\_H-\{128,256,512,768\}\_A-\{2,4,8,12\}}. We accessed the checkpoints through their public Hugging Face URLs, and did not modify or redistribute the checkpoint files. In this model family, the number of self-attention heads is $A_{h}=N/d_{h}$ with $d_{h}=64$, and the intermediate dimension of the feedforward sublayer is $I=4N$. The grid is
\[
    L\in\{2,4,6,8,10,12\}
    ,
    \qquad
    N\in\{128,256,512,768\}
    .
\]

There are several architectural differences between these checkpoints and the theoretical model in the main text. Since the theory is formulated in terms of single-head weight matrices, the BERT-adapted proxies below modify only the corresponding weight-matrix components. Specifically, we adapt the theoretical query-key, value, and feedforward matrices to four BERT-specific features: multi-head attention, headwise query-key products, the value-output projection structure, and the two-layer feedforward sublayer.

First, BERT uses multi-head attention. We therefore treat each head $h\in[A_{h}]$ as a local single-head component. Second, the query-key matrix in our theory is represented by the headwise product
\[
    W^{QK,(\ell,h)}
    =(W^{Q,(\ell,h)})^{\top}W^{K,(\ell,h)}
    \in\mathbb{R}^{N\times N}
    .
\]
Since the query and key matrices in these checkpoints have dimensions $W^{Q,(\ell,h)},W^{K,(\ell,h)}\in\mathbb{R}^{d_{h}\times N}$, the matrix $W^{QK,(\ell,h)}$ has rank at most $d_{h}=64$. Third, our single-head theoretical model contains a single value matrix $W^{V,(\ell)}\in\mathbb{R}^{N\times N}$ and does not include a separate output projection for each attention head. In BERT, for head $h$, we write the value projection and the corresponding slice of the output projection in row-vector orientation as $W^{V,(\ell,h)}\in\mathbb{R}^{N\times d_{h}}$ and $W^{O,(\ell,h)}\in\mathbb{R}^{d_{h}\times N}$. We therefore use the composed headwise value-output matrix $W^{\tilde{V},(\ell,h)}:=W^{V,(\ell,h)}W^{O,(\ell,h)}\in\mathbb{R}^{N\times N}$, whose rank is at most $d_{h}=64$. Fourth, the feedforward sublayer is a two-layer map
\[
    X
    \mapsto\phi_{\mathrm{GELU}}(XW^{M,(\ell,\mathrm{in})}+b^{M,(\ell,\mathrm{in})})W^{M,(\ell,\mathrm{out})}+b^{M,(\ell,\mathrm{out})}
    ,
\]
so we use $W^{M,(\ell,\mathrm{in})}\in\mathbb{R}^{N\times I}$ and $W^{M,(\ell,\mathrm{out})}\in\mathbb{R}^{I\times N}$ separately in the proxy. We also use $L_{\mathrm{GELU}}\le1.13$ for the Lipschitz constant of the elementwise GELU map.

These definitions should not be interpreted as a formal extension of our generalization bounds to the full BERT encoder architecture. They do not incorporate residual connections, LayerNorm, bias parameters, or token and positional embeddings. These components may affect the effective Lipschitz constants, norm propagation, and parameter complexity of literal BERT generalization bounds, and incorporating them would require a separate covering analysis. Thus, the quantities defined in this appendix should be understood as BERT-adapted leading-factor diagnostics, whose purpose is to compare how the leading spectral and mixed-norm complexity factors scale across checkpoints.

Following Lemma~\ref{lem:GenGapPostHocGrid}, we isolate the leading polynomial factor of our post hoc bounds. We suppress absolute constants, logarithmic factors, and confidence terms, because the experiment is intended to compare scaling across checkpoints. To adapt the single-head expression in Lemma~\ref{lem:GenGapPostHocGrid} to BERT, we first replace the propagation factor $\alpha^{(\ell)}$ by
\[
    \tilde{\alpha}^{(\ell)}
    =\prod_{k=\ell+1}^{L}L_{\phi}C_{2}^{M,(k,\mathrm{in})}C_{2}^{M,(k,\mathrm{out})}\sum_{h=1}^{A_{h}}C_{2}^{\tilde{V},(k,h)}(1+4C_{2}^{QK,(k,h)})
    .
\]
We also replace the local factors $\gamma^{\star,(\ell)}$ by
\begin{align*}
    &\tilde{\gamma}^{QK,(\ell,h)}
    =2C_{2}^{\tilde{V},(\ell,h)}C_{2}^{M,(\ell,\mathrm{out})}C_{2}^{M,(\ell,\mathrm{in})}
    ,
    &&\tilde{\gamma}^{V,(\ell,h)}
    =C_{2}^{M,(\ell,\mathrm{out})}C_{2}^{M,(\ell,\mathrm{in})}
    ,\\
    &\tilde{\gamma}^{M,(\ell,\mathrm{in})}
    =C_{2}^{M,(\ell,\mathrm{out})}
    ,
    &&\tilde{\gamma}^{M,(\ell,\mathrm{out})}
    =1
    .
\end{align*}
With $m=\lceil L+\log(N)\rceil$, we define the BERT-adapted proxies for the polynomial part of our bounds by $\tilde{O}(B_{\mathrm{ours}}/\sqrt{n})$, where
\begin{align*}
    B_{\mathrm{ours}}
    &=\inf_{\bm{p}\in\mathcal{P}_{m}}\mathfrak{B}_{\mathrm{BERT}}(\bm{p})+\sqrt{L}
    ,\\
    \mathfrak{B}_{\mathrm{BERT}}(\bm{p})
    &=\sum_{\ell=1}^{L}\left(\sum_{h=1}^{A_{h}}\mathfrak{B}_{QK}^{(\ell,h)}+\sum_{h=1}^{A_{h}}\mathfrak{B}_{\tilde{V}}^{(\ell,h)}+\mathfrak{B}_{M}^{(\ell,\mathrm{in})}+\mathfrak{B}_{M}^{(\ell,\mathrm{out})}\right)
    ,
\end{align*}
and for $\star\in\{QK,\tilde{V},M\}$, with $a\in[A_{h}]$ when $\star\in\{QK,\tilde{V}\}$ and $a\in\{\mathrm{in},\mathrm{out}\}$ when $\star=M$,
\begin{equation}\label{eq:MatrixWiseBoundProxy}
    \mathfrak{B}_{\star}^{(\ell,a)}
    =(\|W^{\star,(\ell,a)}\|_{\mathrm{s},p^{\star,(\ell,a)}}^{p^{\star,(\ell,a)}})^{\frac{1}{p^{\star,(\ell,a)}+2}}(\tilde{\gamma}^{\star,(\ell,a)}\tilde{\alpha}^{(\ell)}L)^{\frac{p^{\star,(\ell,a)}}{p^{\star,(\ell,a)}+2}}N^{\frac{p^{\star,(\ell,a)}+1}{p^{\star,(\ell,a)}+2}}
    .
\end{equation}
Here $\mathcal{P}_{m}$ is the corresponding grid of Schatten indices for all headwise attention matrices and feedforward matrices.

For the bounds of \citet{edelman2022inductive}, we analogously retain only the leading polynomial factor and define the BERT-adapted proxies by $\tilde{O}(B_{\mathrm{Edelman}}/\sqrt{n})$, where
\begin{align*}
    &B_{\mathrm{Edelman}}
    =\left(1+\sum_{\ell=1}^{L}(\tilde{\alpha}^{(\ell)})^{\frac{2}{3}}\tilde{\xi}^{(\ell)}\right)^{\frac{3}{2}}
    ,\\
    &\tilde{\xi}^{(\ell)}
    =\sum_{h=1}^{A_{h}}(C_{2}^{M,(\ell,\mathrm{out})}C_{2}^{M,(\ell,\mathrm{in})}C_{2}^{\tilde{V},(\ell,h)}C_{2,1}^{QK,(\ell,h)})^{\frac{2}{3}}+\sum_{h=1}^{A_{h}}(C_{2}^{M,(\ell,\mathrm{out})}C_{2}^{M,(\ell,\mathrm{in})}C_{2,1}^{\tilde{V},(\ell,h)})^{\frac{2}{3}}\\
    &\qquad+(C_{2}^{M,(\ell,\mathrm{out})}C_{2,1}^{M,(\ell,\mathrm{in})})^{\frac{2}{3}}+(C_{2,1}^{M,(\ell,\mathrm{out})})^{\frac{2}{3}}
    .
\end{align*}

One may worry that the feedforward matrices in the BERT checkpoints have intermediate dimension $I=4N$, and hence that the mixed $(2,1)$-norms used in the Edelman-type proxies are penalized by the larger rectangular shape. This effect should be interpreted with some care. Since the expansion ratio $I/N=4$ is fixed for all checkpoints, replacing an $N\times N$ feedforward matrix by the two rectangular matrices $W^{M,(\ell,\mathrm{in})}\in\mathbb{R}^{N\times 4N}$ and $W^{M,(\ell,\mathrm{out})}\in\mathbb{R}^{4N\times N}$ does not by itself change the polynomial exponent in $N$; it changes only fixed aspect-ratio constants, together with the actual scaling of the trained weights. Indeed, if $W^{M,(\ell,\mathrm{in})}$ is partitioned into four $N\times N$ column blocks, then its mixed $(2,1)$-norm is exactly the sum of the mixed $(2,1)$-norms of these four blocks. Similarly, if $W^{M,(\ell,\mathrm{out})}$ is partitioned into four $N\times N$ row blocks, then its mixed $(2,1)$-norm lies between one half of the sum and the full sum of the blockwise mixed $(2,1)$-norms.\footnote{In general, if we partition $W\in\mathbb{R}^{I\times N}$ into four $N\times N$ row blocks as $W^{\top}=[W_{1}^{\top}\ W_{2}^{\top}\ W_{3}^{\top}\ W_{4}^{\top}]$ with $W_{b}\in\mathbb{R}^{N\times N}$, we have $2^{-1}\sum_{b=1}^{4}\|W_{b}\|_{2,1}\le\|W\|_{2,1}\le\sum_{b=1}^{4}\|W_{b}\|_{2,1}$.} Thus the intermediate dimension contributes a fixed-width factor rather than a new $N$-dependent exponent. The remaining growth of the Edelman-type proxies with $N$ is therefore not merely an artifact of using $I=4N$, but reflects the behavior of the mixed norms of the trained feedforward weights.

\subsection{Comparison and interpretation}\label{subsec:EmpiricalComparisonInterpret}

Figure~\ref{fig:ObservedBoundComparison} plots the normalized version of the generalization gap proxies $B_{\mathrm{ours}}$ and $B_{\mathrm{Edelman}}$ as $L$ or $N$ varies. The resulting curves show that our proxies grow more slowly than the proxies obtained from \citet{edelman2022inductive} both when the depth is varied at fixed hidden dimension and when the hidden dimension is varied at fixed depth.

We next examine how the post hoc choice of Schatten indices produces this behavior. In all BERT Miniatures checkpoints considered here, the minimum of $\mathfrak{B}_{\mathrm{BERT}}(\bm{p})$ is attained at $p=0$ for every matrix type, layer, and head. This common optimizer should not be interpreted as saying that all trained weight matrices have the same spectral structure. Rather, it means that, for the contribution term in Eq.~\eqref{eq:MatrixWiseBoundProxy}, the rank-based endpoint gives the most favorable balance among the Schatten term, the layerwise propagation factor, the depth, and the hidden dimension.

The reason is visible directly from Eq.~\eqref{eq:MatrixWiseBoundProxy}. Increasing $p$ replaces the rank-like term by a more norm-like Schatten quantity, but it also increases the architectural factor
\[
    (\tilde{\gamma}^{\star,(\ell,a)}\tilde{\alpha}^{(\ell)}L)^{\frac{p^{\star,(\ell,a)}}{p^{\star,(\ell,a)}+2}}N^{\frac{p^{\star,(\ell,a)}+1}{p^{\star,(\ell,a)}+2}}
    .
\]
Therefore, positive Schatten indices can improve the proxies only when the reduction in the Schatten term is large enough to compensate for the additional depth- and dimension-dependent factors. In the grid of BERT Miniatures checkpoints used here, this compensation does not occur, so the rank-based endpoint remains optimal.

\begin{figure}[t]
    \centering
    \includegraphics[width=0.45\linewidth]{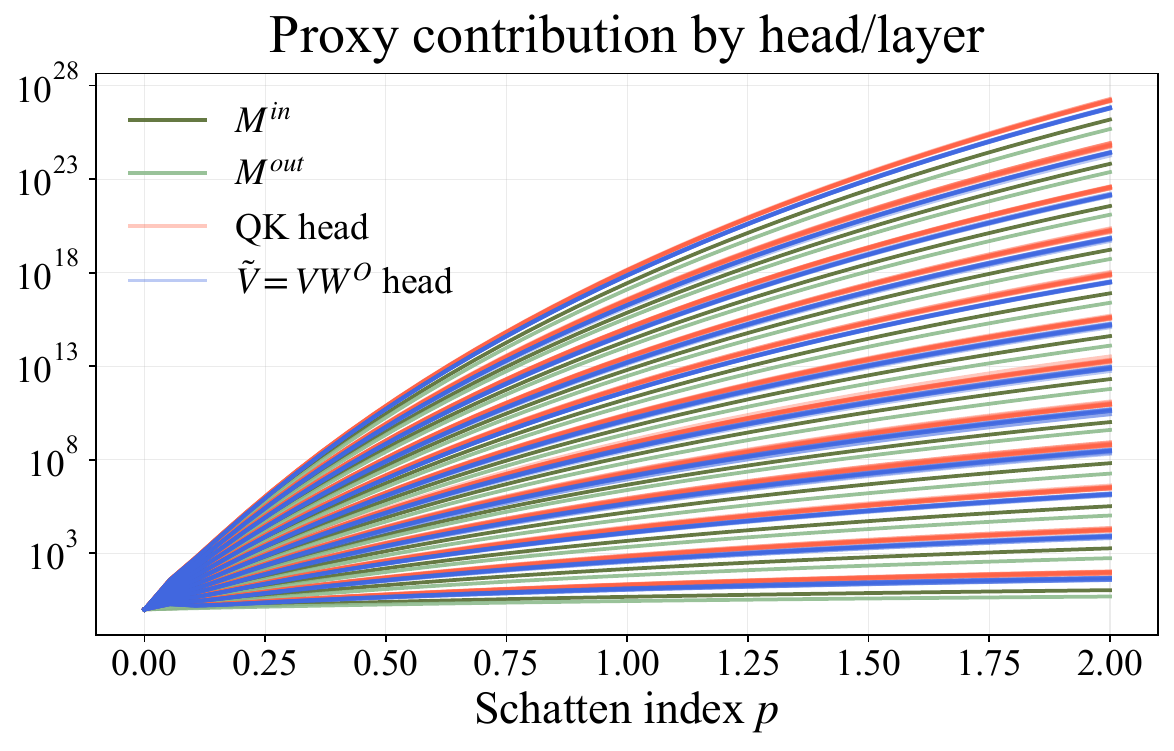}
    \includegraphics[width=0.45\linewidth]{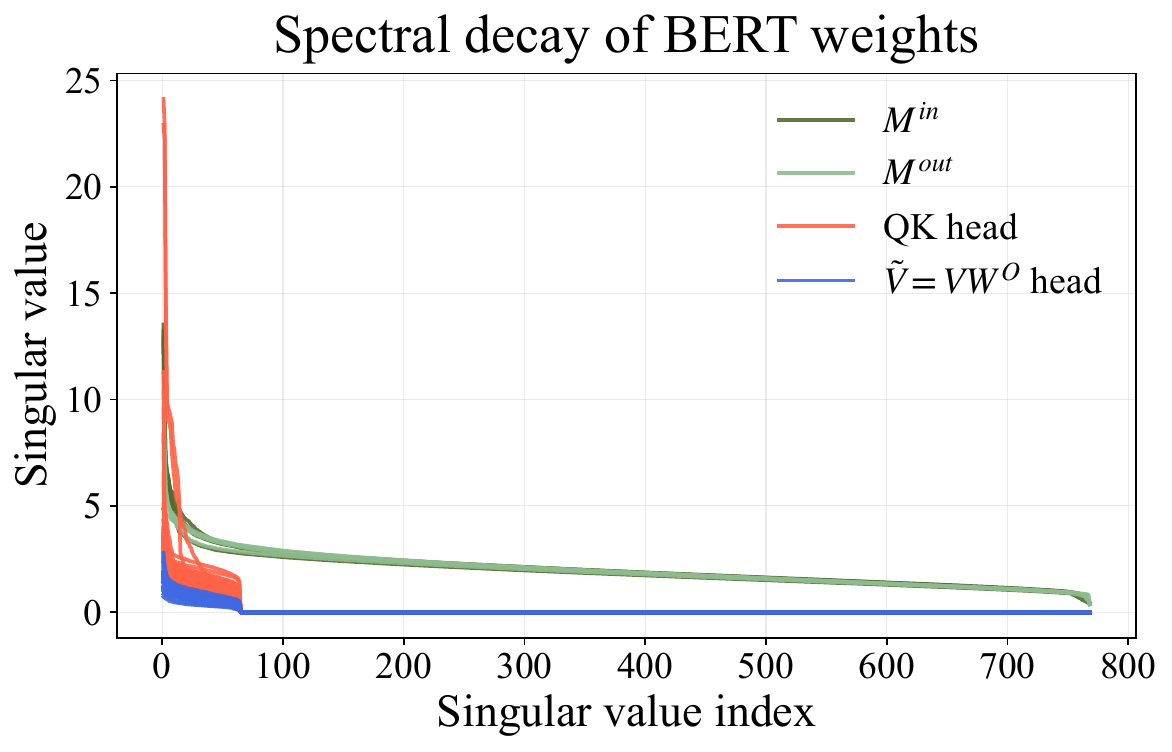}
    \caption{Diagnostics for the post hoc Schatten-index choice for each weight matrix at hidden dimension $N=768$ and depth $L=12$. \textbf{Left:} Matrixwise proxy contributions as a function of the Schatten index $p$, normalized by the value at $p=0$. The increase of all curves as $p$ moves away from zero explains why the proxies select the rank-based endpoint. \textbf{Right:} Singular-value spectra of BERT weight matrices. It shows that this common endpoint selection does not come from identical spectra: the headwise attention-related matrices are exactly low-rank with rank at most $d_{h}=64$, whereas the feedforward matrices have slower spectral decay. Together, the two panels show that the selected indices reflect a bound-dependent trade-off between spectral structure and architectural factors.}
    \label{fig:ChoosingP}
\end{figure}

The left panel of Figure~\ref{fig:ChoosingP} gives a matrixwise view of the endpoint selection. After normalization by the value at $p=0$, every contribution increases as $p$ moves away from zero. The curves are organized mainly by layer rather than by head: within a fixed layer, the headwise curves for $W^{QK}$ and $W^{\tilde{V}}$ remain close to one another, whereas the separation across layers is much larger. The higher blocks correspond to earlier layers and the lower blocks to later layers, which is consistent with the stronger effect of the propagation factor $\tilde{\alpha}^{(\ell)}$ in earlier layers. Within a fixed layer, the rise away from $p=0$ is qualitatively strongest for query-key matrices, then for value-related matrices, and then for feedforward matrices.

The right panel of Figure~\ref{fig:ChoosingP} explains why this common optimizer still hides substantial matrix-type heterogeneity. The attention-related matrices $W^{QK}$ and $W^{\tilde{V}}$ are headwise objects and, by construction, have rank at most $d_{h}=64$. Hence, for the BERT checkpoints considered here, these matrices are exactly low-rank relative to the hidden dimension. This makes $p=0$ particularly natural for the attention-related terms. Moreover, the singular values of $W^{QK}$ tend to be larger than those of $W^{\tilde{V}}$, so using positive Schatten indices is relatively more expensive for the query-key proxies.

The feedforward matrices $W^{M,\mathrm{in}}$ and $W^{M,\mathrm{out}}$ behave differently. They do not have the same exact low-rank structure, and their spectral decay is slower. Thus, for these matrices, positive values of $p$ could in principle exploit norm-based spectral information from the singular-value profile. However, in the present proxies, that possible spectral gain is outweighed by the corresponding increase in the propagation, depth, and dimension factors. Overall, Figure~\ref{fig:ChoosingP} shows that the post hoc Schatten indices are not merely diagnostics of spectral decay, but are bound-dependent complexity parameters that balance spectral structure against architectural scaling factors.

\vfill
\pagebreak

\input{checklist}

\end{document}

%% file: checklist.tex
\section*{NeurIPS Paper Checklist}

\begin{enumerate}

\item {\bf Claims}
    \item[] Question: Do the main claims made in the abstract and introduction accurately reflect the paper's contributions and scope?
    \item[] Answer: \answerYes{} 
    \item[] Justification: The abstract and introduction state the main contributions: spectrum-adaptive post hoc generalization bounds for multi-layer Transformers, a matrix-valued parametric interpolation argument, comparisons with existing norm-based bounds, and proxy experiments using BERT Miniatures checkpoints; see Sections~\ref{sec:Introduction}, \ref{sec:MainResult}, \ref{sec:examples}, and \ref{sec:experiments}.

\item {\bf Limitations}
    \item[] Question: Does the paper discuss the limitations of the work performed by the authors?
    \item[] Answer: \answerYes{} 
    \item[] Justification: The paper discusses limitations in the conclusion: the analysis is training-agnostic, uses the learned weights only through post hoc spectral quantities, and does not exploit possible low-dimensional activation structure. It also explains that the BERT experiment is a proxy comparison rather than a literal numerical evaluation of the theorems; see Sections~\ref{sec:experiments} and \ref{sec:Conclusion}.

\item {\bf Theory assumptions and proofs}
    \item[] Question: For each theoretical result, does the paper provide the full set of assumptions and a complete (and correct) proof?
    \item[] Answer: \answerYes{} 
    \item[] Justification: The paper states the relevant assumptions for the main theoretical results, including the Transformer model, the bounded/Lipschitz loss setup, layerwise spectral norm control, and the Lipschitz activation assumption in Sections~\ref{sec:Preliminaries} and \ref{sec:MainResult}. Proof sketches are given in Section~\ref{sec:ProofOutline}, and the complete technical arguments are provided in Appendices~\ref{sec:MathPrelim}--\ref{sec:ProofsMainResults}.

    \item {\bf Experimental result reproducibility}
    \item[] Question: Does the paper fully disclose all the information needed to reproduce the main experimental results of the paper to the extent that it affects the main claims and/or conclusions of the paper (regardless of whether the code and data are provided or not)?
    \item[] Answer: \answerYes{} 
    \item[] Justification: The empirical part is a deterministic proxy comparison rather than a training experiment. Section~\ref{sec:experiments} and Appendix~\ref{sec:EmpiricalComparisonDiscussion} specify the use of BERT Miniatures checkpoints, the depth/hidden-dimension grid, the architectural adaptations from the theoretical model to BERT, and the construction and normalization of the proxy quantities.

\item {\bf Open access to data and code}
    \item[] Question: Does the paper provide open access to the data and code, with sufficient instructions to faithfully reproduce the main experimental results, as described in supplemental material?
    \item[] Answer: \answerYes{} 
    \item[] Justification: The BERT Miniatures checkpoints used in the experiments are publicly available and cited, and Appendix~\ref{sec:EmpiricalComparisonDiscussion} specifies the checkpoint grid and the BERT-adapted proxy construction. The supplementary material includes simulation code and instructions for reproducing the figures.

\item {\bf Experimental setting/details}
    \item[] Question: Does the paper specify all the training and test details (e.g., data splits, hyperparameters, how they were chosen, type of optimizer) necessary to understand the results?
    \item[] Answer: \answerYes{} 
    \item[] Justification: The paper specifies that no new model is trained for the reported experiments; instead, the experiments compute proxy quantities from publicly released BERT Miniatures checkpoints. Appendix~\ref{sec:EmpiricalComparisonDiscussion} gives the checkpoint grid, the treatment of multi-head attention and feedforward layers, the value and query-key proxy definitions, and the normalization used for the plots.

\item {\bf Experiment statistical significance}
    \item[] Question: Does the paper report error bars suitably and correctly defined or other appropriate information about the statistical significance of the experiments?
    \item[] Answer: \answerNo{} 
    \item[] Justification: The reported empirical results are deterministic calculations of normalized leading-complexity proxy quantities on fixed public checkpoints, rather than stochastic training or test-set evaluations. Therefore, error bars or confidence intervals are not directly applicable to the main empirical comparison.

\item {\bf Experiments compute resources}
    \item[] Question: For each experiment, does the paper provide sufficient information on the computer resources (type of compute workers, memory, time of execution) needed to reproduce the experiments?
    \item[] Answer: \answerYes{} 
    \item[] Justification: The experiments only compute deterministic proxy quantities from publicly available checkpoints and generate the corresponding plots. They do not involve model training or fine-tuning. The full experiment can be run on a standard laptop CPU in a few minutes.
    
\item {\bf Code of ethics}
    \item[] Question: Does the research conducted in the paper conform, in every respect, with the NeurIPS Code of Ethics \url{https://neurips.cc/public/EthicsGuidelines}?
    \item[] Answer: \answerYes{} 
    \item[] Justification: The work is primarily theoretical, and its empirical analysis is limited to evaluating of generalization bounds using public checkpoints. It does not involve private data, human subjects, or deployment of a potentially harmful system.

\item {\bf Broader impacts}
    \item[] Question: Does the paper discuss both potential positive societal impacts and negative societal impacts of the work performed?
    \item[] Answer: \answerNA{} 
    \item[] Justification: The paper is foundational theory for understanding generalization in Transformers and we do not identify a direct path from the stated contributions to concrete societal impacts.
    
\item {\bf Safeguards}
    \item[] Question: Does the paper describe safeguards that have been put in place for responsible release of data or models that have a high risk for misuse (e.g., pre-trained language models, image generators, or scraped datasets)?
    \item[] Answer: \answerNA{} 
    \item[] Justification: The paper does not release a new dataset, pretrained model, or other asset with a high risk of misuse.

\item {\bf Licenses for existing assets}
    \item[] Question: Are the creators or original owners of assets (e.g., code, data, models), used in the paper, properly credited and are the license and terms of use explicitly mentioned and properly respected?
    \item[] Answer: \answerYes{} 
    \item[] Justification: The paper uses the publicly released BERT Miniatures checkpoints of \citet{turc2019well} and cites the original work in Appendix~\ref{sec:EmpiricalComparisonDiscussion}. The checkpoints hosted under the Google organization on Hugging Face are listed with the Apache-2.0 license. We use the checkpoints only to extract trained weight matrices for computing deterministic proxy quantities, and do not redistribute modified checkpoint files.

\item {\bf New assets}
    \item[] Question: Are new assets introduced in the paper well documented and is the documentation provided alongside the assets?
    \item[] Answer: \answerYes{} 
    \item[] Justification: Simulation code is submitted and documented in the supplementary material.

\item {\bf Crowdsourcing and research with human subjects}
    \item[] Question: For crowdsourcing experiments and research with human subjects, does the paper include the full text of instructions given to participants and screenshots, if applicable, as well as details about compensation (if any)? 
    \item[] Answer: \answerNA{} 
    \item[] Justification: This paper does not involve crowdsourcing or research with human subjects.

\item {\bf Institutional review board (IRB) approvals or equivalent for research with human subjects}
    \item[] Question: Does the paper describe potential risks incurred by study participants, whether such risks were disclosed to the subjects, and whether Institutional Review Board (IRB) approvals (or an equivalent approval/review based on the requirements of your country or institution) were obtained?
    \item[] Answer: \answerNA{} 
    \item[] Justification: This paper does not involve crowdsourcing nor research with human subjects.

\item {\bf Declaration of LLM usage}
    \item[] Question: Does the paper describe the usage of LLMs if it is an important, original, or non-standard component of the core methods in this research? Note that if the LLM is used only for writing, editing, or formatting purposes and does \emph{not} impact the core methodology, scientific rigor, or originality of the research, declaration is not required.
    \item[] Answer: \answerNA{} 
    \item[] Justification: The core method development in this research does not involve LLMs as any important, original, or non-standard components.

\end{enumerate}